\documentclass[11pt]{article}

\usepackage[numbers, compress]{natbib}

\usepackage{fullpage,times}
\usepackage{parskip}
\usepackage{multirow}
\usepackage{wrapfig}

\usepackage{lipsum}  
\makeatletter
\newcommand{\printfnsymbol}[1]{%
  \textsuperscript{\@fnsymbol{#1}}%
}
\makeatother

\makeatletter
\def\hlinewd#1{%
\noalign{\ifnum0=`}\fi\hrule \@height #1 %
\futurelet\reserved@a\@xhline}
\makeatother

\title{
Flow Straight and Fast: \\ 
 Learning to Generate and Transfer Data with 
 {\Name} Flow 
}
\author{%
	Xingchao Liu\thanks{XL and CG contributed equally to this work.}\\
    ~University of Texas at Austin\\
	\texttt{xcliu@utexas.edu}
	\and
	Chengyue Gong\printfnsymbol{1}\\
	~University of Texas at Austin\\
	\texttt{cygong@cs.utexas.edu} \\
	\and 
	Qiang Liu\\
	University of Texas at Austin\\
	\texttt{lqiang@cs.utexas.edu} \\
}

\date{}

\usepackage{qiangstyle}

\begin{document}

\definecolor{commentcolor}{RGB}{110,154,155}   % define comment color
\definecolor{inputcolor}{RGB}{255, 105, 180}   
\newcommand{\PyComment}[1]{\ttfamily\textcolor{commentcolor}{\# #1}}  % add a "#" before the input text "#1"
\newcommand{\PyInput}[1]{\ttfamily\textcolor{inputcolor}{\# #1}}
\newcommand{\PyCode}[1]{\ttfamily\textcolor{black}{#1}} % \ttfamily is the code font

\maketitle
%\qq{brainstorm titles}
 
\begin{abstract}
We present rectified flow, 
a  surprisingly simple approach 
to learning (neural) ordinary differential equation (ODE) models to transport between two empirically observed distributions $\tg_0$ and $\tg_1$,
hence providing a unified solution to 
generative modeling and domain transfer, among various other tasks involving distribution transport. 
The idea of rectified flow is to learn the ODE to follow the straight paths connecting 
the points drawn from $\tg_0$ and $\tg_1$ as much as possible. 
This is  achieved by solving a straightforward nonlinear least squares optimization problem, 
which can be easily scaled to large models 
without introducing extra parameters beyond standard supervised learning. 
%much hyperparameter tuning.
The straight paths  
are special and preferred because they are the shortest  
paths between two points, and can be simulated exactly without time discretization and hence yield computationally efficient models.  
We show that the procedure of learning a rectified flow from data, called rectification, turns an arbitrary coupling of $\tg_0$ and $\tg_1$ to a  new deterministic coupling with provably non-increasing convex transport costs. 
In addition, recursively applying 
rectification allows us to obtain a sequence of flows with increasingly straight paths, which can be simulated accurately with coarse time discretization in the inference phase. %The methods of
%\qq{\todo} As we demonstrate 
In empirical %Empirically, 
studies, 
we show that
%the same algorithm of 
rectified flow performs superbly on 
image generation, 
%provides an attractive and unified algorithm for all tasks involving distribution transferring, including generative modeling,
image-to-image translation, and domain adaptation. 
In particular, on image generation and translation, our method yields 
nearly straight flows 
that give high quality results even with \emph{a single Euler discretization step}. 
\end{abstract}

\section{Introduction} 

Compared with supervised learning, 
the shared difficulty of various forms of unsupervised learning is the lack of 
\emph{paired} input/output data  
with which standard regression or classification tasks can be invoked. 
The gist of 
most unsupervised methods is to 
find, in one way or another,  meaningful correspondences between 
points from two distributions. 
For example, generative models such as 
generative adversarial networks (GAN) 
and variational autoencoders (VAE) 
\cite[e.g.,][]{goodfellow2014generative, kingma2013auto, dinh2016density}
seek to 
map data points 
to latent codes 
following 
a simple elementary (Gaussian) distribution
with which the data can be generated and manipulated. 
Representation learning 
rests on the idea that 
if a sufficiently smooth function can map
a structured data distribution to an elementary 
distribution, it can (likely) be endowed 
with certain  semantically meaningful interpretation
and useful for various downstream learning tasks.  
On the other hand, 
domain transfer methods 
find mappings to transfer points 
from two different data distributions, 
both observed empirically, 
for the purpose of image-to-image translation,  %
style transfer, and domain adaption \citep[e.g.,][]{cyclegan, flamary2016optimal, trigila2016data, peyre2019computational}.
All these tasks can be framed  unifiedly %
as finding a transport map between two distributions: 

\noindent\textbf{The Transport Mapping Problem}  
\emph{Given empirical observations 
of two distributions $X_0\sim \tg_0,  X_1\sim \tg_1$ on $\RR^d$, 
find a %
transport map $T\colon\RR^d\to\RR^d$ (hopefully nice or optimal in certain sense),  
such that $Z_1 \defeq T(Z_0)\sim \tg_1$ when $Z_0 \sim \tg_0$, that is, $(Z_0,Z_1)$ is a coupling (a.k.a transport plan) of $\tg_0$ and $\tg_1$.  
}

Several lines of techniques have been developed depending on how to  represent and train  the map $T$. 
In traditional generative models, 
$T$ 
is parameterized as a neural network, 
and trained with either 
GAN-type minimax algorithms or (approximate) maximum likelihood estimation (MLE).  
However, GANs are known to suffer from numerically instability  and mode collapse issues,  
and require substantial engineering efforts 
and human tuning, which often do not transfer well across different model architecture and datasets.  
On the other hand, MLE tends to be intractable for complex models, 
and hence
requires 
approximate variational or Monte Carlo inference techniques  
such as those used in variational auto-encoders (VAE), 
or special model structures  
such as normalizing flow and auto-regressive models, to yield tractable likelihood, causing  difficult trade-offs between expressive power and computational cost. %

Recently, advances have been made by representing 
the transport plan 
\emph{implicitly as a continuous time  process}, such as 
flow models with neural ordinary differential equations (ODEs) \citep[e.g.,][]{chen2018neural, papamakarios2021normalizing} and  diffusion models by stochastic differential equations (SDEs) \citep[e.g.,][]{song2020score, ho2020denoising, tzen2019theoretical, de2021diffusion, vargas2021solving};
in these models, 
a neural network 
is trained to represent the 
drift force of the processes
and a numerical ODE/SDE solver is used to simulate the process during inference. 
The key idea is that, 
by leveraging the mathematical structures of ODEs/SDEs, 
the continuous-time models 
can be trained efficiently without 
resorting to minimax or traditional approximate inference techniques. 
The most notable examples are 
the recent score-based generative models \cite{song2019generative, song2020improved, song2020score} 
and denoising diffusion probabilistic models (DDPM) \citep{ho2020denoising}, 
which we call denoising diffusion methods collectively. 
These methods 
allow us to train 
large-scale diffusion/SDE-based generative models that 
surpass GANs on image generation in both image quality and diversity,
without the instability and mode collapse issues
\citep[e.g.,][]{dhariwal2021diffusion, glide, dalle2, imagegen}.  
The learned SDEs can be converted into deterministic ODE models  for faster inference with the method of probability flow ODEs \citep{song2020score} and DDIM \citep{song2020denoising}. 

However, compared with the traditional one-step models like GAN and VAE, 
a key drawback of  continuous-times models %
is the 
high computational cost in inference time: drawing a single point (e.g., image) requires to solve the ODE/SDE with a numerical solver that needs to repeatedly call the expensive neural drift function. 
In addition, the existing denoising diffusion techniques 
require substantial hyper-parameter search in an involved design space 
and are still poorly understood both empirically and theoretically \citep{elucidating}.

In existing approaches, generative modeling and domain transfer 
are  typically treated  separately. 
It often requires to extend or customize a generative learning techniques to solve domain transfer problems; see e.g., Cycle GAN \citep{cyclegan} and diffusion-based image-to-image translation \citep[e.g.,][]{su2022dual, zhao2022egsde}. 
One framework that naturally unifies both domains is 
 optimal transport (OT)  \citep[e.g.,][]{villani2021topics,ambrosio2021lectures,figalli2021invitation,peyre2019computational}, which 
endows a collection of techniques for finding optimal
couplings with minimum transport costs of form 
$\E[c(Z_1-Z_0)]$ w.r.t. a cost function $c\colon \RR^d\to \RR$, yielding 
natural applications  to both generative and transfer learning. 
However, 
the existing OT techniques 
are slow for problems with high dimensional and large volumes of data \citep{peyre2019computational}. 
Furthermore, 
as the transport costs do not perfectly align with 
the actual learning performance, 
methods that faithfully find the optimal transport maps 
do not necessarily have better learning performance \citep{korotin2021neural}.

\begin{figure}[h]
    \centering
    \includegraphics[width=0.95\textwidth]{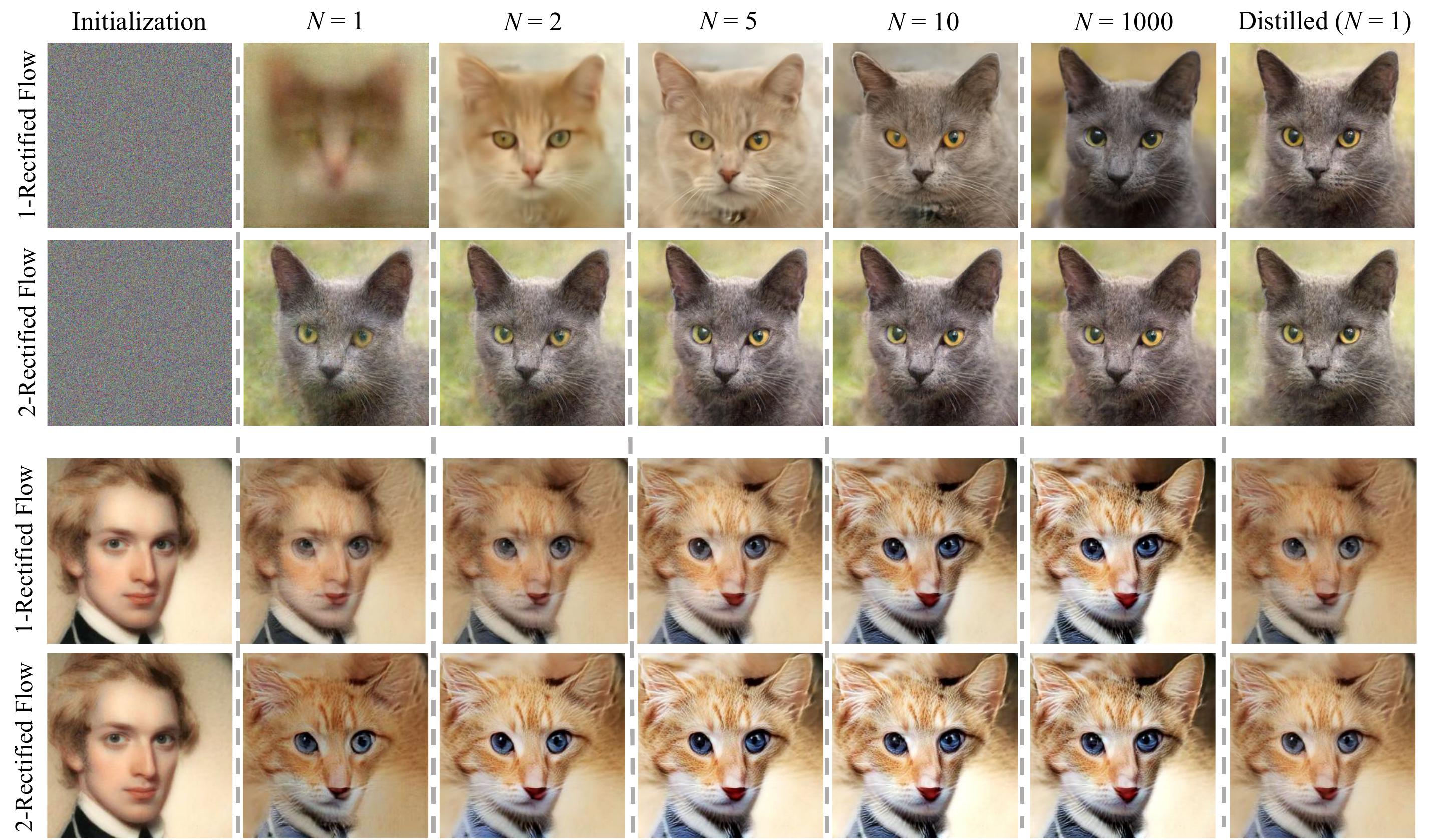} %
    \caption{The trajectories of rectified flows for 
    image generation 
    ($\tg_0$: standard Gaussian noise,  $\tg_1$: cat faces, top two rows), 
    and image transfer between human and cat faces ($\tg_0$: human faces,  $\tg_1$: cat faces, bottom two rows), 
    when simulated using Euler method with step size $1/N$ for $N$ steps. 
    The first rectified flow induced from the training data (\emph{1-rectified flow}) yields good results with a very small number (e.g., $\geq 2$) of steps; 
   the straightened reflow induced from \emph{1-rectified flow} (denoted as \emph{2-rectified flow})   
   has nearly straight line trajectories and yield good results even with one  discretization step. 
    }
    \label{fig:cat_k}
\end{figure}

\subsection*{Contribution}
We introduce 
\emph{rectified flow}, 
a surprisingly  
simple approach  
to the transport mapping problem, 
which unifiedly solves both 
generative modeling and domain transfer. 
The rectified flow 
is an ODE model 
that transport 
distribution $\tg_0$ to $\tg_1$
by 
\emph{following straight line paths as much as possible.} 
The straight paths 
are preferred both theoretically 
because it is the shortest path 
between two end points, 
and computationally because it can be 
exactly simulated without time discretization. 
Hence,
 flows  with straight paths 
bridge the gap between one-step and continuous-time models. %

Algorithmically, the rectified flow 
is trained with 
a simple and scalable unconstrained 
least squares optimization procedure,    which 
avoids the instability issues of GANs,  
the intractable likelihood of MLE methods, and the subtle hyper-parameter decisions of denoising diffusion models.  
The procedure of obtaining the rectified flow from the training data has the attractive theoretical property of 
1) yielding a coupling with 
non-increasing transport cost jointly for all convex cost $c$, 
and 2) making the paths of flow increasingly straight and hence incurring lower error with numerical solvers.
Therefore, 
with a \emph{reflow} 
procedure that iteratively trains new 
rectified flows
with the data simulated from the previously obtained rectified flow, 
we obtain nearly straight flows 
that yield good results even with
the coarsest time discretization, i.e., 
one Euler step. 
Our method is purely ODE-based, 
and is both conceptually simpler and practically faster in inference time than the SDE-based approaches of
\cite{ho2020denoising, song2020score, song2020denoising}. 

Empirically, 
rectified flow 
can yield
high-quality results 
for image generation when simulated with a very few number of Euler steps (see Figure~\ref{fig:cat_k}, top row).
Moreover, with just one step of reflow, 
the flow becomes nearly straight and hence yield good results with a single Euler discretization step (Figure~\ref{fig:cat_k}, the second row).
This substantially improves 
over the standard denoising diffusion methods. %
Quantitatively, we claim a state-of-the-art result of FID (4.85) and recall (0.51) on CIFAR10 for one-step fast diffusion/flow models~\citep{bao2022analytic, lyu2022accelerating, xiao2021tackling, zheng2022truncated, luhman2021knowledge}. 
The same algorithm 
also achieves superb result on domain transfer tasks such as image-to-image translation (see the bottom two rows of Figure \ref{fig:cat_k}) and 
transfer learning.

\section{Method} 
We provide a quick overview of the method in Section~\ref{sec:firstintro}, 
followed with some discussion and remarks in Section~\ref{sec:secondintro}. 
We introduce a nonlinear extension of our method in Section~\ref{sec:nonlinear}, 
with which we clarify the connection and advantages of our method with the method of probability flow ODEs \citep{song2020score} and DDIM \citep{song2020denoising}.

\begin{figure}[t]
\hspace{-1em}
\begin{tabular}{wl{0.17\textwidth}wl{0.3\textwidth}wl{0.17\textwidth}wl{0.17\textwidth}wl{0.15\textwidth}wl{0.1\textwidth}} 
\includegraphics[width=.15\textwidth]{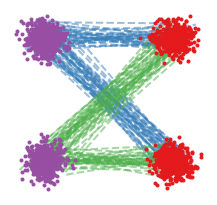}  &
\includegraphics[width=.15\textwidth]{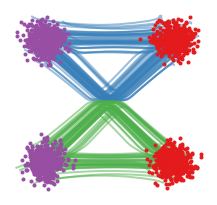} 
\raisebox{1em}{{\includegraphics[width=.12\textwidth]{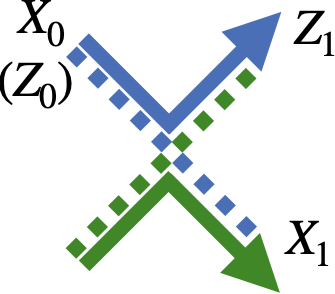}}}%
\hspace{-2em} 
 & 
 \hspace{-2em} 
\includegraphics[width=.15\textwidth]{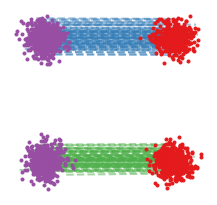} 
 & 
 \hspace{-2em} 
\includegraphics[width=.15\textwidth]{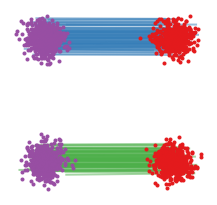} 
 & 
 \hspace{-3.5em}
\raisebox{1.2em}{\fbox{\includegraphics[width=.15\textwidth]{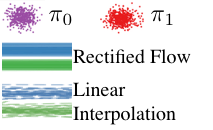}}} 
 & 
\\
\begin{tabular}{l} 
\scriptsize (a) Linear interpolation \\ 
\scriptsize $X_t = tX_1+(1-t)X_0$  
\end{tabular} 
& 
\begin{tabular}{l} 
\scriptsize (b)  Rectified flow $Z_t$ \\ 
\scriptsize induced by $(X_0,X_1)$ 
\end{tabular} 
 & 
  \hspace{-2em} 
\begin{tabular}{l} 
\scriptsize (c) Linear interpolation \\ 
\scriptsize $Z_t = tZ_1+(1-t)Z_0$   
\end{tabular} 
& 
  \hspace{-2em} 
\begin{tabular}{l} 
\scriptsize (d)  Rectified flow $Z'_t$ \\ 
\scriptsize induced by $(Z_0,Z_1)$ 
\end{tabular} 
&  
& 
\end{tabular}
\caption{(a) Linear interpolation of data input $(X_0,X_1)\sim \tg_0\times \tg_1$. 
(b) The rectified flow $Z_t$ induced by $(X_0,X_1)$;  
the trajectories are ``rewired" at the intersection points 
to avoid the crossing. 
 (c) The linear interpolation of the end points $(Z_0,Z_1)$ of flow $Z_t$.
 (d) The rectified flow induced from $(Z_0,Z_1)$, which  follows straight paths. 
}
\label{fig:twodotstoy}
\end{figure}

\subsection{Overview}%
\label{sec:firstintro}

\paragraph{Rectified flow}
Given empirical observations of $X_0\sim \tg_0, X_1\sim \tg_1$, 
the rectified flow induced from $(X_0,X_1)$ 
is an ordinary differentiable model (ODE)
on time $t\in[0,1]$, 
$$
\d Z_t = v(Z_t, t) \dt, 
$$
which converts $Z_0$ from $\tg_0$ to a $Z_1$ following $\tg_1$. 
The drift force 
 $v\colon \RR^d\to \RR^d$  is set to drive the flow to 
 follow the direction $(X_1-X_0)$ 
 of the linear path pointing from $X_0$ to $X_1$
 as much as possible, by solving a simple 
 least squares regression problem:  \bbb\label{equ:mainf}
\min_{v} %
\int_0^1\E \left [ \norm{( X_1 - X_0) - v\big (X_t,~ t\big)}^2
\right ] \dt, 
~~~~~\text{with}~~~~
X_t = t X_1 + (1-t) X_0, 
\eee 
where $X_t$ is 
the linear interpolation  of  $X_0$ and  $X_1$. 
Naviely, $X_t$ follows the ODE of $\d  X_t =(X_1-X_0)\dt,$ which is non-causal (or anticipating) as the update of $X_t$ requires the information of the final point $X_1$.  
By fitting the drift $v$ with $X_1-X_0$, 
the rectified flow \emph{causalizes}
the paths of linear interpolation $X_t$, 
yielding an ODE flow %
that can be simulated without seeing the future. %

 In practice, we parameterize $v$ with a neural network or other nonlinear models 
 and solve  
  \eqref{equ:mainf} %
  with any  off-the-shelf stochastic optimizer, such as stochastic gradient descent,  
  with empirical draws of $(X_0,X_1)$. See 
  Algorithm~\ref{alg:cap}. 
  After we get $v$, we solve the ODE starting from $Z_0 \sim \tg_0$ to transfer $\tg_0$ to $\tg_1$,
   backwardly starting from $Z_1\sim \tg_1$ to transfer $\tg_1$ to $\tg_0$. Specifically, 
  for  backward sampling, 
  we simply  
  solve $\d \tilde X_t = -v(\tilde X_t, t) \dt $ initialized from $\tilde X_0 \sim \tg_1$ and set 
   $X_t= \tilde X_{1-t}$. 
   The forward and backward sampling 
are equally favored by the training  algorithm, because  the objective  in \eqref{equ:mainf} is 
\emph{time-symmetric} 
in that it yields the equivalent problem if we exchange $X_0$ and $X_1$ and flip the sign of $v$.

\paragraph{Flows avoid crossing} %
A key to understanding the method is the non-crossing property of flows: 
the different paths 
following a well defined ODE $\d Z_t = v(Z_t,t) \dt $, 
whose solution exists and is unique, \emph{cannot cross each other} at any time $t\in[0,1)$. 
Specifically,
there exists no location $z\in\RR^d$ and time $t\in[0,1)$, such that two paths go across $z$ at time $t$ along different directions, because otherwise the solution of the ODE would be non-unique. 
On the other hand, the paths of the interpolation process $X_t$ may intersect with each other (Figure~\ref{fig:twodotstoy}a), which makes it non-causal. 
Hence, %
as shown in Figure~\ref{fig:twodotstoy}b,  
the rectified flow 
\emph{rewires} the individual trajectories passing through the intersection points to avoid crossing, while
tracing out the same density map as the linear interpolation paths
due to the optimization of 
\eqref{equ:mainf}.  
We can view the linear interpolation $X_t$ as building
 roads (or tunnels) to connect $\tg_0$ and $\tg_1$, 
 and the rectified flow as traffics of particles passing through the roads in a myopic, memoryless, non-crossing way, 
 which allows them to ignore the global  
 path information of how $X_0$ and $X_1$ are paired, 
 and rebuild a more deterministic pairing of $(Z_0,Z_1)$.

\paragraph{Rectified flows reduce transport costs} 
If \eqref{equ:mainf} is solved exactly, 
the pair $(Z_0,Z_1)$ 
of the rectified flow 
is guaranteed to be a valid coupling of $\tg_0,\tg_1$ (Theorem~\ref{thm:marginal}), 
that is, $Z_1$ follows $\tg_1$ if $Z_0\sim \tg_0$.  
Moreover, $(Z_0,Z_1)$ guarantees to yield no larger transport cost than the data pair $(X_0,X_1)$ simultaneously for \emph{all} convex cost functions $c$ (Theorem~\ref{thm:cost}). 
The data pair $(X_0,X_1)$ can be an arbitrary coupling of $\tg_0,\tg_1$, typically independent  (i.e., $(X_0,X_1)\sim \tg_0\times \tg_1$)  
as dictated by the lack of meaningfully paired observations in practical problems. %
In comparison, 
the rectified coupling $(Z_0,Z_1)$ has a deterministic dependency as it is constructed from an ODE model. 
Denote by $(Z_0,Z_1)=\rectify((X_0,X_1))$ 
the mapping from $(X_0,X_1)$ to $(Z_0,Z_1)$. 
Hence, %
$\rectify(\cdot)$ converts an arbitrary coupling into a deterministic coupling with lower convex transport costs.

\paragraph{Straight line flows yield fast simulation}
Following Algorithm~\ref{alg:cap}, 
denote by $\vv Z = \rectifyflow((X_0,X_1))$ the rectified flow induced from $(X_0,X_1)$. 
Applying this operator recursively 
yields a sequence of rectified  flows 
$\vv Z^{k+1} = \rectifyflow((Z_0^k, Z_1^k))$ with $(Z_0^0,Z_1^0)=(X_0,X_1)$, where $\vv Z^k$ is the $k$-th rectified flow, or simply \emph{$k$-rectified flow}, induced from $(X_0,X_1)$. 

This 
\emph{reflow} 
procedure %
not only decreases transport cost,
but also has the important effect of 
straightening paths of rectified flows, 
that is, making the paths of the flow more straight. 
This is highly attractive computationally  
as flows with nearly straight paths 
incur small time-discretization error 
in numerical simulation. Indeed, perfectly straight paths can be simulated exactly with a single Euler step and is effectively a one-step model. 
This addresses the very bottleneck of high inference cost 
in existing continuous-time ODE/SDE models.

\begin{algorithm}[t]
\caption{{\Name} Flow: Main Algorithm}\label{alg:cap}
\begin{algorithmic}
\STATE \textbf{Procedure}: $\vv Z={\rectifyflow((X_0,X_1))}$: 
\vspace{.2\baselineskip}

\STATE~~~\emph{Inputs}: Draws from a coupling $(X_0,X_1)$ of $\tg_0$ and $\tg_1$; velocity model $v_\theta \colon \RR^d \to \RR^d$ with parameter $\theta.$ 
\vspace{.155\baselineskip}
\STATE~~\emph{Training}:
$
\displaystyle 
\hat \theta = \argmin_\theta
\E\left [\norm{X_1 - X_0 - v(t X_1 +(1-t)X_0, ~t )}^2 \right], 
$
with $t \sim \uniform([0,1])$.
\vspace{.15\baselineskip}

\STATE~~\emph{Sampling}: Draw $(Z_0,Z_1)$ following $\d Z_t = v_{\hat\theta}(Z_t,t)\dt$ starting from $Z_0 \sim \tg_0$ (or backwardly $Z_1\sim \tg_1$).

\STATE~~\emph{Return}: $\vv Z = \{Z_t\colon t\in[0,1]\}$.
\vspace{.4\baselineskip}

\STATE \textbf{Reflow} (optional): $%
\vv Z^{k+1} = \rectifyflow((Z_0^k, Z_1^k))$, starting from $(Z^0_0,Z^0_1) = (X_0,X_1)$.
\vspace{.2\baselineskip}

\STATE \textbf{Distill} (optional): Learn a neural network $\hat T$ to distill the $k$-rectified flow, such that $Z^k_1 \approx \hat T(Z^k_0)$. 
\end{algorithmic}
\end{algorithm}

 \subsection{Main Results and Properties} 
 \label{sec:secondintro}

We provide more in-depth discussions 
on the main properties of  rectified flow. 
We keep the discussion informal 
to highlight the intuitions  in this section and defer the 
  full course theoretical analysis to  Section~\ref{sec:theory}. %
 
 First, 
 for a given input coupling $(X_0,X_1)$, 
it is easy to see that the exact minimum of \eqref{equ:mainf} is achieved if 
 \bbb \label{equ:vxoracle}
 v^X(x,t) = \E[X_1- X_0 ~|~ X_t=x],
 \eee 
 which is the expectation of the line directions $X_1-X_0$ that pass through $x$ at time $t$. 
 We discuss below the property of rectified flow $\d Z_t = v^X(Z_t,t)\dt $ with $Z_0\sim \tg_0$,
 assuming that the ODE has an unique solution. %
 
 \paragraph{Marginal preserving property} [Theorem~\ref{thm:marginal}] \emph{The pair $(Z_0,Z_1)$ is a coupling of $\tg_0$ and $\tg_1$. 
In fact, the marginal law of $Z_t$ equals that of $X_t$ at every time  $t$, that is, $\law(Z_t)=\law(X_t), \forall t\in[0,1]$. } 

Intuitively, this is because, by the definition of $v^X$ in \eqref{equ:vxoracle}, the expected amount of mass
that  passes through every infinitesmal volume at all location and time are equal under the dynamics of $X_t$ and $Z_t$, which ensures that they trace out the same marginal distributions: \newcommand{\imineq}[2]{\vcenter{\hbox{\includegraphics[height=#2ex]{#1}}}}
\bb 
\text{
Flow in $\&$ out}\!\left(\imineq{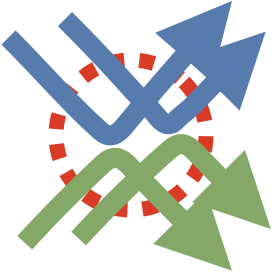}{6} \right) \!\! = 
\text{Flow in $\&$ out}\!\left(\imineq{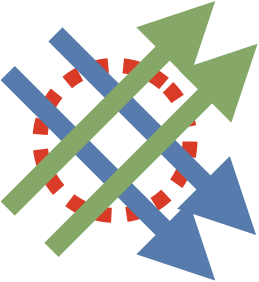}{6} \right)\!, ~\forall \text{time $\&$ location} 
&&\implies && \law(Z_t)=\law(X_t),\forall t.%
\ee 
On the other hand, the joint distributions of the whole trajectory of $Z_t$ and that of $X_t$ are  different in general.  
In particular, $X_t$ is in general a non-causal, non-Markov process, 
with $(X_0,X_1)$ a stochastic coupling,
and $Z_t$  \emph{causalizes}, \emph{Markovianizes}
 and \emph{derandomizes} $X_t$, while preserving the marginal distributions at all time.

\paragraph{Reducing  transport costs} [Theorem~\ref{thm:cost}] \emph{The coupling $(Z_0,Z_1)$ yields lower or equal convex transport costs than the input $(X_0,X_1)$ in that  $\E[c(Z_1-Z_0)] \leq \E [c(X_1 - X_0)]$ for any convex cost $c\colon \RR^d\to \RR$. }

The transport costs measure the expense of transporting the mass of one distribution to another following the assignment relation specified by the coupling and is a central topic in optimal transport \citep[e.g.,][]{villani2009optimal,villani2021topics, santambrogio2015optimal, peyre2019computational, figalli2021invitation}. %
Typical examples are $c(\cdot)= \norm{\cdot}^\alpha$ with $\alpha\geq 1$.
Hence, $\rectify(\cdot)$
yields a Pareto descent on the collection of all convex transport costs, without targeting any specific $c$. 
This distinguishes it from the 
typical optimal transport optimization methods, which are explicitly framed 
to optimize a given $c$. 
As a result, recursive application of $\rectify(\cdot)$  
does not guarantee to attain the $c$-optimal coupling for any given $c$, with the exception 
in the one-dimensional case when the fixed point of $\rectify(\cdot)$ coincides with the unique monotonic coupling that simultaneously minimizes all non-negative convex costs $c$; see Section~\ref{sec:stcouplings}. 

Intuitively, 
the convex transport costs 
are guaranteed to decrease because the paths of the rectified flow $Z_t$ is a rewiring of the straight paths connecting $(X_0,X_1)$. 
To give an illustration, %
consider the simple case of $c(\cdot) = \norm{\cdot}$ when transport costs $\E[\norm{X_0-X_1}]$ and $\E[\norm{Z_0-Z_1}]$  
 are the expected length of the straight lines connecting the end points. 
 The inequality can be proved graphically as follows: 
$$
\text{ 
$\E[\norm{Z_0-Z_1}]$
$= $ 
Length$\left(\imineq{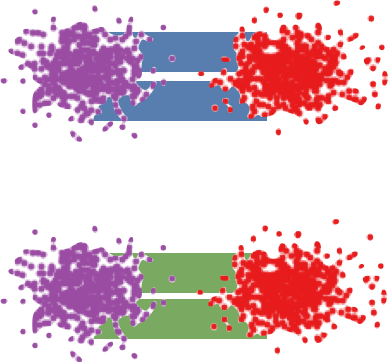}{6}\right)$ 
$\overset{\scriptsize (*)}{\leq}$
Length$\left(\imineq{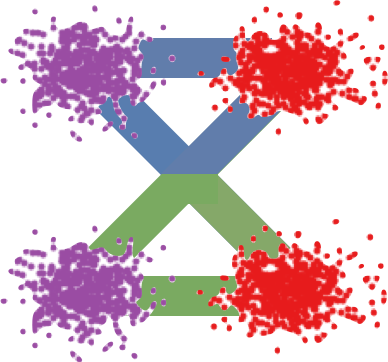}{6}\right)$ %
$\overset{\scriptsize (**)}{=}$ Length$\left(\imineq{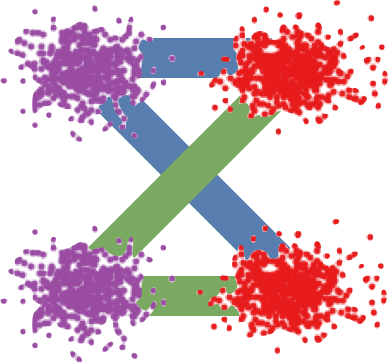}{6}\right)$
$=$ %
${\E[\norm{X_0-X_1}]}$
}, 
$$
where $\overset{\scriptsize (*)}{\leq}$ uses 
the triangle inequality,
and $\overset{\scriptsize(**)}{=}$ holds because the paths of $Z_t$ is a rewiring 
of the straight paths of $X_t$, following the 
construction of $v^\X$ in \eqref{equ:vxoracle}.  
For general convex $c$, 
a similar proof using Jensen's inequality is shown in Section~\ref{sec:cost}.

\paragraph{Reflow, straightening, fast simulation} 
As shown in Figure~\ref{fig:toystar}, 
when we recursively apply the procedure $\vv Z^{k+1} = \rectifyflow((Z_0^k, Z_1^k))$, 
 the paths of the $k$-rectified flow $\vv Z^k$ 
are increasingly 
straight, and hence easier to simulate numerically, as $k$ increases.  %
This straightening tendency can be guaranteed theoretically. 

\begin{figure}[h]
~~~~~~
\begin{tabular}{wc{0.15\textwidth}wc{0.32\textwidth}wc{0.22\textwidth}wc{0.2\textwidth}}
\includegraphics[height=.18\textwidth]{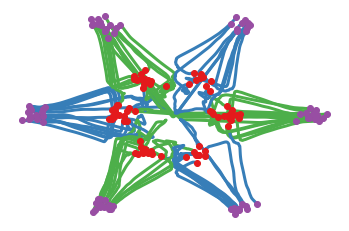} \hspace{-.05\textwidth} & 
\includegraphics[height=.18\textwidth]{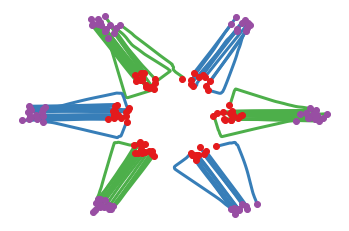} & \hspace{-.05\textwidth} 
\includegraphics[height=.18\textwidth]{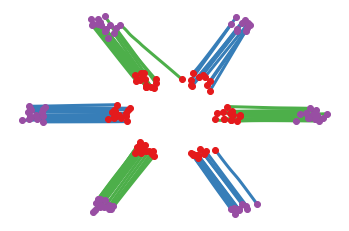} & \hspace{-.00\textwidth} 
\includegraphics[height=.18\textwidth]{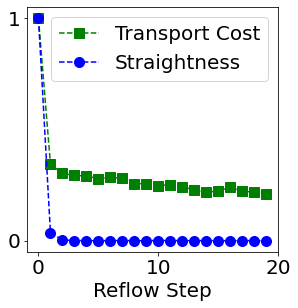}  \\
\begin{tabular}{l}
\small (a)The 1st rectified flow $\vv Z^1$ \\ %
\scriptsize $\vv Z^1=\rectifyflow((X_0,X_1))$
\end{tabular}
& 
\begin{tabular}{l}
\small (b) Reflow $\vv Z^2$ \\ %
\scriptsize $\vv Z^2=\rectifyflow((Z_0^1,Z_1^1))$
\end{tabular}
& 
\begin{tabular}{l}
\small (c)  Reflow $\vv Z^3$ \\ %
\scriptsize $\vv Z^3=\rectifyflow((Z_0^2,Z_1^2))$
\end{tabular}
&  
\begin{tabular}{l}
\small(d)
Transport cost, \\ %
\small Straightness
\end{tabular}
\end{tabular}
\definecolor{dotscolor1}{HTML}{984EA3}
\definecolor{dotscolor2}{HTML}{E41A1C}
\caption{(a)-(c) Samples of trajectories  drawn from the reflows on a toy example (\textcolor{black}{$\tg_0$: purple dots},
\textcolor{black}{$\tg_1$: red dots}; 
the green and blue lines are trajectories connecting different modes of $\tg_0,\tg_1$). 
(d) The
straightness and the relative L2 transport cost   v.s. 
the reflow steps; the values are scaled into $[0,1]$, so 0 corresponds to  straight lines and $L_2$ optimal transport; see Section~\ref{sec:toy} for more information.  %
We use the non-parametric model in \eqref{equ:npfunc} with bandwidth $h=0.1$. 
}
\label{fig:toystar}
\end{figure}

Specifically, 
we say that a flow $\d Z_t = v(Z_t, t)\dt$ %
is straight if we have almost surely that $Z_t = t Z_1 + (1-t) Z_0$ for $\forall t\in[0,1]$, 
or equivalently $v(Z_t,t)=Z_1-Z_0=\const$ following each path. 
(More precisely, ``straight'' here refers to  straight with a constant speed.) 
Such straight flows are highly attractive computationally 
as it is effective a one-step model: 
 a single Euler step update $Z_1 = Z_0 + v(Z_0,0)$ calculates the exact $Z_1$ from $Z_0$. %
 Note that the linear interpolation $\vv X= \{X_t\}$ is straight by this definition but it is not a (causal) flow and hence can not be simulated without an oracle assess to draws of both  $\tg_0$ and $\tg_1$.  
 In comparison, it is non-trivial 
to make a flow $\d Z_t = v(Z_t,t)\dt$  straight, because if so $v$ must satisfy the inviscid Burgers' equation $\partial_t v + (\partial_z v) v = 0$: %
$$
\ddt v(Z_t,t)
= \partial_z v(Z_t, t)  \dot Z_t + \partial_t v(Z_t,t) = 
\partial_z v(Z_t, t)  v(Z_t,t) + \partial_t v(Z_t,t) = 0. 
$$
More generally, 
we can measure the straightness of any 
continuously differentiable process $\vv Z=\{Z_t\}$ by %
\bbb \label{equ:straight}
S(\traj Z) = \int_0^1 \e{\norm{(Z_1-Z_0) -\dot Z_t}^2} \dt.
\eee 
$S(\traj Z) = 0$ means exact straightness. 
A flow whose $S(\traj Z)$  is small 
has nearly straight  paths and hence can be simulated accurately using numerical solvers with a small number of discretization steps.  %
Section~\ref{sec:straight} shows that 
applying rectification recursively 
provably decreases 
 $S(\traj Z)$ towards zero. %

\emph{[Theorem~\ref{thm:convergence}] 
Let $\traj Z^k $ be the  $k$-th rectified flow %
induced from $(X_0,X_1)$. 
Then %
$$
\min_{k\in\{0\cdots K\}} S(\traj Z^k) \leq \frac{\E[\norm{X_1-X_0}^2]}{K}.  
$$
}

As shown Figure~\ref{fig:cat_k}, 
applying one step of reflow %
can already provide nearly straight flows that yield good performance when simulated with a single Euler step.  
It is not recommended to apply 
too  many reflow steps as it may accumulate  estimation error on $v^X$.

\paragraph{Distillation} 
After obtaining the $k$-th rectified flow $\vv Z^k$, we can further improve the inference speed by distilling the relation of $(Z_0^k, Z_1^k)$ into a neural network $\hat T$ to directly predict  $Z_1^k $ from $Z_0^k$ without simulating the flow. 
Given that the flow is already nearly straight (and hence well approximated by the one-step update), 
and  the distillation can be done efficiently. 
In particular, if we take $\hat T(z_0) = z_0 + v(z_0, 0)$, 
then the loss function for distilling $\vv Z^k$ is $\e{\norm{(Z_1^k - Z_0^k) - v(Z_0^k, 0)}^2}$, which is the term in \eqref{equ:mainf} when $t =0$.   %

We should highlight the difference between  distillation and  rectification:  %
 distillation attempts to faithfully approximate the coupling $(Z_0^k,Z_1^k)$ 
 while rectification  yields a different coupling $(Z_0^{k+1}, Z_1^{k+1})$ 
 with lower transport cost and more straight flow. 
 Hence, distillation should be applied only in the final stage when we want to fine-tune the model for fast one-step inference.

\paragraph{%
On the velocity field $v^X$}   
If  $X_0$ yields a conditional density function $\rho(x_0~|~x_1)$ when conditioned on $X_1 = x_1,$  %
then the optimal velocity field $v^X(z,t) = \E[X_1-X_0|X_t=z]$ can be represented by %
\bbb  \label{equ:vxformula}
v^X(z,t) = \e{\frac{X_1-z}{1-t} \eta_t(X_1, z)}, 
&&
\eta_t(X_1, z) = {\rho\left (\frac{z-t X_1}{1-t} ~\bigg |~X_1 \right )}\bigg  /{\e{\rho\left (\frac{z-t X_1}{1-t} ~\bigg |~X_1\right )}},
\eee 
where the expectation $\e{\cdot}$ is taken w.r.t. $X_1\sim \tg_1$. This can be seen by noting that $X_0 = \frac{z-t X_1}{1-t}$ and $X_1-X_0 = \frac{X_1-z}{1-t}$, when conditioned on $X_t = z$.
Hence, if  $\rho$ is positive and continuous everywhere, then $v^X$ is well defined and continuous on $\RR^d\times [0,1)$. 
Further, if $\log \eta_t$ is continuously differentiable w.r.t. $z$, we can show that 
$$
\dd_z v^X(z,t) = 
\frac{1}{1-t} \E\left [ 
\left (({X_1-z}) \dd_z \log \eta_t(X_1,z) - 1 \right ) \eta_t(X_1,z) 
\right]. 
$$
Note that $\d Z_t = v^X(Z_t,t)\dt $ 
is guaranteed to have a unique solution if 
$v^X$ is uniformly Lipschitz continuous on $[0,a]$ for any $a<1$.

If %
$X_0|X_1=x_1$ does not yield a conditional density function, 
$v^X(z,t)$ may be undefined or discontinuous, 
making the %
ODE $\d Z_t = v^X(Z_t,t)\dt$ 
ill-behaved. 
A simple fix is to add $X_0$ with a Gaussian noise $\xi \sim \normal(0, \sigma^2I)$ independent of $(X_0,X_1)$ to yield a smoothed variable $\tilde X_0 = X_0 +  \xi$, and transfer $\tilde X_0$ to $X_1$ using rectified flow.  
This would effectively give a randomized mapping of form $T(X_0 + \xi)$ transporting  $\tg_0$ to $\tg_1$. %

\paragraph{Smooth function approximation}
Following \eqref{equ:vxformula}, 
we can \emph{exactly} calculate $v^X$ 
if the conditional density function $\rho(\cdot|x_1)$ exists and is known, 
and $\tg_1$ is the empirical measure of a finite number of points (whose expectation can be evaluated exactly). 
In this case, running the rectified flow 
forwardly would 
precisely recover the points in $\tg_1$. 
This, however, is not practically useful in most cases as it completely overfits the data.
Hence, it is both necessary and beneficial 
to fit 
$v^X$ with a smooth function approximator such as neural network or non-parametric models, 
to obtain 
smoothed distributions  
with novel samples 
that  are %
practically useful.

Deep neural networks are no doubt the best function approximators for large scale problems. 
For low dimensional problems, 
the following simple Nadaraya–Watson style non-parametric estimator of  $v^X$ 
can yield a good approximation to the exact rectified flow 
without knowing
the conditional density $\rho$: 
\bbb \label{equ:npfunc}
v^{X,h}(z,t) 
=  \e{\frac{X_1-z}{1-t}\omega_h(X_t, z)},  %
\eee  
where $\omega_h(X_t,z) = \frac{\kappa_h(X_t, z)}{\e{\kappa_h( X_t,z)}} $, and  $\kappa_h(x,z)$ is a smoothing kernel with a bandwith parameter $h>0$ 
that measures the similarity between $z$ and $x$.  
Taking the Gaussian RBF kernel $\kappa_h(x, z) = \exp(-\norm{x-z}^2/2h^2)$, 
then when $h\to 0^+$, 
it can be shown that 
$v^{X,h}(z,t)$   converges to $v^X(z,t) = \e{\frac{X_1-z}{1-t} ~|~X_t=z}$
on points $z$ 
that can be attained by $X_t$ (i.e., the conditional expectation $\e{\cdot~|~X_t = z}$ exists. ). 
On points $z$ that $X_t$ can not attain, 
$v^{X,h}(z,t)$ extrapolates the value by finding the $X_t$ that is close to $z$. 
In practice, we replace the expectations in \eqref{equ:npfunc} with empirical averaging. 
We find that $v^{X,h}$ performs well in practice 
because it is a mixture of linear functions that always point to a point 
in the support of $\tg_1$.

\subsection{A Nonlinear Extension} %
\label{sec:nonlinear}

We %
present a nonlinear
  extension of rectified flow  
  in which the  linear interpolation  $X_t$  is replaced by any time-differentiable curve connecting $X_0$ and $X_1$. 
    Such generalized rectified flows can still  transport $\tg_0$ to $\tg_1$ (Theorem~\ref{thm:marginal}),  
    but no longer guarantee to decrease convex transport costs, or have the straightening effect. 
Importantly, 
the method of  probability flows \citep{song2020score} and  DDIM \citep{song2020denoising} can be viewed (approximately) as special cases of this framework, 
allows us to clarify the connection with and the advantages over these methods. 

Let 
$\X = \{X_t\colon t\in[0,1]\}$  be  any  time-differentiable 
random process that connects $X_0$ and $X_1$. Let $\dot X_t$ be the time derivative of $X_t$. 
The (nonlinear) rectified flow induced from $\vv X$ is defined as 
\bb 
\d Z_t = v^\X(Z_t,t)\dt, ~~~\text{with}~~~~ Z_0 = X_0, ~~~~   \text{and}~~~~ v^\X(z,t) = \e{\dot X_t ~|~ X_t =t}.%
\ee 
We can estimate $v^\X$ by solving 
\bbb \label{equ:Lgvphi}
\min_{v} %
\int_0^1 
\e{ w_t 
\norm{v(X_t, ~t) - \ddtdot X_t }^2}   \dt 
, 
\eee 
where %
$w_t \colon (0,1)\to (0,+\infty)$ is a positive weighting sequence ($w_t=1$ by default). 
When using the linear interpolation $X_t=tX_1 + (1-t)X_0$, we have $\dot X_t = X_1- X_0$ and 
\eqref{equ:Lgvphi} with $w_t=1$ reduces to \eqref{equ:mainf}. 
As we show in Theorem~\ref{thm:marginal}, 
the flow $\vv Z$ given by this method 
still preserves the marginal laws of $\vv X$, that is, $\law(Z_t)=\law(X_t)$, $\forall t\in[0,1]$, and hence  
$(Z_0,Z_1)$ remains to be a coupling of $\tg_0,\tg_1$. 
However, if $\vv X$ is not straight, 
$(Z_0,Z_1)$ no longer guarantees to decrease the convex transport costs over $(X_0,X_1)$. More importantly,  the reflow procedure no longer straightens the paths of $Z_t$.

A simple class of interpolation processes is $X_t = \alpha_t X_1 + \beta_t X_0$ where $\alpha_t$ and $\beta_t$ 
are two differentiable sequences that satisfy $\alpha_1=\beta_0=1$ and $\alpha_0=\beta_1 = 0$ to ensure that the process equals $X_0,X_1$ at the starting and end points. In this case,
we have $\dot X_t = \dot \alpha_t X_1 + \dot \beta_t X_0$ in \eqref{equ:Lgvphi} where $\dot \alpha_t$ and $\dot\beta_t$ are the time derivatives of $\alpha_t$ and $\beta_t$. 
The shape of the curve is controlled by the relation of $\alpha_t$ and $\beta_t$. If we take $ \beta_t =1-\alpha_t$ for all $t$, then  $X_t$ have  straight paths but does not travel at a constant speed; it can be viewed as a time-changed variant of the canonical case $X_t = t X_1 + (1-t)X_0$ when $t$ is reparameterized to $\alpha_t$. %
When $\beta_t\neq 1-\alpha_t$, the paths of $X_t$ 
are not straight lines except some special cases (e.g., $\dot\alpha_t X_1=0$ or $\dot \beta_t X_0 = 0$, or $X_1 = a X_1$ for some $a\in\RR$).

\subsubsection{Probability Flow ODEs and DDIM} 
The probability flow ODEs (PF-ODEs) \cite{song2020score}  and denoising diffusion implicit models (DDIM)  \cite{song2020denoising} 
are methods for learning  ODE-based generative models of $\tg_1$ from a spherical Gaussian initial distribution $\tg_0$, derived by converting a SDE learned by  
denoising diffusion methods to an ODE 
with equivalent marginal laws. 
In \cite{song2020score}, 
three types of PF-ODEs are derived from 
three types of SDEs learned as score-based generative models,  including 
variance-exploding (VE) SDE,
variance-preserving (VP) SDE,
and sub-VP SDE, 
which we denote by 
VE ODE, VP ODE, and sub-VP ODE, respectively. 
VP ODE is equivalent to the continuous time limit of DDIM, 
which is derived from the denoising diffusion probability model (DDPM) \cite{ho2020denoising}. 
As the derivations of PF-ODEs and DDIM 
require advanced tools in stochastic calculus, 
we limit our discussion on the final algorithmic  procedures suggested in \cite{song2020score, ho2020denoising}, which we summarize in Section~\ref{sec:diffusion}.  
The readers are referred to \cite{song2020score, song2020denoising} for the details.

\emph{[Proposition~\ref{pro:ddim}] 
All variants of PF-ODEs 
can be 
viewed as instances of \eqref{equ:Lgvphi}  
when using 
$X_t = \alpha_t X_1 +
\beta_t \xi$ for some $\alpha_t,\beta_t$ with $\alpha_1=1,\beta_1=0$, where 
$\xi \sim \normal(0, I)$ is a standard Gaussian random variable. 
}

Here we need to use introduce $\xi$ to replace $X_0$ because the choices of $\alpha_t$ and $\beta_t$  suggested in \cite{song2020score, song2020denoising, ho2020denoising} 
do not satisfy the boundary condition of $\alpha_0 = 0$ and $\beta_0 = 1$ at $t=0$, and hence $X_0 \neq \xi$. 
Instead,  
in these methods, the initial distribution $X_0\sim \tg_0$ is implicitly defined as $X_0 = \alpha_0  X_1 + \beta_0 \xi$, 
which is approximated by $X_0 \approx \beta_0 \xi $ by making $\alpha_0 X_1 \ll \beta_0 \xi$.   
Hence, $\tg_0$ is set to be  $\normal(0, \beta_0^2 I)$ in these methods. 
Viewed through our framework, 
there is no reason to restrict $\xi$ to be $\normal(0, \beta_0^2 I)$,  %
or not set $\alpha_0=0,\beta_0=1$ to avoid the approximation. 

\begin{figure}[h]
\begin{tabular}{wc{0.3\textwidth}|wc{0.3\textwidth}|wc{0.3\textwidth}}
{\small Rectified flow} 
& {\small VP ODE} 
& {\small sub-VP ODE} 
\\
 \begin{tabular}{c}
\small ($\alpha_t=t$, $\beta_t = 1-t$) 
\end{tabular}
&  \begin{tabular}{c}
\small  ($\alpha_t$ in \eqref{equ:vpode}, $\beta_t = \sqrt{1-\alpha_t^2}$) 
\end{tabular}
&  \begin{tabular}{c}
\small($\alpha_t$ in \eqref{equ:vpode}, $\beta_t = {1-\alpha_t^2}$)
\end{tabular}\\
 \includegraphics[width=.2\textwidth]{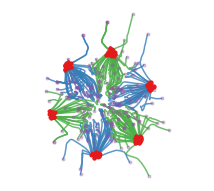} 
   \hspace{-2em} \includegraphics[width=.2\textwidth]{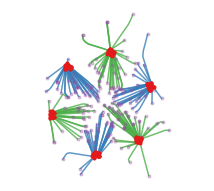} 
& \includegraphics[width=.2\textwidth]{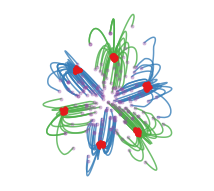} 
 \hspace{-2em}\includegraphics[width=.2\textwidth]{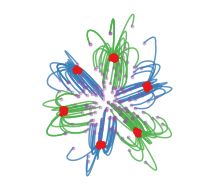} 
& \includegraphics[width=.2\textwidth]{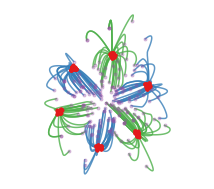} 
    \hspace{-2em}\includegraphics[width=.2\textwidth]{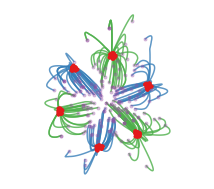}  
\\
{\scriptsize \reflow{1}} \hfill
  {\scriptsize \reflow{2}} 
& {\scriptsize \reflow{1}} \hfill
 {\scriptsize \reflow{2}} 
& {\scriptsize \reflow{1}} \hfill
 {\scriptsize \reflow{2}} 
\end{tabular}
\caption{
Comparing rectified flow with VP ODE and sub-VP ODE when $\tg_0 = \normal(0,I)$ (purple dots) and $\tg_1$ is a low variance Gaussian mixture shown as the red dots. 
The linear rectified flow yields nearly straight trajectories with one step of reflow. But the trajectories of VP ODE and sub-VP ODE are curved and can not be straightened by reflowing. 
} 
\label{fig:gauss2dots}
\end{figure}

\newcommand{\trimlen}{1.5cm}
\begin{figure}[h]
    \centering
 \resizebox{\textwidth}{!}{    
\begin{tabular}{c|c|c|c|c}   
Time-Discretization & 
Rectified Flow & 
VP ODE & 
sub-VP ODE & 
VP ODE (const speed) \\
Steps $N$ & 
{\small $\alpha_t = t, \beta_t = 1-t$} & 
{\small $\alpha_t$ in \eqref{equ:vpode}, $\beta_t =\sqrt{1-\alpha_t^2}$} & 
{\small $\alpha_t$ in \eqref{equ:vpode}, $\beta_t = {1-\alpha_t^2}$}  & 
{\small $\alpha_t=t$, $\beta_t = \sqrt{1-\alpha_t^2}$} \\ \hline
 \small  $N=1$ & 
    \includegraphics[width=.2\textwidth, trim={0cm 3cm 0 3cm},clip]{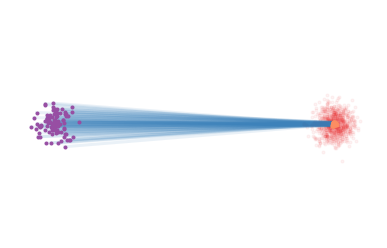} & 
    \includegraphics[width=.2\textwidth, trim={0cm 3cm 0 3cm},clip]{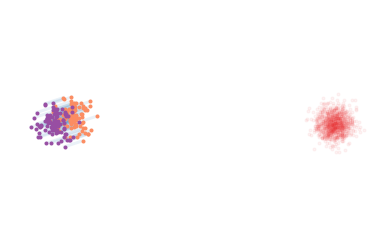} & 
    \includegraphics[width=.2\textwidth, trim={0cm 3cm 0 3cm},clip]{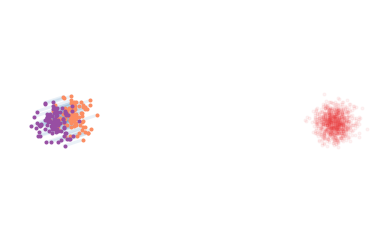} & 
    \includegraphics[width=.2\textwidth, trim={0cm 3cm 0 .4cm},clip]{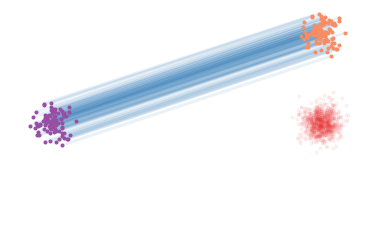}  \\  \hline  
    \small  $N=2$ &    
    \includegraphics[width=.2\textwidth, trim={0cm 3cm 0 3cm},clip]{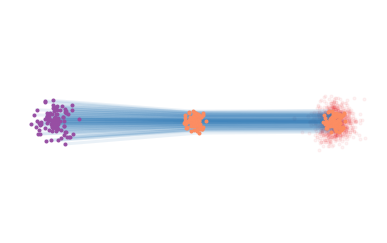} & 
        \includegraphics[width=.2\textwidth, trim={0cm 3cm 0 \trimlen},clip]{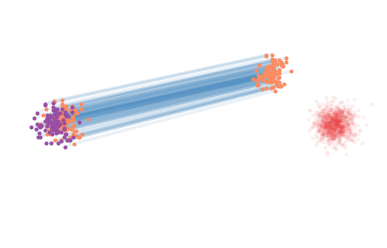} & 
        \includegraphics[width=.2\textwidth, trim={0cm 3cm 0 \trimlen},clip]{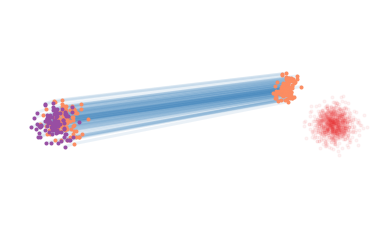} & 
    \includegraphics[width=.2\textwidth, trim={0cm 3cm 0 \trimlen},clip]{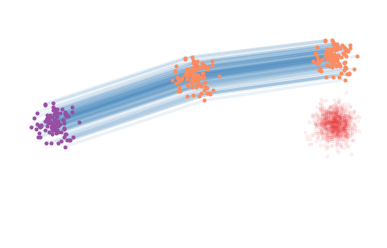}  \\ \hline 
\small  $N=5$ &     
    \includegraphics[width=.2\textwidth, trim={0cm 3cm 0 3cm},clip]{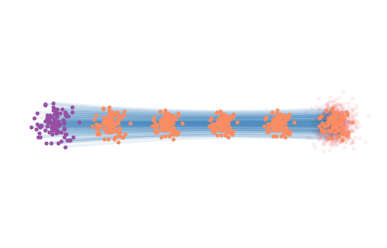} & 
    \includegraphics[width=.2\textwidth, trim={0cm 3cm 0 \trimlen},clip]{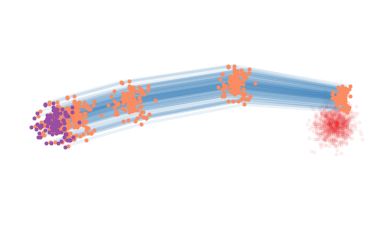} &     
    \includegraphics[width=.2\textwidth, trim={0cm 3cm 0 \trimlen},clip]{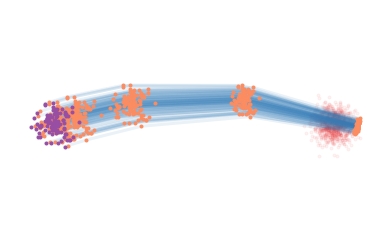} &     
    \includegraphics[width=.2\textwidth, trim={0cm 3cm 0 \trimlen},clip]{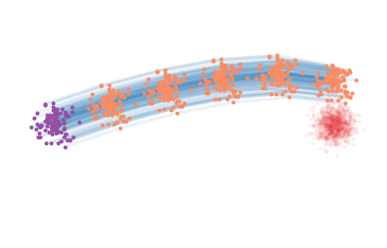}  \\   \hline   
\small  $N=100$ &    
    \includegraphics[width=.2\textwidth, trim={0cm 3cm 0 3cm},clip]{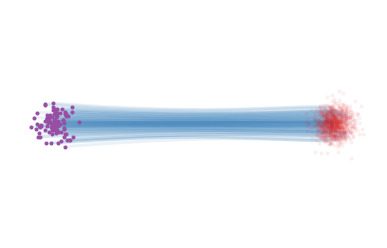} & 
    \includegraphics[width=.2\textwidth, trim={0cm 3cm 0 \trimlen},clip]{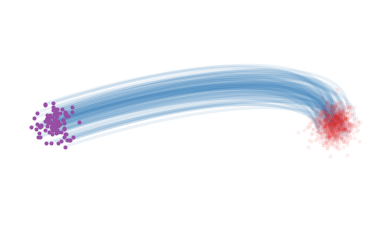} & 
    \includegraphics[width=.2\textwidth, trim={0cm 3cm 0 \trimlen},clip]{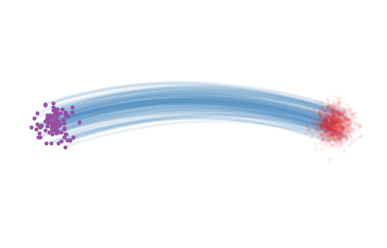}  & 
    \includegraphics[width=.2\textwidth, trim={0cm 3cm 0 \trimlen},clip]{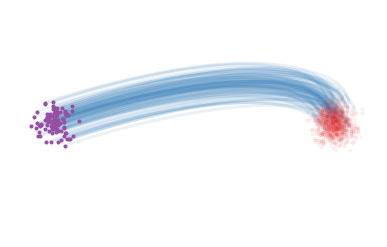}     
\end{tabular}  
}
    \caption{Trajectories of different methods when varying the number of discretization steps $N$ (purple dots: $\tg_0$; red dots: $\tg_1$; orangle dots: intermediate steps; blue curves: flow trajectories).  
    The rectified flow travels in straight lines and progresses uniformly in time; 
    it generates the mean of $\tg_1$ when simulated with a single Euler step, and quickly covers the whole distribution $\tg_1$ with more steps (in this case $N=2$ is sufficient). %
    In comparison, 
    VP ODE and sub-VP ODE travel in curves with non-uniform speed: they tend to be slow in the beginning and speed up in the later phase (much of the update happens when $t{ \scriptstyle\gtrapprox}0.5$). The non-uniform speed can be avoided by setting $\alpha_t =t $ (see the last column). 
    }
    \label{fig:speedtoy}
\end{figure}

\paragraph{VP ODE and sub-VP ODE} 
The VP ODE and sub-VP ODE 
of \cite{song2020score}
 use the following shared $\alpha_t$: 
\bbb \label{equ:vpode}
\text{(sub-)VP ODE:}~~~~
\alpha_t = \exp\left (-\frac{1}{4} a(1-t)^2 %
- \frac{1}{2}b (1-t) 
\right );  && \text{default values: $a=19.9$, $b=0.1$,}
\eee  
where the default values of  $a,b$ are %
 chosen to match the continuous time limit of the shared training procedure of DDIM and DDPM. 
 The difference of VP ODE and sub-VP ODE is on the choice of $\beta_t$, given as follows:
\bbb \label{equ:vpodebeta} 
\text{VP ODE:~~~}\beta_t = \sqrt{1-\alpha_t^2},  &&
\text{sub-VP ODE:~~~}
\beta_t = {1-\alpha_t^2}. 
\eee  

As $\beta_0 \approx 1$ in both VP and sub-VP ODE, 
the $\tg_0$ in both cases are taken as  $\normal(0,I)$. %

The choices 
of $\alpha_t,\beta_t$ 
above  are the consequence of the 
SDE-based derivation in \cite{song2020score}. %
However, 
they are not well-motivated when we exam the path properties  of the induced  
ODEs:  %

\emph{$\bullet$~Non-straight paths:} 
Due to choices of $\beta_t$ in  \eqref{equ:vpodebeta}, the trajectories of VP ODE and sub-VP ODE are curved in general,  
and can not be straightened by the reflow procedure. 
We should choose 
 $\beta_t = 1-\alpha_t$ to induce straight paths.

\begin{wrapfigure}{r}{0.4\textwidth}
  \vspace{-1\baselineskip}  
  \begin{center}
    \includegraphics[width=0.4\textwidth]{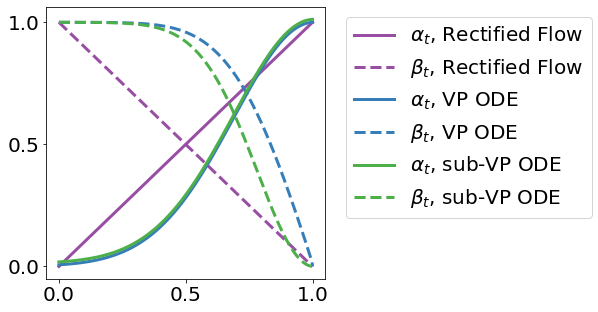}
  \end{center}
  \vspace{-.5\baselineskip}
  \caption{$t$ vs. $\alpha_t,\beta_t$ of different methods. 
  }
  \vspace{-.5\baselineskip}  
\end{wrapfigure}
\emph{$\bullet$~Non-uniform speed:} 
The exponential form of $\alpha_t$  in \eqref{equ:vpode} 
is a consequence of using  Ornstein–Uhlenbeck processes in the derivation of SDE models \cite{song2020score, ho2020denoising}.  
However, there is no clear advantage of using \eqref{equ:vpode} for  ODEs. 
As shown in Figure~\ref{fig:speedtoy},  
the $\alpha_t$ and $\beta_t$ of VP and sub-VP ODE 
change slowly in the early phase ($t{\scriptstyle \lessapprox} 0.5$).
As a result, 
the flow also moves slowly in beginning 
and hence most of the updates are concentrated in the later phase. 
Such non-uniform update speed, 
in addition to the non-straight paths,  
make VP ODE and sub-VP ODE perform sub-optimally 
when using large step sizes, even
for transport between simple spherical Gaussian distributions (see Figure~\ref{fig:speedtoy}).  
As we show in the last column of Figure~\ref{fig:speedtoy}, 
 changing the exponential $\alpha_t$ to the  linear function $\alpha_t = t$ in VP ODE 
 allows us to get a uniform update speed while 
 preserving the same continuous-time trajectories.

\paragraph{VE ODE} 
The VE ODE  of  \cite{song2020score} 
uses  
$\alpha_t = 1$ and 
$\beta_t = \sigma_{\min}\sqrt{r^{2(1-t)}-1}$   
where $\sigma_{\min} =0.01$ by default $r$ is set such that   $\sigma_{\max} \defeq r\sigma_{\min}$ is as large as the maximum Euclidean distance between all pairs of training data points from $\tg_1$  (Technique 1 of \cite{song2020improved}). 
Assume that %
$\sigma_{\max}^2$ 
is much larger than both $\sigma_{\min}^2$ and the variance of $X_1$,  
then  $X_0 = X_1 + \beta_0 \xi \approx \sigma_{\max}\xi $, and 
we can set the  initial distribution to be  $\tg_0\sim \normal(0,\sigma_{\max}^2 I)$, which has much larger variance than $\tg_1$. 
Hence, VE ODE can not be applied to (and not shown in) the toys in Figure~\ref{fig:gauss2dots} and Figure~\ref{fig:speedtoy}. 
As the case of (sub-)VP ODE, the restriction 
on $\xi$ is in fact unnecessary and requirement that $\sigma_{\max}$ is unnatural viewed from our framework. 
On the other hand, the trajectories of $X_t$ in VE ODE are indeed straight lines, 
because the direction of $\dot X_t = \dot \beta_t \xi$ is always the same as $\xi$. However, the choice of $\beta_t$ causes a  non-uniform speed issue similar to that of (sub-)VP ODE.  

Following \cite{song2020score, ho2020denoising}, a line of works have been proposed to improve the choices of $\alpha_t,\beta_t$, 
but remain to be constrained 
by the basic design space from the SDE-to-ODE derivation; see for example \cite{nichol2021improved, elucidating, zhang2022gddim}.

To summarize, the simple nonlinear rectified flow framework in \eqref{equ:Lgvphi} 
both simplifies and extends the existing framework, and sheds a number of importance insights: %

$\bullet$~Learning ODEs can be considered directly and independently without resorting to diffusion/SDE methods; 

$\bullet$~The paths of the learned ODEs can be specified by any smooth interpolation curve $X_t$ of $X_0$ and $X_1$;  

$\bullet$~The initial distribution $\tg_0$ can be chosen arbitrarily, 
independent with the choice of 
the interpolation 
$X_t$.  

 $\bullet$~The canonical linear interpolation  $X_t = t X_1 + (1-t) X_0$ should be recommended as a default choice. 
 
On the other hand, non-linear choices of $X_t$ can be useful %
when we want to incorporate certain non-Euclidan geometry structure of the variable, 
or want to place certain constraints on the  trajectories of the ODEs.  
We leave this for future works.

\section{Theoretical Analysis} 
\label{sec:theory}
We present the theoretical analysis for rectified flow. The results are summarized as follows.

$\bullet$~[Section~\ref{sec:marginal}] %
All nonlinear rectified flows with any interpolation $X_t$ preserve the marginal laws.  

$\bullet$~[Section~\ref{sec:cost}] The rectified flow (with the canonical linear interpolation) reduces convex transport costs.  %

$\bullet$~[Section~\ref{sec:straight}] %
Reflow guarantees to straighten the (linear) rectified flows. 

$\bullet$~[Section~\ref{sec:stcouplings}] %
We clarify the relation between 
straight couplings and $c$-optimal couplings. 

$\bullet$~[Section~\ref{sec:diffusion}] We %
establish  PF-ODEs as instances of nonlinear rectified flows.

\subsection{The Marginal Preserving Property}
\label{sec:marginal} 
The marginal preserving property that $\law(Z_t)=\law(X_t)$ for $\forall t$ is a general property of the nonlinear rectified flows in  \eqref{equ:Lgvphi},
regardless whether the interpolation $X_t$ is straight or not. 
\begin{mydef}
For a path-wise continuously differentiable random process  $\X = \{X_t\colon t\in[0,1]\}$, its expected velocity $v^\X$ is defined as 
$$v^\X(x, t) = 
\E[\dot X_t ~|~ X_t = x],~~~~~ \forall x\in \supp(X_t).
$$
For $x\not\in \supp(X_t)$, the conditional expectation is not defined and we set 
$v^\X$ arbitrarily, say $v^\X(x,t) = 0$. 
\end{mydef}

\begin{mydef}
We call that $\vv X$ is rectifiable if $v^\X$ is locally bounded and 
the solution of the integral equation below exists and  is unique: %
\bbb \label{equ:znonlinear} 
 Z_t = Z_0 + \int_0^t v^\X(Z_t, t) \dt,~~~~\forall t\in[0,1],~~~~ Z_0 = X_0.
\eee 
In this case, $\vv Z = \{Z_t\colon t\in[0,1]\}$ is called the rectified flow induced from $\vv X$. 
 \end{mydef}

\begin{thm}\label{thm:marginal}
Assume $\vv X$ is rectifiable and $\vv Z$ is its rectified flow. 
 Then  $\law(Z_t) =\law(X_t)$ for $\forall t\in[0,1]$. 
\end{thm}

\begin{proof} 
For any compactly supported continuously differentiable test function $h\colon \RR^d\to \RR$, we have 
\bbb  \label{equ:ehxt}
\ddt \E[h(X_t)]
 = \E[\dd h(X_t)\tt \dot X_t ] 
 = \E[\dd h(X_t)\tt v^\X(X_t,t)],  %
\eee  
where we used  $v^\X(X_t,t) = \E[\dot X_t |X_t]$.  
By definition, 
this %
is equivalent to that 
$\tg_t \defeq \law(X_t)$ solves 
in the sense of distributions the continuity equation with drift $v_t^\X \defeq v^\X(\cdot,t)$: 
\bbb \label{equ:contieq}
\dot \tg_t + \div (v_t^\X \tg_t) = 0.  
\eee 
To see the equivalence of \eqref{equ:ehxt} and \eqref{equ:contieq}, %
we can 
multiply \eqref{equ:contieq} with $h$ and integrate both sides:  
\bb 
0 = 
\int h(\dot \tg_t + \div (v_t^\X \tg_t)) 
= \int h \dot \tg_t - \dd h\tt v_t^\X \tg_t =  \ddt  \E[h(X_t) ]  - \E[\dd h(X_t)\tt v^\X(X_t,t)], 
\ee 
where we use integration by parts that $\int h\div (v^\X_t \tg_t)= - \int \dd h\tt (v^\X_t \tg_t)$. 

Because $Z_t$ is driven by the same velocity field $v^\X$, its  marginal law $\law(Z_t)$ solves the very same equation 
with the same initial condition ($Z_0=X_0$).  
Hence,  the equivalence of $\law(Z_t)$ and $\law(X_t)$ follows 
if the solution of \eqref{equ:contieq} is unique, 
which is equivalent to the uniqueness of the solution of $\d Z_t = \vofX(Z_t,t)$ following Corollary 1.3 of \citet{kurtz2011equivalence} (see also Theorem 4.1 of \citet{ambrosio2008existence}). 
\end{proof}

\subsection{Reducing Convex Transport Costs}
\label{sec:cost}
The fact that $(Z_0,Z_1)$ yields no larger convex transport costs than $(X_0,X_1)$ 
is a consequence of using the special linear interpolation $X_t = t X_1 + (1-t)X_0$ as the geodesic of Euclidean space. 

\begin{mydef}
A coupling $(X_0,X_1)$ is called rectifiable if its linear interpolation process $\vv X = \{tX_1+(1-t)X_0\colon t\in[0,1]\}$ is rectifiable. 
In this case, the 
$\vv Z=\{Z_t\colon t \in[0,1]\}$ in 
\eqref{equ:znonlinear} is called the 
rectified flow of coupling $(X_0,X_1)$, denoted as $\vv Z = \rectifyflow((X_0,X_1))$, 
and $(Z_0,Z_1)$ is called the rectified coupling of $(X_0,X_1)$, denoted as $(Z_0,Z_1) = \rectify((X_0,X_1)).$
\end{mydef}

\begin{thm} \label{thm:cost}
Assume $(X_0,X_1)$ is rectifiable and $(Z_0,Z_1) = \rectify((X_0,X_1))$. 
Then for any convex function $c\colon \RR^d\to \RR$, we have 
$$\E[c(Z_1 - Z_0)] \leq \E[c(X_1 - X_0)].$$ 
\end{thm} 
\begin{proof} 
The proof is based on elementary applications of Jensen's inequality. 
\bb
\E\left [c(Z_1 - Z_0)\right]
& = \E \left[ c\left(\int_0^1 \vofX(Z_t, t) \dt\right) \right ]  \ant{as $\d Z_t = \vofX(Z_t,t)\dt$}\\
&\leq \E \left[ \int_0^1 c \left (\vofX(Z_t, t)  \right) \dt \right ]  \ant{convexity of $c$, Jensen's inequality}\\
& = \E \left[ \int_0^1 c \left( \vofX(X_t, t)  \right) \dt \right ]  \ant{$X_t$ and $Z_t$ shares the same marginals}\\
& =  \E \left[ \int_0^1 c\left ( \E\left [(X_1-X_0)~|~X_t \right ] \right) \dt \right ] \ant{definition of $\vofX$}\\
& \leq  \E \left[ \int_0^1  \E\left [ c\left(X_1-X_0 \right) |~X_t \right ]   \dt \right ] \ant{convexity of $c$, Jensen's inequality}\\
& = \int_0^1\E\left [c\left(X_1 - X_0\right)\right] \dt  \ant{$\E[\E[(X_1-X_0) | X_t]] = \E[(X_1-X_0)]$}\\
& = \E\left [c\left (X_1 - X_0 \right) \right]. 
\ee 
\end{proof} 
 If $X_t$ is straight but with positive non-constant speed, that is, 
$X_t = \alpha_t X_1 + \beta_t X_0$ with $\beta_t = 1-\alpha_t$ and $\dot \alpha_t\geq0$, 
then we still have $\E[c(Z_1-Z_0)] \leq \E[c(X_1-X_0)]$ 
if $c$ is convex and $m$-homogeneous in that $c(a x) = \abs{a}^m c(x)$ 
for $\forall a\in \RR, x\in \RR^d$, with 
some constant $m\in(0, 1]$.

\subsection{The Straightening Effect}
\label{sec:straight}

A coupling $(X_0,X_1)$ 
is said to be straight 
(or fully rectified)  if 
it is a fixed point of the $\rectify(\cdot)$ mapping.
It is desirable to obtain a straight coupling because its rectified flow is straight and 
hence can be simulated exactly with one step using numerical solvers. 
In this section, 
we first characterize the basic properties of straight couplings, showing that a coupling is straight iff its linear interpolation paths do not intersect with each other. Then, we prove that recursive rectification 
straightens the coupling and its related flow 
with a $\bigO{1/k}$ rate, where $k$ is the number of rectification steps. %

\begin{thm}\label{thm:straightness} 
Assume $(X_0,X_1)$ is rectifiable. Let $X_t =  tX_1 +(1-t)X_0$ and 
$\vv Z =\rectifyflow((X_0,X_1))$. 
Then $(X_0,X_1)$ is a straight coupling iff the following equivalent statements hold. %
\begin{enumerate} 
\item There exists a strictly convex function $c\colon \RR^d\to \RR$, such that $\E[c(Z_1-Z_0)] = \E[c(X_1-X_0)]$. 
\item $(X_0,X_1)$ is a fixed point of $\rectify(\cdot)$, that is, $(X_0,X_1) = (Z_0,Z_1)$. %
 
\item The rectified flow coincides with the linear interpolation process: $\vv X = \vv Z$. 

\item The paths of the linear interpolation $\vv X$ do not intersect: %
\bbb 
\hspace{-.05\textwidth}
V((X_0,X_1)):= 
\int_0^1 \e{\norm{X_1-X_0-\e{X_1-X_0~|~X_t}}^2} \dt = 0, 
\eee 
where %
 $V((X_0,X_1))=0$ 
 indicates 
 that $X_1-X_0=\E[X_1-X_0|X_t]$ almost surely when $t\sim \uniform([0,1])$, 
 meaning that the lines passing through each $X_t$ is unique, and hence no linear interpolation paths intersect. 
\end{enumerate}
\end{thm}
\begin{proof}

$3 \to 2\to 1$: Obvious.

$1\to 4$: 
If $\E[c(Z_1-Z_0)] = \E[c(X_1-X_0)]$, the two applications of 
Jensen's inequality in the proof of Theorem~\ref{thm:cost}
are tight. Because $c$ is strictly convex,
the second Jensen's inequality in the proof implies that 
$X_1-X_0 = \E[X_1-X_0~|~X_t]$ almost surely w.r.t. $X$ and $t\sim \uniform([0,1]),$ which implies that $V(\vv X)=0.$

$4\to 3$:   
If $V(\vv X) =0$, we have $\int_0^s( X_1 -X_0) \dt = \int_0^s\E[X_1-X_0|X_t]\dt = \int_0^sv^X(X_t,t)\dt $ for $s\in(0,1]$.
Hence 
\bb 
X_t = X_0 +\int_0^t (X_1-X_0) \dt 
= X_0 + \int_0^t v^X(X_t, t) \dt. 
\ee 
Because $\vv Z$ satisfies the same equation 
\eqref{equ:znonlinear}, we have $\vv X=\vv Z$ by the uniqueness of the solution.

\end{proof}

\paragraph{$\bigO{1/K}$ convergence rate}  
We now show that as we apply rectification recursively, 
the rectified flows become increasingly straight and the linear interpolation of the couplings becomes increasingly 
non-intersecting. 
\begin{thm} \label{thm:convergence} 
 Let $\vv Z^k$ the $k$-th rectified flow    
 of $(X_0,X_1)$, 
that is, $\vv Z^{k+1} = \rectifyflow((Z_0^k, Z_1^k))$ and $(Z_0^0,Z_1^0) = (X_0,X_1)$. 
Assume each $(Z_0^k,Z_1^k)$ is rectifiable for $k=0,\ldots, K$.  

Then %
\bb 
\sum_{k=0}^K 
S(\vv Z^{k+1}) + V( (Z^k_0, Z^k_1))
\leq  \e{\norm{X_1-X_0}^2}.
\ee 
 \end{thm}
 Hence, $\E[\norm{X_1-X_0}^2]<+\infty$,  we have 
 $\min_{k\leq K} (S(\vv Z^k) + V((Z_0^k,Z_1^k)) = \bigO{1/K}$.

\begin{proof}
Taking  $c(x) = \norm{x}^2$ in 
the proof of Theorem 3.5, we can obtain that 
\bbb  \label{diff}
\e{\norm{X_1-X_0}} - \e{\norm{Z_1-Z_0}} = S(\vv Z) + V((X_0,X_1)).
\eee  
Applying it to each rectification step yields
$$
\E\left [\norm{Z^k_1-Z^k_0}^2\right ] - \E\left [\norm{Z^{k+1}_1-Z^{k+1}_0}^2\right ] = S(\vv Z^{k+1}) + V((Z_0^k,Z_1^k)). 
$$
A telescoping sum on $k=0,\ldots, K$ gives the result. 

\end{proof}

\subsection{Straight vs. Optimal Couplings} 
\label{sec:stcouplings}

A coupling $(X_0,X_1)$ is called $c$-optimal if it achieves the minimum of $\E[c(X_1-X_0)]$ among all couplings that share the same marginals.   
Understanding and computing the optimal couplings 
have been the main focus of optimal transport \cite[e.g.,][]{villani2009optimal,ambrosio2021lectures, figalli2021invitation, peyre2019computational}. 
Straight couplings is a different desirable property. %
In the following, 
we show that straightness is a necessary but not sufficient condition of being $c$-optimal for a strictly convex function $c$, except in the one dimensional case when the two concepts coincides. 
Hence, it is ``easier'' to find a straight coupling than a $c$-optimal couplings. 

\begin{thm}\label{thm:straightOpt}
If a rectifiable  coupling $(X_0,X_1)$ 
is $c$-optimal for some strictly convex cost function $c$, then $(X_0,X_1)$ is a straight coupling. 
\end{thm}
\begin{proof}
Let $(Z_0,Z_1) = \rectify((X_0,X_1))$. 
If $(X_0,X_1)$ is $c$-optimal, we must have $\E[c(Z_1-Z_0)] = \E[c(X_1-X_0)]$. 
This implies Statement 1 in Theorem \ref{thm:straightness} and hence that $(X_0,X_1)$ is straight. 
\end{proof}

\paragraph{1D Case}  
For any $\tg_0,\tg_1$ on $\RR$, there exists an unique coupling $(X_0^*,X_1^*)$
that is simultaneously optimal for all non-negative convex cost functions $c$. This coupling is uniquely characterized by a monotonic property: for every $(x_0,x_1)$ and $(x_0',x_1')$  in the support of $(X_0^*, X_1^*)$, 
if  $x_0< x_0'$, then $x_1 \leq x_1'$. 
Furthermore, if $\tg_0$ is absolutely continuously w.r.t. the Lebesgue measure, 
then $(X_0^*, X_1^*)$ must be deterministic in that there exists a mapping $T\colon \RR\to\RR$ such that $X_1^* = T(X_0^*)$. 
See \cite{santambrogio2015optimal}.   

In the following, we show that straight couplings on $\RR$ coincides with the deterministic monotonic coupling $(X_0^*, X_1^*)$ and hence is unique and simultaneously optimal for all convex $c$ when $\tg_0$ is absolutely continuous.
The idea is that, in $\RR$, a coupling is monotonic 
iff 
its linear interpolation paths do not intersect, a characteristic feature of straight couplings. 
\begin{lem}\label{thm:monostraightlemma}
A coupling on $\RR$ is  straight iff it is deterministic and monotonic. 
\end{lem} 
\begin{thm}\label{thm:1dstraight} 
For any $\tg_0,\tg_1$ on $\RR$, 
there exists either no straight coupling,
or a unique straight coupling. 
Further, if exists, the unique straight coupling is  deterministic and monotonic,  and jointly optimal w.r.t. all  convex cost functions $c\colon \RR^d\to [0,+\infty)$ 
for which the minimum value of $\e{c(X_1-X_0)}$ exists and is finite.  
\end{thm} 
\begin{proof}[Proof of Lemma~\ref{thm:monostraightlemma}] 
If $(X_0,X_1)$ on $\RR$ is straight, 
then it coincides with its rectified coupling $(Z_0,Z_1) = \rectify((X_0,X_1)).$ 
But because $(Z_0,Z_1)$ is induced from the rectified flow $\d Z_t = v^X(Z_t, t) \dt $, it is obviously deterministic. It is also monotonic due to the non-crossing property of flows.  Specifically, 
if $(Z_0,Z_1)$ is not monotonic, there exists $(z_0,z_1)$ and $(z_0',z_1')$ in the support of $(Z_0,Z_1)$ such that $z_0<z_0'$ and $z_1 > z_1'$. If this happens, there must exists $t_0\in(0,1)$, such that $z_{t_0} = z_{t_0}'$. But by the uniqueness of the solution, we have $z_t=z_t$ for $t\geq t_0$, which is conflicting with $z_1 > z_1'$. 

Assume $(X_0,X_1)$ is deterministic and monotonic. %
Due to the monotonicity, there exists no $x_0$ and $x_0'$ in the support of $\tg_0$, such that $x_0 \neq x_0'$ and $x_{t_0} = x_{t_0}'$ for some $t_0<1$.  This suggests that $X_1-X_0 = \E[X_1-X_0~|~X_t] = v^X(X_t,t)$ for $t\in(0,1)$, and hence  $\d  X_t = (X_1-X_0)\dt = v^X(X_t) \dt $, which is the ODE of the rectified flow. In addition, $X_t$ is obviously the unique solution of this ODE.  Hence $(X_0,X_1)$ is rectifiable and straight following Statement 3 of Theorem~\ref{thm:straightness}.  
\end{proof} 

\begin{proof}[Proof of Theorem~\ref{thm:1dstraight}]
This is the result of Lemma~\ref{thm:monostraightlemma} combined with the fact that the monotonic coupling is unique and jointly optimal for all convex $c$ for which the optimal coupling exists, following 
Lemma 2.8 and 
Theorem 2.9 of \cite{santambrogio2015optimal}. 
\end{proof}

\paragraph{Multi-dimensional cases} 
On the other hand, 
on $\RR^d$ with $d\geq 2$, 
the different cost functions $c$  do not share a common optimal coupling in general, 
and a straight coupling is not guaranteed to optimize 
a specific $c$; this is expected because the $\rectify(\cdot)$ procedure does not depend on a particular choice of $c$.
Hence, one must modify the $\rectify(\cdot)$ procedure 
to tailor it to a specific $c$ of interest.

In a recent work 
\cite{khrulkov2022understanding}, 
it was  conjectured that the couplings $(Z_0,Z_1)$ induced from VP ODE (equivalently DDIM) yields an optimal coupling w.r.t. the quadratic loss, which was proved to be false in  \cite{lavenant2022flow, tanana2021comparison}. 
Here we show that 
even straight couplings 
are not guaranteed to be optimal, not to mention that VP ODE does not follow straight paths by design. 

We explore this in a separate work  \cite{rectifyOT} that is devoted to
modifying rectified flow to  find 
$c$-optimal couplings;
a result 
from \cite{rectifyOT} 
that can be easily stated 
is that 
the optimal coupling 
w.r.t. the quadratic cost $c(\cdot)=\norm{\cdot}^2$ 
can be achieved as the fixed point of $\rectify(\cdot)$ 
if $v$ is restricted to be a gradient field of form $v(x,t) = \dd f(x,t)$ 
when solving \eqref{equ:mainf}. Restricting $v$ to be a gradient field removes the rotational component of the velocity field $v^\X$ that causes sub-optimal transport cost.

\subsection{Denoising Diffusion Models and Probability Flow ODEs}  
\label{sec:diffusion}
We prove that the probability flow ODEs (PF-ODEs) of \cite{song2020score} can be viewed as  nonlinear rectified flows in \eqref{equ:Lgvphi} with $X_t = \alpha_t X_1 + \beta_t \xi.$ 
We start with introducing the algorithmic procedures of the denoising diffusion models and PF-ODEs, and refer the readers to the original works \cite{song2020score, ho2020denoising, song2020denoising} for the theoretical derivations. 

The 
denoising diffusion methods %
learn to %
generative models by constructing  
an SDE model driven by a standard Brownian motion $W_t$: 
\bbb \label{equ:ztvdw} 
\d U_t = b(U_t, t) \dt + \sigma_t \d W_t, ~~~ U_0 \sim \tg_0,
\eee  
where $\sigma_t\colon [0,1]\to [0,+\infty)$ 
is a (typically) fixed diffusion
coefficient, $b$ is a trainable neural network, 
and the initial distribution $\tg_0$ 
is restricted to a spherical Gaussian distribution determined by hyper-parameter setting of the algorithm.  
The idea is to first collapse the data into an (approximate) Gaussian distribution using a diffusion process, mostly an Ornstein-Uhlenbeck (OU) process, and then estimate 
the generative diffusion process \eqref{equ:ztvdw} as the time reversal \citep[e.g.,][]{anderson1982reverse} of the collapsing process. 

Without diving into the derivations, 
the training loss of the VE, VP, sub-VP SDEs 
for $b$ in \cite{song2020score} can be summarized as follows:  
\bbb \label{equ:oud} 
\min_{v}\int_0^1 \e{w_t \norm{v(V_t,t) -Y_t}^2_2} \dt,
&& V_t = \alpha_t X_1 + \betatogamma_t \xi_t, && 
Y_t = - \eta_{t} V_t  
- \frac{\sigma_t^2 }{\betatogamma_t}\xi_t,  
\eee 
where $\xi_t$
is a diffusion process satisfying $\xi_t \sim \normal(0,I)$, 
and $\eta_t, \sigma_t$ are the hyper-parameter sequences
of the algorithm, and $\alpha_t, \betatogamma_t$ are determined by $\eta_t, \sigma_t$ via %
\bbb \label{equ:alphadiff}
 \alpha_t = \exp\left(\int_t^1 \eta_s \d s \right), ~~~~~ \betatogamma_t^2 = \int_t^1 \exp\left (2\int_t^s \eta_r \d r \right)  \sigma_{s}^2 \d s. 
\eee
The relation in \eqref{equ:alphadiff} 
is derived to make $\tilde V_t = V_{1-t}=\alpha_{1-t} X_1 + \beta_{1-t} \xi_t$ follow the Ornstein-Uhlenbeck (OU) processes $\d \tilde V_t = \eta_{1-t} \tilde V_t \dt + \sigma_{1-t} \d W_t$. 

VE SDE, which is equivalent to SMLD in \cite{song2019generative, song2020improved}, takes $\eta_t=0$ and hence has $\alpha_t = 1$. 
(sub-)VP SDE takes $\eta_s$ to be a linear function of $s$, yielding the exponential $\alpha_t$ in \eqref{equ:vpode}. 
VP SDE (which is equivalent to DDPM \cite{ho2020denoising})  
takes $\eta_t = - \frac{1}{2} \sigma_t^2$ %
which yields that 
$
\alpha_t^2 + 
\betatogamma_t^2 =1$ as shown in \eqref{equ:vpodebeta}. 
In DDPM, it was suggested to write $b(x,t) = -\eta_t x - \frac{\sigma_t^2}{\betatogamma_t} \epsilon(x,t)$ 
, and estimate $\epsilon$ as a neural network that predicts 
$\xi_t$  from $(V_t,t)$.

Theoretically, the SDE in \eqref{equ:ztvdw}  with $b$ solving \eqref{equ:oud} 
is ensured to yield $\law(U_1) = \law(X_1) = \tg_1$  
when  initialized from $U_0 = \alpha_0X_1 + \betatogamma_0 \xi_0$, which can be approximated by $U_0 \approx \betatogamma\xi_0$ when $\alpha_0 X_1 \ll \betatogamma_0 \xi_0$. 

By using the properties of Fokker-Planck equations, 
it was observed in \cite{song2020score, song2020denoising}
that the SDE in \eqref{equ:ztvdw} with $b$ trained in \eqref{equ:oud} can be converted into an ODE that share the same marginal laws:
\bbb \label{equ:xtddpmbb} 
\d Z_t = \tilde b(Z_t, t)
\dt,~~~~\text{with}~~~~
\tilde b(z,t) = \frac{1}{2} (b(z,t) - \eta_t z), 
~~~~\text{starting from~~} Z_0 = U_0 = \alpha_0 X_1 + \beta_0 \xi_0. 
\eee 
Equivalently, we can regard $\tilde b$  as the solution of 
\bbb \label{equ:odeobjgg}
\min_v \int_0^1 \e{w_t \norm{v(V_t, t) - \tilde Y_t}^2_2} \dt, 
&& V_t = \alpha_t X_1 + \betatogamma_t \xi_t, && 
\tilde Y_t = - \eta_{t} V_t  
- \frac{\sigma_t^2 }{2\betatogamma_t}\xi_t,
\eee 
which defers from \eqref{equ:ztvdw} only by a factor of $1/2$ in the second term of $Y_t$. 
This simple equivalence holds only when  \eqref{equ:ztvdw} and \eqref{equ:xtddpmbb} use the special initialization of $
Z_0 = U_0 = \alpha_0 X_1 + \beta_0 \xi_0$.

In the following, we are ready to prove that %
\eqref{equ:odeobjgg} is can be viewed as %
the nonlinear rectified flow objective in  \eqref{equ:Lgvphi} 
using $X_t = \alpha_t X_1 +
\betatogamma_t \xi$ with 
$\xi \sim \normal(0, I)$. 
We mainly need to show that $\tilde Y_t$ 
is equivalent to $\dot X_t$ by eliminating $\eta_t$ and $\sigma_t$ using the relation in \eqref{equ:alphadiff}.

\begin{pro}\label{pro:ddim}
Assume \eqref{equ:alphadiff} hold. %
Then \eqref{equ:odeobjgg} is equivalent to \eqref{equ:Lgvphi} with $X_t = \alpha_t X_1 + \betatogamma_t \xi$.  
\end{pro} 
\begin{proof} 
First, note that we can take $\xi_t = \xi$ for all time $t$, as the correlation structure of $\xi_t$ does not impact the result. 
Hence, we have $V_t = X_t = \alpha_t X_1 + \beta_t \xi$. 
To show the equivalence of \eqref{equ:odeobjgg} and \eqref{equ:Lgvphi}, 
we just need to verify that $\dot X_t = \tilde Y_t$. %
\bb 
\tilde Y_t 
& =
 -\eta_t X_t + \frac{\sigma_t^2}{2\betatogamma_t^2} (
\alpha_t X_1 - X_t) \\ 
& = -\dot \eta_t \left ( \alpha_t X_1 + \betatogamma_t \xi \right) + \frac{\sigma_t^2}{2\betatogamma_t}  \xi \\ 
& = - \dot \eta \alpha_t X_1 
+ %
\left ( - \dot \eta_t \betatogamma_t + \frac{\sigma_t^2}{2\betatogamma_t} \right )
\xi \\
& \overset{(*)}{=} \dot \alpha_t X_1 +  \dot \betatogamma_t \xi \\
& = \dot X_t. 
\ee 
where in $\overset{(*)}{=}$ we used that $\eta_t = - \frac{\dot \alpha_t}{\alpha_t}$ and $ {\sigma_t^2} = 2\betatogamma_t^2  \left (\frac{\dot\alpha_t}{\alpha_t} -\frac{\dot \betatogamma_t}{\betatogamma_t} \right )$ which can be derived from \eqref{equ:alphadiff}. 
\end{proof}

\section{Related Works and Discussion}

\paragraph{Learning one-step models}
GANs \citep{goodfellow2014generative, arjovsky2017wasserstein, liu2021fusedream}, VAEs \citep{kingma2013auto}, and (discrete-time) normalizing flows  \citep{rezende2015variational, dinh2014nice,dinh2016density}
have been three classical approaches for learning deep generative models. %
GANs have been most successful in terms of  generation qualities (for images in particular), but suffer from the notorious training instability and mode collapse issues due to use of minimax updates. 
  VAEs and normalizing flows are both trained based on the principle of maximum likelihood estimation (MLE) and  need to introduce constraints on the model architecture and/or special approximation techniques to ensure tractable likelihood computation: 
  VAEs typically use a conditional Gaussian distribution in addition to the variational approximation of the likelihood; 
  normalizing flows require to use specially designed invertible architectures and need to copy with calculating expensive Jacobian matrices.   
 
 The reflow+distillation approach in this work provides another promising approach to  training one-step models,  avoiding the minimax issues of GANs and the intractability issues of the likelihood-based methods. 

\paragraph{Learning ODEs: MLE and PF-ODEs} %
There are two major approaches for learning neural ODEs: 
the PF-ODEs/DDIM approach discussed in Section~\ref{sec:nonlinear}, and the more classical MLE based approach
 of \cite{chen2018neural}. 

\emph{$\bullet$~The MLE approach.}  
In \cite{chen2018neural}, 
neural ODEs are trained 
for learning generative models by maximizing the likelihood of the distribution of the ODE outcome $Z_1$ at time $t=1$ under the data distribution $\tg_1$. Specifically, 
with observations from $\tg_1$, it estimates a neural drift $v$ of an ODE  $\d Z_t = v(Z_t,t)\dt $ by 
\bbb \label{equ:neuralode}
\max_{v}  \mathbb{D}(\tg_1;~~\rho^{v,\tg_0}), 
\eee 
where $\mathbb{D}(\cdot;~\cdot)$ denotes 
KL divergence (or other discrepancy measures),  and $\rho^{v,\tg_0}$ is the density of $Z_1$ following $\d Z_t = v(Z_t,t)\dt $ from $Z_0\sim \tg_0$; the density of $\tg_0$ should be known and tractable to calculate. 

By using  
an instantaneous change of variables formula, 
it was observed in \cite{chen2018neural} that the likelihood of neural ODEs
are easier to compute than the discrete-time normalizing flow without constraints on the model structures. 
However, this MLE approach is still computationally expensive for large scale models as it requires repeated simulation of the ODE during each training step.
In addition, 
as the optimization procedure of MLE requires to  backpropagate through time,
it can easily suffer the gradient vanishing/exploding problem unless proper regularization is added. 

Another fundamental problem is that 
the MLE \eqref{equ:neuralode}  of neural ODEs is theoretically under-specified, 
because MLE only concerns matching the law of the final outcome $Z_1$ with the data distribution $\tg_1$, and  there are infinitely many ODEs to achieve the same output law of $Z_1$ while traveling through different paths. 
A number of works have been proposed to remedy  this by adding 
regularization terms, such as these based on transport costs, to favor shorter paths; see  \cite[][]{nichol2021improved, onken2021ot}. 
With a regularization term, the ODE learned by MLE would be implicitly determined by the initialization and other hyper-parameters of the optimizer used to solve \eqref{equ:neuralode}.  

\emph{$\bullet$~Probability Flow ODEs.}  
The method of PF-ODEs \cite{song2020score} and DDIM \cite{song2020denoising} provides a different approach to learning ODEs that avoids the main disadvantages of the MLE approach, including %
the expensive likelihood calculation, training-time simulation of the ODE models, and the need of backpropagation through time.  %
However,
because PF-ODEs and DDIM 
were derived as the side product of learning the mathematically more involved diffusion/SDE models, 
their theories and algorithm forms were made 
unnecessarily restrictive and complicated. 
The nonlinear rectified flow framework shows that the 
learning of ODEs 
can be approached directly in a very simple way, allowing us to identify the canonical case of linear rectified flow  and open the door of further improvements with flexible and decoupled choices of the interpolation curves $X_t$ and initial distributions $\tg_0.$

 Viewed through the general non-linear rectified flow  framework, 
 the computational and theoretical drawbacks of MLE can be avoided because 
we can simply pre-determines the ``roads'' 
that the ODEs should travel through 
by specifying the interpolation curve $X_t$,  
rather than leaving it for the algorithm to figure out implicitly.  
It is theoretically valid to pre-specify any interpolation  $X_t$ 
because the neural ODE is highly over-parameterized as a generative model: when $v$ is a universal approximator and $\tg_0$ is absolutely continuous, 
the distribution of $Z_1$ can  approximate any distribution given any fixed interpolation curve $X_t$. The idea of rectified flow is to the 
simplest geodesic paths for $X_t$. %

\paragraph{Learning SDEs with denoising diffusion} 
Although the scope of this work 
is limited to learning ODEs, the  score-based generative models \citep{song2019generative, song2020improved, song2020score, song2021maximum}
and denoising diffusion probability models (DDPM) \citep{ho2020denoising} 
are of high relevance 
as the basis of PF-ODEs and DDIM. 
The diffusion/SDE models trained with these methods have been found outperforming GANs in image synthesis 
in both quality and diversity \cite{dhariwal2021diffusion}. 
Notably,  
thanks to the stable and scalable optimization-based training procedure, 
the diffusion models have successfully used in  huge text-to-image  generation models with astonishing results 
\citep[e.g.,][]{glide, dalle2, imagegen}.  
It has been quickly popularized in other domains, such as
video \citep[e.g.,][]{ho2022video, yang2022diffusion, harvey2022flexible}, 
music \citep{mittal2021symbolic}, audio \citep[e.g.,][]{kong2020diffwave, lee2021nu, popov2021grad}, and text \citep{li2022diffusion, wang2022language}, 
and more tasks such as image editing  \citep{zhao2022egsde, meng2021sdedit}. 
A growing literature has been developed for improving the inference speed of denoising diffusion models, 
an example of which is 
the PF-ODEs/DDIM approach which gains speedup by turning SDEs into ODEs. 
We provide  below some examples of recent works,
which is by no mean  exhaustive. 

\emph{$\bullet$~Improved training and inference.} 
A line of works focus on  improving the inference and sampling procedure of denoising diffusion models. 
For example, 
\cite{nichol2021improved} presents 
a few simple modifications of DDPM to improve the likelihood, sampling speed, and generation quality. 
\cite{elucidating}  systematic exams  the design space of diffusion generative models 
with empirical studies and identifies a number of training and inference recipes 
for better generative quality with fewer sampling steps. 
\cite{zhang2022fast} proposes a diffusion exponential integrator sampler for fast sampling of diffusion models. \cite{lu2022dpm} provides a customized high order solver for PF-ODEs. 
\citep{bao2022analytic} provides an analytic estimate of the optimal diffusion coefficient.

\emph{$\bullet$~Combination with other methods.}
 Another direction is to speed up diffusion models by combining them with GANs and other generative models. 
DDPM Distillation~\citep{luhman2021knowledge}
accelerates the inference speed 
by distilling the trajectories of a diffusion model into a series of conditional GANs. 
The truncated diffusion probabilistic model (TDPM) of \citep{zheng2022truncated}
trains a GAN model as $\pi_0$ so that the diffusion process can be truncated to improve the speed; 
the similar idea was explored in \cite{lyu2022accelerating, franzese2022much}, 
and \citep{franzese2022much} provides an analysis on the optimal truncation time.  
\citep{sinha2021d2c, wehenkel2021diffusion, vahdat2021score} learns a denoising diffusion model in the latent spaces and combines it with variational auto-encoders. 
These methods can be potentially applied to rectified flow to gain similar speedups for learning neural ODEs.   

\emph{$\bullet$~Unpaired Image-to-Image translation.} 
The standard denoising diffusion and PF-ODEs methods focus on the generative task of transferring a Gaussian noise ($\tg_0$) to the data ($\tg_1$). A number of works have been proposed to 
adapt it to transferring data between arbitrary pairs of source-target domains. 
For example, 
SDEdit \cite{meng2021sdedit} 
synthesizes realistic images 
guided by an input image by 
first adding noising to the input and then denoising the resulting image through 
a pre-trained SDE model. 
\cite{choi2021ilvr} proposes 
a method to guide the generative process of DDPM to generate realistic images based on a given reference image.  
\cite{su2022dual} leverages two 
two PF-ODEs for image translation, 
one translating source images to a latent variable, and the other constructing  the target images from the latent variable.  
\cite{zhao2022egsde} proposes 
an energy-guided approach that employs an energy function pre-trained on the source and target domains to guide the inference process of a pretrained SDE for better image translation. 
In comparison, 
our framework shows that domain transfer 
can be achieved by essentially the same algorithm as generative modeling, by simply setting $\tg_0$ to be the source domain. %

\emph{$\bullet$~Diffusion bridges.}
Some recent works \cite{peluchetti2021non, bridge} show that the design space of denoising  diffusion models can be made highly  flexible with the assistant of diffusion bridge processes that are pinned to a fixed data point at the end time. This reduces the design of denoising diffusion methods to 
constructing a proper bridge processes. 
The bridges in 
\citet{song2020score} are  
constructed by a time-reversal technique, which can be equivalently achieved by Doob's $h$-transform as shown in \cite{peluchetti2021non, bridge}, 
and more general construction techniques are discussed in \cite{bridge, geobridge}.  
Despite the significantly extended  design spaces, an unanswered question %
is 
what type of diffusion bridge processes should be preferred. 
This question is made challenging because the presence of diffusion noise and the need of advanced stochastic calculus tools  make it hard to intuit  how the methods work.
By removing the diffusion noise, 
our work makes it clear that straight paths should be preferred. We expect that the idea can be extended to provide guidance on designing optimal bridge processes for learning SDEs. 

\emph{$\bullet$~Schrodinger bridges.} 
Another body of works  \citep{wang2021deep, de2021diffusion, chen2021likelihood, vargas2021solving}  
leverages  
Schrodinger bridges (SB) as an alternative approach to learning diffusion generative models. These approaches are attractive theoretically, but casts significant  
computational challenges for solving the Schrodinger bridge problem.

\paragraph{Re-thinking the role of diffusion noise}
The introduction of diffusion noise 
was consider essential due to the key role  
it plays in the derivations of the successful methods \citep{song2020score, ho2020denoising}. 
However, 
as rectified flow 
can achieve better or comparable results with a ODE-only framework, 
the role of diffusion mechanisms 
should be re-examed and clearly decoupled from the other merits of denoising diffusion models. 
The success of  
the denoising diffusion models may be 
 mainly attributed to the simple and stable  optimization-based training procedure 
that allows us to avoid the instability issues 
and the need of case-by-case tuning of GANs, rather than the presence of diffusion noises.

Because our work shows that there is no need to invoke SDE tools if the goal is to learn ODEs,
the remaining question is 
whether we should learn an ODE or an SDE 
for a given problem. 
As already argued by a number of works 
\citep{song2020score,song2020denoising, elucidating},  
ODEs should be preferred over SDEs in general. 
Below is a detailed comparison between ODEs and SDEs.

 $\bullet$~\emph{Conceptual simplicity and numerical speed.} %
 SDEs are more mathematically involved and  are more difficult to understand.  Numerical simulation of ODEs are simpler and faster than SDEs.

$\bullet$~\emph{Time reversibility.} 
It is equally easy to 
solve the ODEs 
forwardly and backwardly.  
In comparison, the time reversal of 
SDEs  \citep[e.g.,][]{anderson1982reverse, haussmann1986time, follmer1985entropy} 
is more involved theoretically and may not be computationally tractable. %

 $\bullet$~\emph{Latent spaces.} 
 The couplings $(Z_0,Z_1)$ 
 of ODEs are deterministic and yield low transport cost in the case of rectified flows, 
 hence providing  a good latent space for representing and manipulating outputs. 
 Introducing diffusion noises 
 make $(Z_0,Z_1)$ more stochastic and hence less useful. In fact, the $(Z_0,Z_1)$ given by DDPM ~\cite{ho2020denoising} and the SDEs of \cite{song2020score} and hence useless for latent presentation.

 $\bullet$~\emph{Training difficulty.}  
 There is no reason to believe 
that training an ODE is harder, if not easier, 
than training an SDE sharing the same marginal laws: the training loss of both cases would  share the distributions of covariant and differ only on the targets. In the setting of \citep{song2020score}, the two loss functions \eqref{equ:oud} and \eqref{equ:odeobjgg} are equivalent upto a linear reparameterization.

$\bullet$~\emph{Expressive power.}   
As every SDE can be converted into an ODE that has the same marginal distribution using the techniques in \cite{song2020denoising, song2020score} (see also \cite{villani2009optimal}), ODEs are as powerful as SDEs for representing marginal distributions,
which is what needed for the transport mapping problems considered in this work. 
On the other hand, SDEs may be preferred if we need to capture richer time-correlation structures. 

 $\bullet$~\emph{Manifold data.}
 When equipped with neural network drifts, 
 the outputs of ODEs tend to fall into a smooth low dimensional manifold,  a key inductive for 
  structured data in AI such as images and text. %
 In comparison, 
 when using SDEs to model manifold data,
 one has to carefully anneal the diffusion noise 
 to obtain smooth outcomes, which  causes slow computation and a burden of hyperparameter tuning. 
 SDEs might be more useful in 
 for modeling highly noisy data in areas like finance and economics, and in 
areas  that involve diffusion processes physically, such as molecule simulation.

\paragraph{Optimal vs. straight transport}
Optimal transport has been extensively explored in machine learning 
as a powerful way to 
 compare and transfer between probability measures. 
 For the transport mapping problem considered in this work,
 a natural approach is to finding the optimal coupling $(Z_0,Z_1)$ that minimizes a transport cost $\E[c(Z_1-Z_0)]$ for a given $c$. 
 The most common choice of $c$ is the quadratic cost $c(\cdot) = \norm{\cdot}^2$. 
 
 However, 
 finding the optimal couplings, 
 especially for high dimensional continuous measures, is highly challenging computationally and is the subject of active research; 
 see for example \citep{seguy2017large, korotin2021neural, korotin2022neural, makkuva2020optimal, rout2021generative, daniels2021score}.    
 In addition, 
although the optimal couplings 
are known to have nice smoothness and other regularity properties, 
it is not necessary to accurately find the optimal coupling because the transport cost
do not exactly align with
the learning  performance of individual problems; see e.g., \cite{korotin2021neural}.  %

 In comparison, 
 our reflow procedure 
 finds a straight coupling, 
 which is not optimal w.r.t. a given $c$ (see Section~\ref{sec:stcouplings}).  
 From the perspective of fast inference, 
 all straight couplings are equally good because they all yield straight rectified flows and hence can be simulated with one Euler step.

\section{Experiments}
\label{sec:empirical} 
We start by studying the impact of reflow on toy examples. 
After that, 
we demonstrate that with multiple times of reflow, rectified flow achieves state-of-the-art performance on CIFAR-10. 
Moreover, it can also generate high-quality images on high-resolution image datasets. 
Going beyond unconditioned image generation, we apply our method to unpaired image-to-image translation tasks to generate visually high-quality image pairs.

\paragraph{Algorithm} 
We follow the procedure in Algorithm~\ref{alg:cap}. 
We start with drawing $(X_0,X_1)\sim \tg_0 \times \tg_1$ and use it to get the first rectified flow $\vv Z^1$ by minimizing \eqref{equ:mainf}.
The second rectified flow $\vv Z^2$ is obtained by the same procedure except with the data replaced by the draws from $(Z_0^1,Z_1^1)$, 
obtained by simulating the first rectified flow $\vv Z^1$. This process is repeated for $k$ times to get the \emph{$k$-rectified flow} $\vv Z^k$.
Finally, we can further distill the  \emph{$k$-rectified flow} $\vv Z^k$  into a one step model $z_1= z_0 +  v(z_0, 0)$ by fitting it on draws from $(Z_0^k, Z_1^k)$. 

By default, the ODEs 
are simulated using the vanilla Euler method with constant step size $1/N$ for $N$ steps, that is,  $\hat Z_{t+1/N} = \hat Z_{t} + v(\hat Z_{t}, t)/N$ for $t \in\{0,\ldots, N\}/N$. 
We use the Runge-Kutta method of order 5(4) from Scipy 
\citep{2020SciPy-NMeth}, denoted as RK45, 
which adaptively decide the step size and number of steps $N$ based on  user-specified relative and absolute tolerances.  
In our experiments, we stick to the same parameters as~\cite{song2020score}.

\subsection{Toy Examples}
\label{sec:toy}
To accurately illustrate the theoretical properties,
we use the non-parametric estimator $v^{X,h}(z,t)$ in \eqref{equ:npfunc} in the toy examples in Figure~\ref{fig:twodotstoy}, \ref{fig:toystar}, \ref{fig:gauss2dots}, \ref{fig:speedtoy}.  
In practice, we approximate the expectation in \eqref{equ:npfunc} an nearest neighbor estimator: given a sample 
$\{x_0\datai, x_1\datai\}_i$ 
drawn from $(X_0,X_1)$, 
we estimate $v^X$ by 
\bb 
v^{X,h}(z,t) \approx \!\!\! \sum_{i\in \mathrm{knn}(z, m)}\!\!\!  \frac{x_1\datai - z}{1-t} \omega_h(x_t\datai, z)~~/\!\!\!
\sum_{i\in \mathrm{knn}(z, m)} \!\!\!\!\! \omega_h(x_t\datai, z), 
&&
x_t\datai =t x_1\datai + (1-t) x_0\datai, 
\ee 
where $\mathrm{knn}(z, m)$ 
denotes the top $m$ nearest neighbors of $z$ in $\{x_t\datai\}_i$.  
We find that the results are not sensitive to the choice of $m$ and the bandwidth $h$ (see Figure~\ref{fig:smoothness}). 
We use $h = 1$ and $m = 100$ by default.  
The flows are simulated using Euler method
with a constant step size of $1/N$ for $N$ steps.
We use $N=100$ steps unless otherwise specified. %

Alternatively, $v^X$ can be parameterized as a neural network and trained with stochastic gradient descent or Adam. %
Figure \ref{fig:smoothness} shows an example of when $v^X$ is parameterized as an 2-hidden-layer fully connected neural network with 64 neurons in both hidden layers. %
We see that the neural networks fit less perfectly with the linear interpolation trajectories 
(which should be piece-wise linear in this toy example).  
As shown in Figure~\ref{fig:smoothness}, 
we find that enhancing the smoothness of the  neural networks (by increasing the L2 regularization coefficient during training) can help  straighten the flow, in addition to the rectification effect. %

\newcommand{\tmpfsize}{\scriptsize}
\begin{figure}[tbh]
\centering
\scalebox{1}{
\begin{tabular}{cccccc}
\tmpfsize  1-Rectified Flow & \tmpfsize  2-Rectified Flow  & \tmpfsize  3-Rectified Flow  & 
\tmpfsize  1-Rectified Flow  & \tmpfsize  2-Rectified Flow   &\tmpfsize   3-Rectified Flow   \\
\raisebox{1em}{\rotatebox{90}{\tmpfsize{L2 Penalty}=0}}
\includegraphics[width=.14\textwidth]{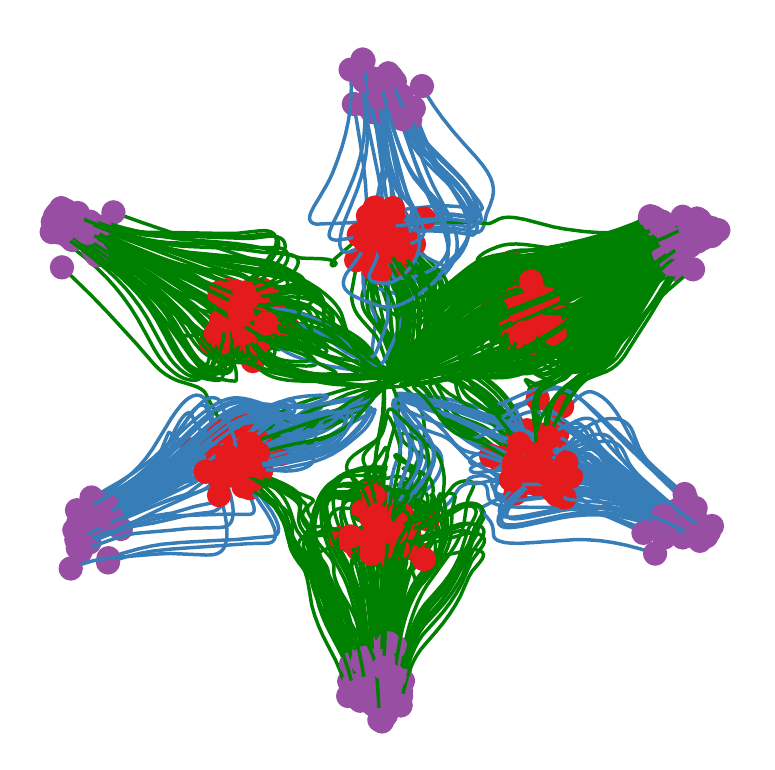}
& \includegraphics[width=.14\textwidth]{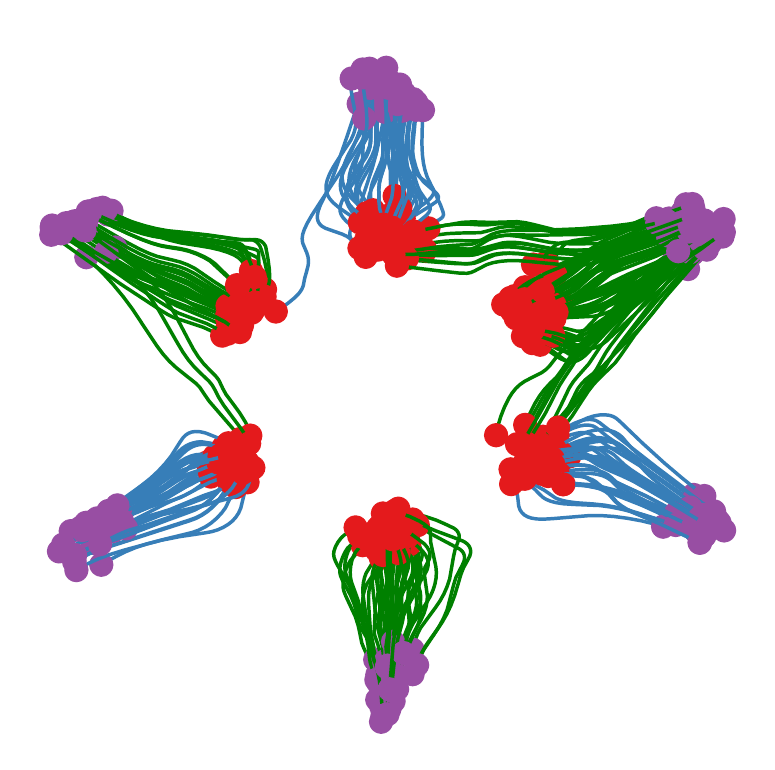} & \includegraphics[width=.14\textwidth]{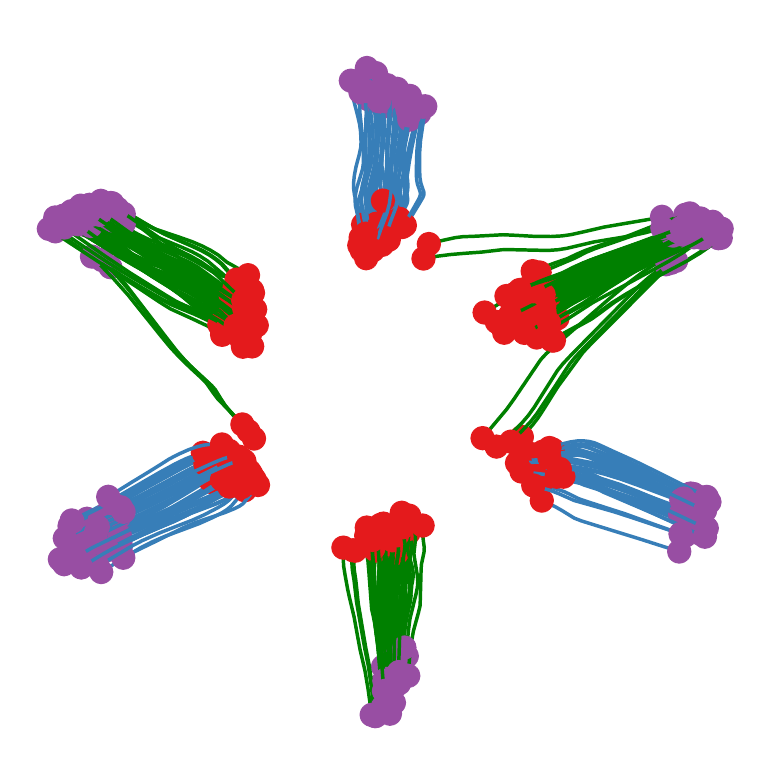} & 
\raisebox{1.8em}{\rotatebox{90}{\tmpfsize{$h=0.01$}}}
\includegraphics[width=.14\textwidth]{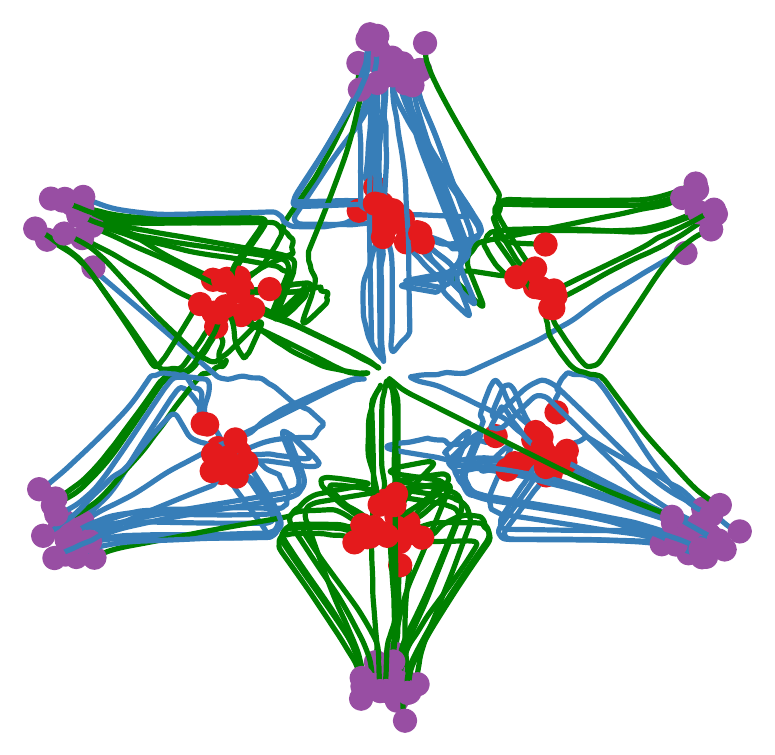}& \includegraphics[width=.14\textwidth]{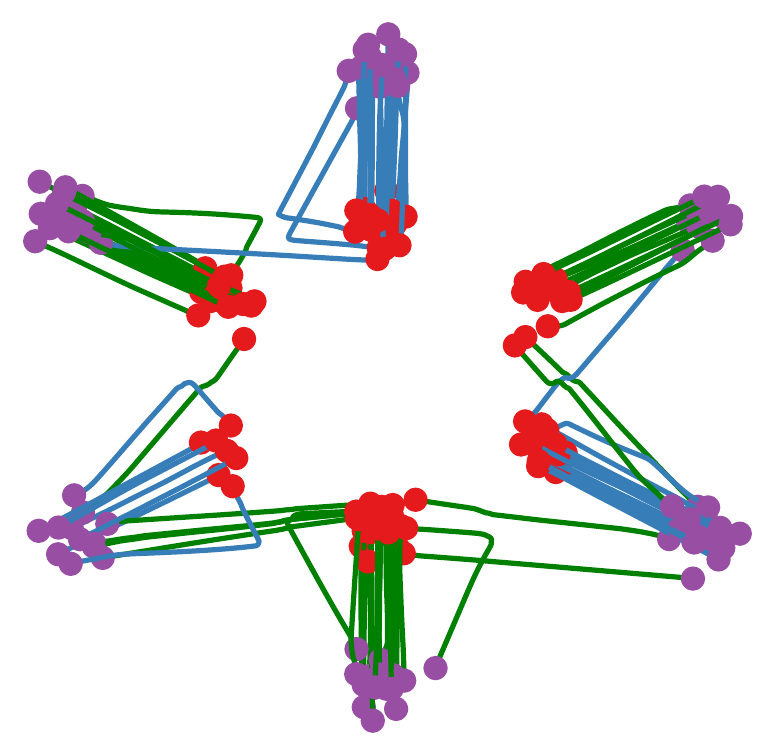} & \includegraphics[width=.14\textwidth]{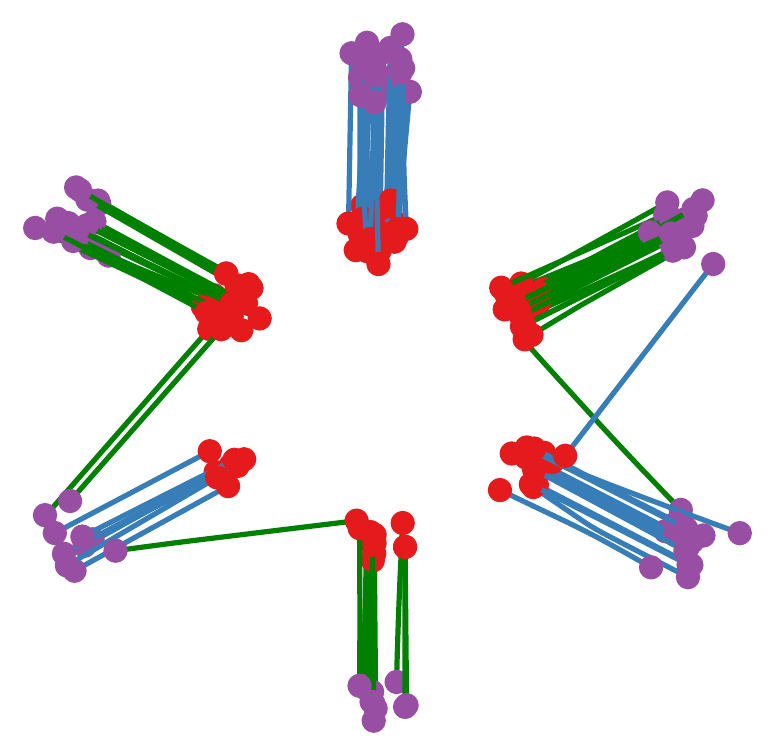} \\
\raisebox{.5em}{\rotatebox{90}{\tmpfsize{L2 Penalty$=0.01$}}}
\includegraphics[width=.14\textwidth]{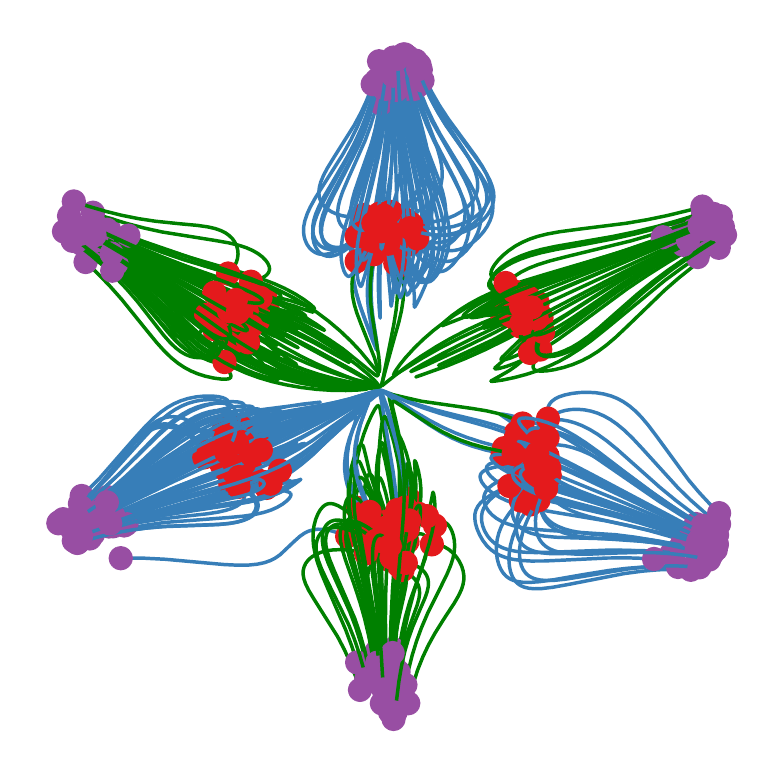}
& \includegraphics[width=.14\textwidth]{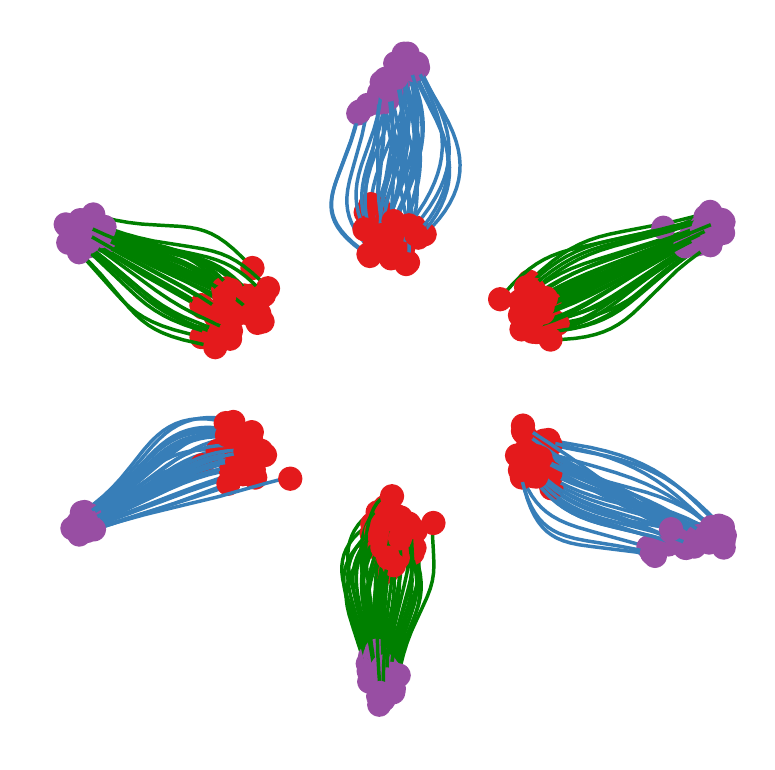} & \includegraphics[width=.14\textwidth]{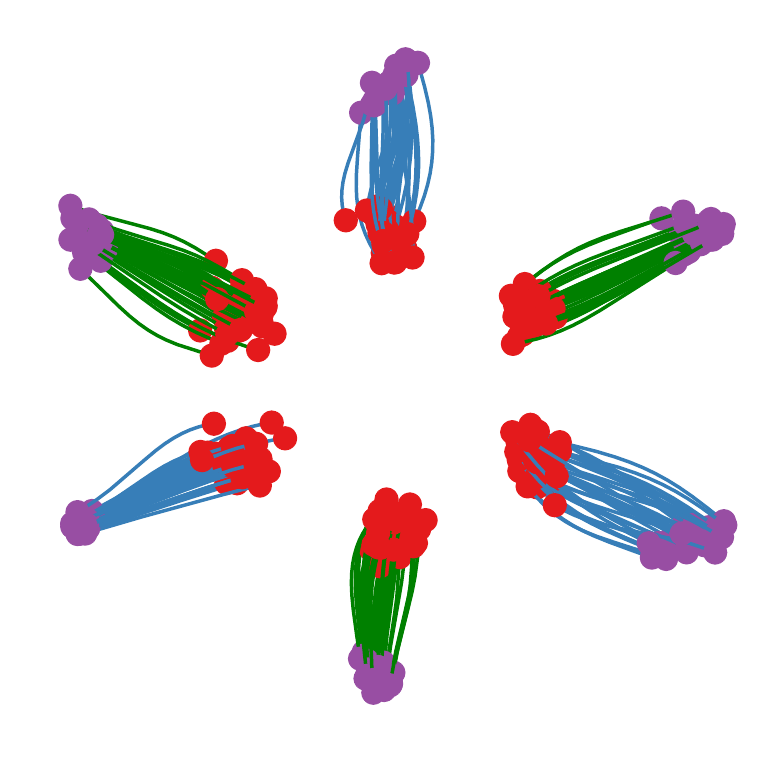} &
\raisebox{1.2em}{\rotatebox{90}{\tmpfsize{$h=1$}}}
\includegraphics[width=.14\textwidth]{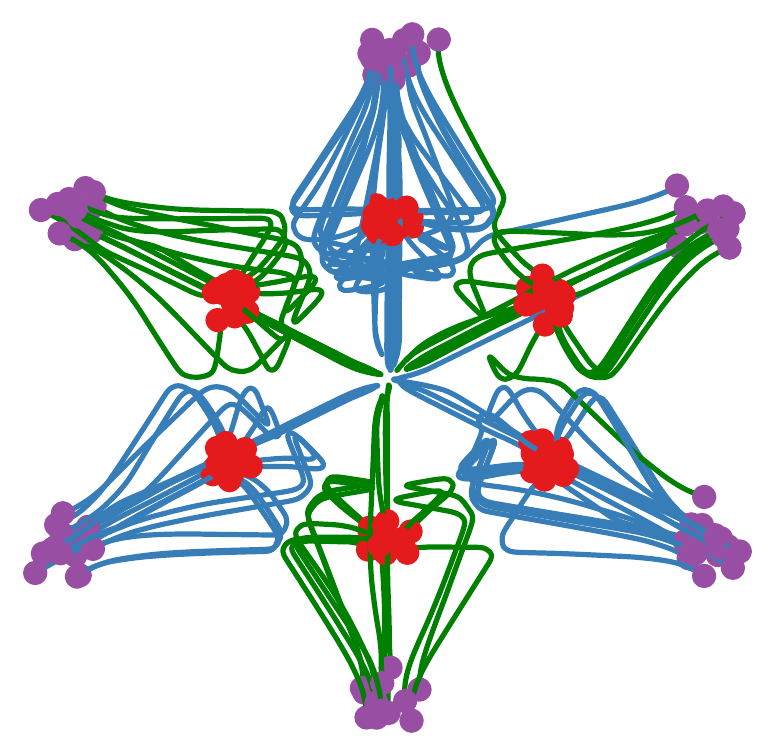}
& \includegraphics[width=.14\textwidth]{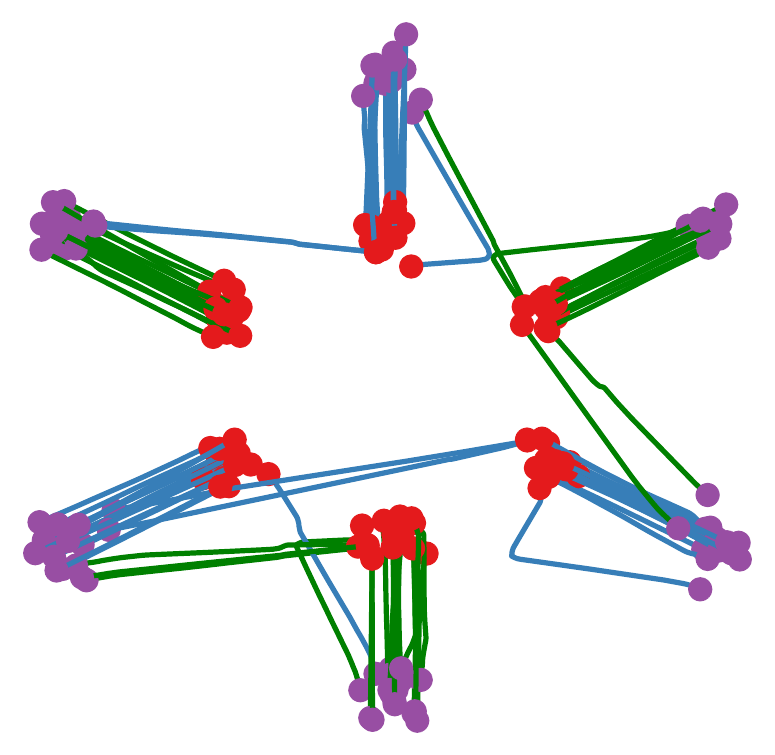} & \includegraphics[width=.14\textwidth]{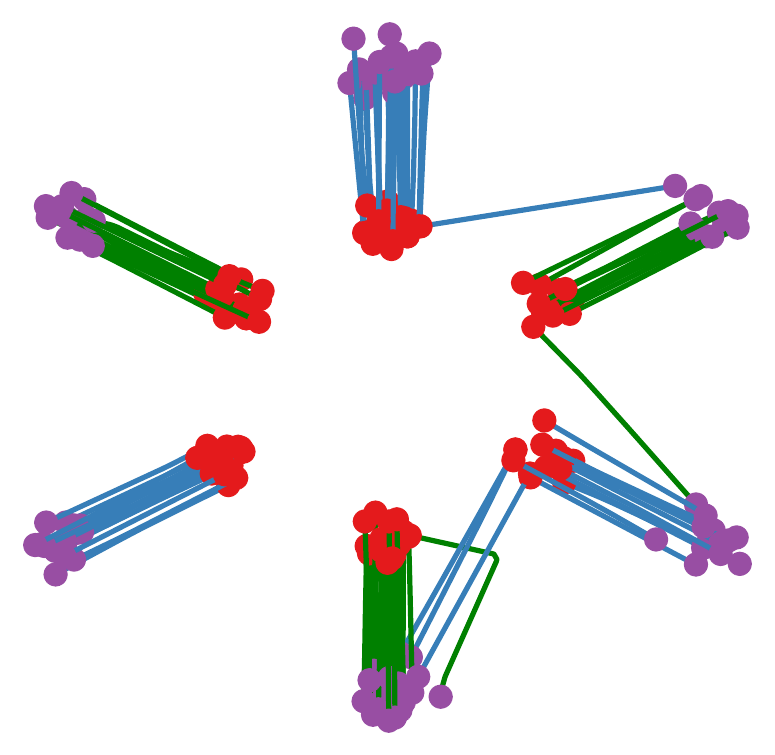} \\

\vspace{-15pt}
\end{tabular}}
\caption{
Rectified flows 
fitted with 
neural networks trained with different L2 penalty (left), and kernel estimator with different bandwidth $h$ (right).  
$\tg_0$: red dots; $\tg_1$: purple dots. 
}
\label{fig:smoothness}
\vspace{-5pt}
\end{figure}

In Figure~\ref{fig:toystar} of Section~\ref{sec:secondintro},  
the straightness is calculated as the empirical estimation of \eqref{equ:straight} based on the simulated trajectories.
The relative transport cost is calculated based on  $\{z_0\datai,z_1\datai\}_{i=1}^n$ drawn from $(Z_0,Z_1)$ by simulating the flow, as $\frac{1}{n}\sum_{i=1}^n \norm{z_1\datai-z_0\datai}^2 - \norm{z_1\dataidx{i^*}-z_0\datai}^2$, where $z_1\dataidx{i^*}$ is the optimal L2 assignment of $z_0\datai$ obtained by solving the discrete L2 optimal transport problem between $\{z_0\datai\}$ and $\{z_1\datai\}$. 
We should 
note that this metric is only useful in low dimensions, as it tends to be identically zero in high dimensional cases 
even $v^X$ is set to be a random neural network.
This misleading phenomenon is what causes  \cite{khrulkov2022understanding} to make the false hypothesis that DDIM yields L2 optimal transport.

\subsection{Unconditioned Image Generation}
\label{sec:exp:cifar}

We test rectified flow for unconditioned image generation on CIAFR-10 and a number of high resolution datasets. 
The methods are evaluated by the quality of generated images by Fréchet inception distance (FID) and inception score (IS), 
 and the diversity of the generated images by the recall score  following \citep{kynkaanniemi2019improved}.

\paragraph{Experiment settings}
For the purpose of generative modeling, we set $\tg_0$ to be the standard Gaussian distribution and $\tg_1$ the data distribution. 
Our implementation of rectified flow is modified upon the open-source code of ~\citep{song2020score}. %
We adopt the U-Net architecture of DDPM++~\cite{song2020score} for representing the drift $v^X$, 
and report in Table~\ref{tab:cifar10} (a)  and Figure~\ref{fig:cifar} the results of our method and the (sub)-VP ODE from \cite{song2020score} using the same architecture.  
Other recent results using different network architectures are shown in Table~\ref{tab:cifar10} (b) for reference. 
More detailed settings can be found in the Appendix.

\begin{table}[h]
    \centering
    \begin{tabular}{cc} 
    \resizebox{.5\textwidth}{!}{
    \begin{tabular}{lcccc}
        \hline \hline
        Method  & NFE($\downarrow$) & IS ($\uparrow$) & FID ($\downarrow$) & Recall ($\uparrow$) \\  \hline \hline 
        \emph{ODE} &
        \multicolumn{4}{l}{\emph{One-Step Generation (Euler solver, N=1)}} \\
        \hline 
         \textbf{1-Rectified Flow (\emph{+Distill}) } & 1 & 1.13 (\emph{\textbf{9.08}}) & 378 (\emph{6.18}) & 0.0 (\emph{0.45}) \\
        \textbf{2-Rectified Flow (\emph{+Distill}) } & 1 & 8.08 (\emph{9.01}) & 12.21 (\emph{\textbf{4.85}}) & 0.34 (\emph{0.50}) \\
         \textbf{3-Rectified Flow (\emph{+Distill}) } & 1 & 8.47 (\emph{8.79}) & 8.15 (\emph{5.21}) & 0.41 (\emph{\textbf{0.51}}) \\
        VP ODE~\citep{song2020score} (\emph{+Distill}) & 1 & 1.20 (\emph{8.73}) & 451 (\emph{16.23}) & 0.0 (\emph{0.29}) \\
         sub-VP ODE~\citep{song2020score} (\emph{+Distill}) & 1 & 1.21 (\emph{8.80}) & 451 (\emph{14.32}) & 0.0 (\emph{0.35}) \\
         \hline \hline  
         \emph{ODE} & 
          \multicolumn{4}{l}{\emph{Full Simulation (Runge–Kutta (RK45), Adaptive $N$)}} \\ 
        \hline           
         \textbf{1-Rectified Flow} & 127 & \textbf{9.60} & \textbf{2.58} & \textbf{0.57} \\
         \textbf{2-Rectified Flow} & 110 & 9.24 & 3.36 & 0.54 \\
         \textbf{3-Rectified Flow} & 104 & 9.01 & 3.96 & 0.53 \\
        VP ODE~\citep{song2020score} & 140 & 9.37 & 3.93 & 0.51 \\
         sub-VP ODE~\citep{song2020score} & 146 & 9.46 & 3.16 & 0.55 \\         
        \hline \hline 
        \emph{SDE} & 
        \multicolumn{4}{l}{\emph{ Full Simulation (Euler solver, N=2000)}} \\ 
        \hline            
         VP SDE~\citep{song2020score} & 2000 & 9.58 & 2.55 & 0.58  \\
         sub-VP SDE~\citep{song2020score} & 2000 & 9.56 & 2.61 & 0.58  \\
        \hline \hline
    \end{tabular}
    }
    &     
    \resizebox{.46\textwidth}{!}{
    \begin{tabular}{lcccc}
        \hline \hline
        Method & NFE($\downarrow$) & IS ($\uparrow$) & FID ($\downarrow$) & Recall ($\uparrow$) \\ \hline  \hline 
        \emph{GAN} & 
        \multicolumn{4}{l}{\emph{One-Step Generation}} \\ 
        \hline            
        SNGAN~\citep{miyato2018spectral} & 1 & 8.22 & 21.7 & 0.44  \\
        StyleGAN2~\citep{karras2020training} & 1 & 9.18 & 8.32 & 0.41 \\
         StyleGAN-XL~\citep{sauer2022stylegan} & 1 & - & 1.85 & 0.47 \\
         StyleGAN2 + ADA~\citep{karras2020training} & 1 & 9.40 & 2.92 & 0.49 \\
         StyleGAN2 + DiffAug~\citep{zhao2020differentiable} & 1 & 9.40 & 5.79 & 0.42 \\
         TransGAN + DiffAug~\citep{jiang2021transgan} & 1 & 9.02 & 9.26 & 0.41 \\
         \hline \hline 
         \emph{GAN with U-Net} & \multicolumn{4}{l}{\emph{One-step Generation}} \\  
         \hline
         TDPM (T=1)~\citep{zheng2022truncated} & 1 & 8.65 & 8.91 & 0.46\\
         Denoising Diffusion GAN (T=1)~\citep{xiao2021tackling} & 1 & 8.93 & 14.6 & 0.19 \\
         \hline \hline
        \emph{ODE} & \multicolumn{4}{l}{\emph{One Step Generation (Euler solver, N=1)}} \\ 
        \hline 
         DDIM Distillation~\citep{luhman2021knowledge} & 1 & 8.36 & 9.36 & 0.51 \\
         NCSN++ (VE ODE) \citep{song2020score} (\emph{+Distill}) & 1 & 1.18 (\emph{2.57}) & 461 (\emph{254}) & 0.0 (\emph{0.0}) \\\hline \hline 
        \emph{ODE} & 
        \multicolumn{4}{l}{\emph{Full Simulation (Runge–Kutta (RK45), Adaptive $N$)}}
        \\\hline  %
        NCSN++ (VE ODE)~\citep{song2020score} & 176 & 9.35 & 5.38 & 0.56   \\
        \hline \hline
        \emph{SDE} & 
        \multicolumn{4}{l}{\emph{Full Simulation (Euler solver)}} 
        \\\hline  %
        DDPM~\citep{ho2020denoising} & 1000 & 9.46 & 3.21 & 0.57 \\
        NCSN++ (VE SDE)~\citep{song2020score} & 2000 & 9.83 & 2.38 & 0.59  \\
\hline         \hline 
         
    \end{tabular}
    } \\
    \scriptsize{ (a) Results using the DDPM++ architecture.}
    & 
    \scriptsize{ (b) Recent results with different architectures reported in literature.}
    \end{tabular} 
    \caption{Results on CIFAR10 unconditioned image generation. 
    Fréchet Inception Distance (FID) and Inception Score (IS) measure the  quality of the generated images, 
    and recall score~\citep{kynkaanniemi2019improved} measures diversity. 
    The number of function evaluation (NFE)  denotes the number of times we need to call the main neural network during inference. 
    It coincides with the number of discretization steps $N$ for ODE and SDE models.
    }
    \label{tab:cifar10}
\end{table}

\begin{figure}[tbh]
\centering
\scalebox{1}{
\begin{tabular}{cc}
\includegraphics[width=.47\textwidth]{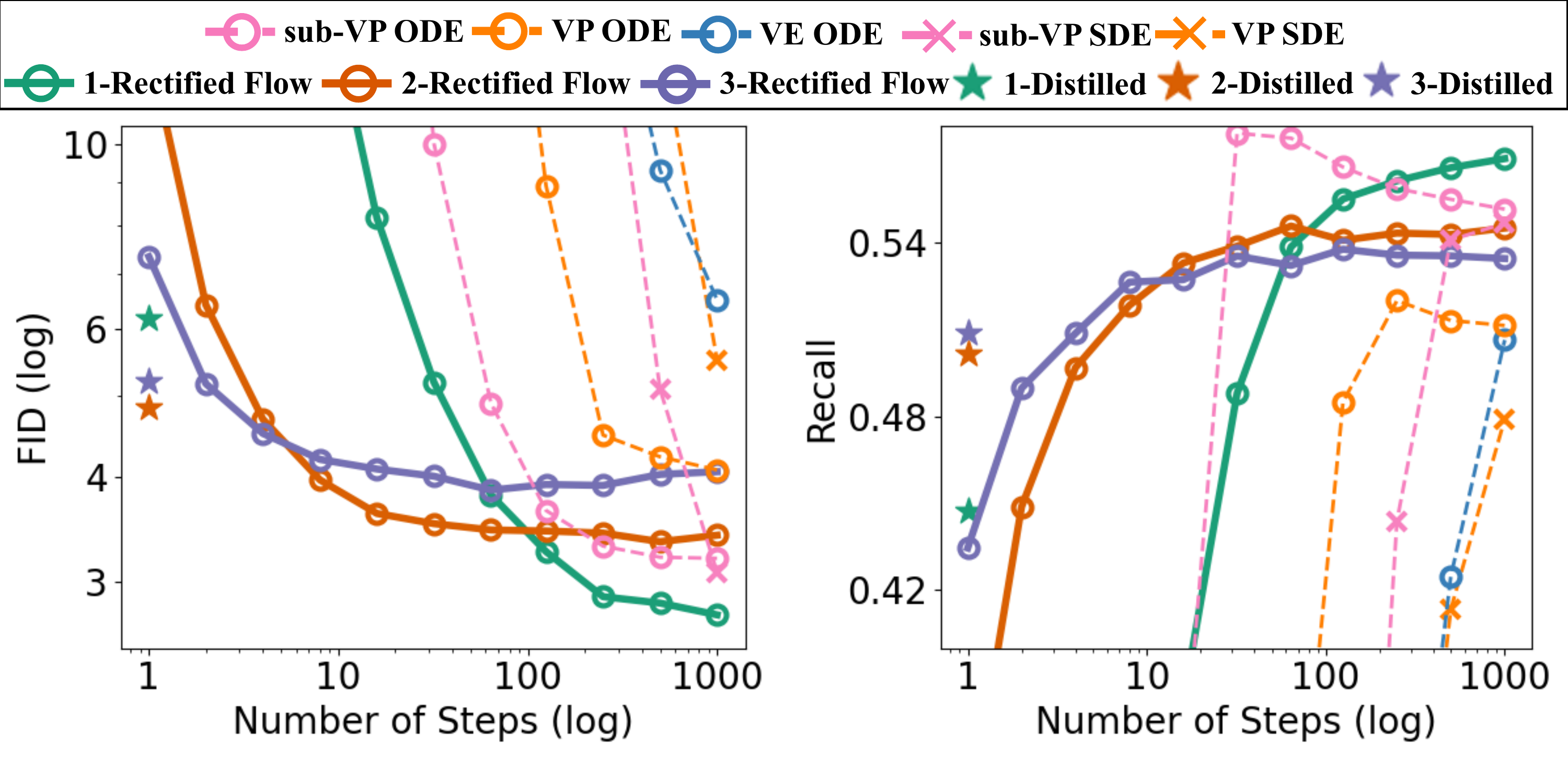} & \includegraphics[width=.47\textwidth]{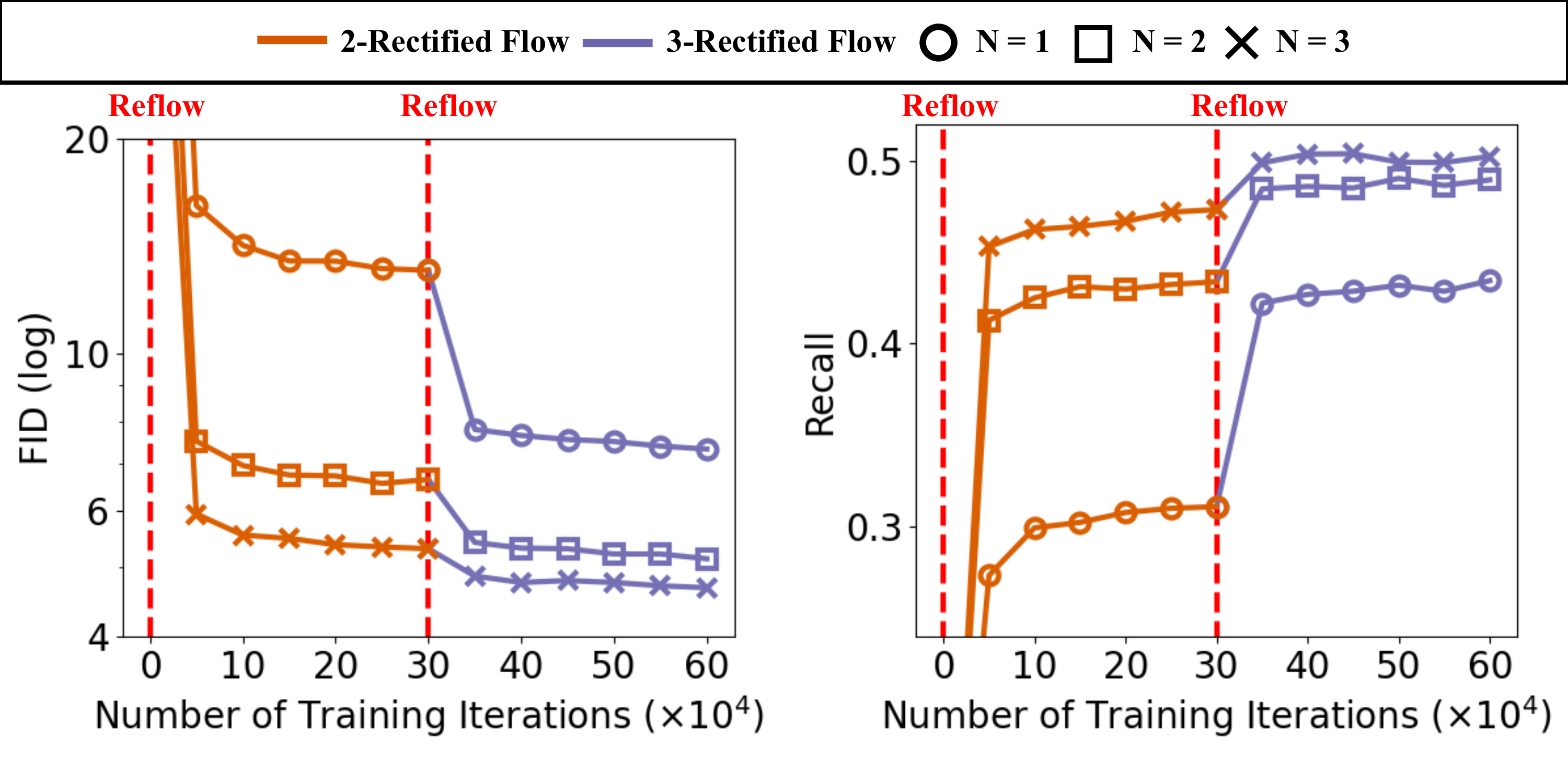}
\\
\footnotesize{(a) FID and Recall vs. Number of Euler discretization steps $N$} & \footnotesize{(b) FID and Recall vs. Training Iterations}
 \\
\vspace{-15pt}
\end{tabular}}
\caption{
(a) Results of rectified flows and (sub-)VP ODE on CIFAR10  
with different number $N$ of Euler discretization steps. 
(b) The FID and recall during different reflow and training steps. 
In (a), \emph{$k$-Distilled} refers to the one-step model distilled from \emph{$k$-Rectified Flow} for $k=1,2,3$. 
}
\label{fig:cifar}
\vspace{-5pt}
\end{figure}

\paragraph{Results} 
\emph{$\bullet$~Results of fully solved ODEs.} 
As shown in Table~\ref{tab:cifar10} (a), 
the 1-rectified flow 
trained on the DDPM++ architecture, solved with RK45, 
yields the lowest FID ($2.58$) and highest recall ($0.57$) among all the ODE-based methods.  
In particular, the recall of 0.57 yields a substantial improvement over existing ODE and GAN methods. 
Using the same RK45 ODE solver, rectified flows require fewer steps to generate the images compared with VE, VP, sub-VP ODEs. 
The results are comparable to the fully simulated (sub-)VP SDE, which yields simulation cost. 

\begin{wrapfigure}{r}{0.45\textwidth}
    \vspace{-1\baselineskip}
    \includegraphics[width=0.48\textwidth]{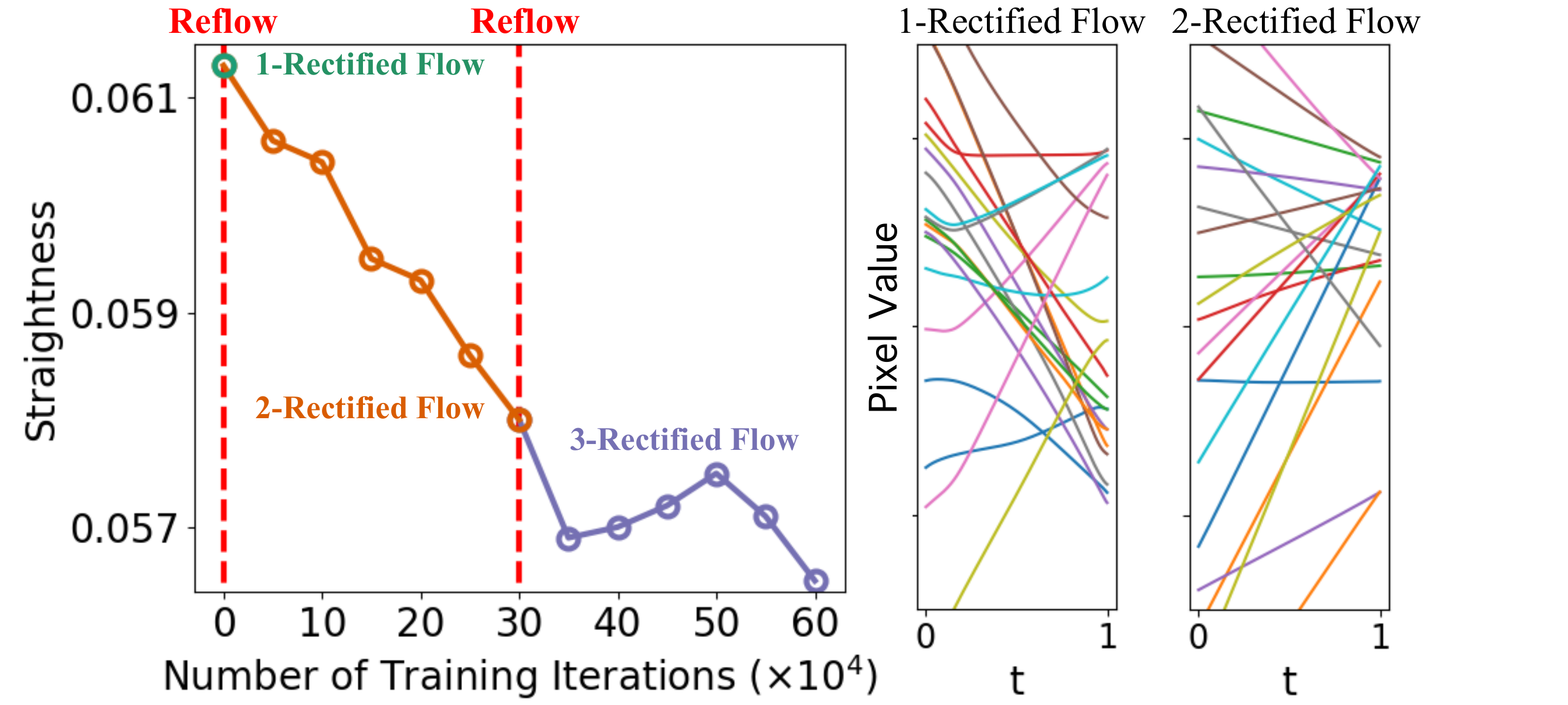}
    \caption{
    The straightening effect on CIFAR10. 
    Left: %
    the straightness measure on different reflow steps and training iterations. 
    Right: trajectories of randomly sampled pixels following 1- and 2-rectified flow. 
    }
    \vspace{-1\baselineskip}
    \label{fig:cifar_straight}
\end{wrapfigure}
\emph{$\bullet$~Results on few and single step generation.} 
As shown in Figure~\ref{fig:cifar}, 
the reflow procedure substantially improves both FID and recall in the small step regime (e.g., $N{\scriptsize\lessapprox}80$), even though it worsens the results 
in the large step regime due to the accumulation of error on estimating $v^x$. 
Figure~\ref{fig:cifar} (b) show that each reflow leads to a noticeable improvement in FID and recall.   
For one-step generation $(N=1)$, 
the results are further boosted by distillation (see the stars in Figure~\ref{fig:cifar} (a)).  
Overall, the distilled $k$-Rectified Flow  with $k=1,2,3$  
yield one-step generative models 
beating all previous ODEs with distillation; 
they also beat the reported results of one-step models with similar U-net type architectures 
trained using GANs (see the \emph{GAN with U-Net} in Table~\ref{tab:cifar10} (b)). 

In particular, 
the distilled 2-rectified flow achieves an FID of $4.85$, beating the best known one-step generative model with U-net architecture, $8.91$ (TDPM, Table~\ref{tab:cifar10} (b)).
The recalls of 
both 2-rectified flow ($0.50$) and 3-rectified flow ($0.51$) outperform the best known results of GANs ($0.49$ from StyleGAN2+ADA)
showing an advantage in diversity. 
We should note that the reported results of GANs have been carefully optimized 
with special techniques such as adaptive discriminator augmentation (ADA)~\citep{karras2020training}, while our results are
based on the vanilla implementation of rectified flow. 
It is likely to further improve rectified flow with proper data augmentation techniques, or the combination of GANs such as 
those proposed by TDPM~\citep{zheng2022truncated} and denoising diffusion GAN~\citep{xiao2021tackling}.

\emph{$\bullet$~Reflow straightens the flow.} 
Figure~\ref{fig:cifar_straight} shows the reflow procedure decreases improves the straightness of the flow on CIFAR10.  
In Figure~\ref{fig:cat_target} visualizes the trajectories of 1-rectified flow and 2-rectified flow on the AFHQ cat dataset: 
at each point $z_t$, we extrapolate the terminal value at $t=1$ by 
$\hat{z}_1^t = z_t + (1-t) v(z_t, t)$; 
if the trajectory of ODE follows a straight line, 
$\hat{z}_1^t$ should not change  as we vary $t$ when following the same path.  
We observe that $\hat{z}_1^t$ is almost independent with $t$ for 2-rectified flow, 
showing the path is almost straight.  
Moreover, even though 1-rectified flow is not straight with $\hat z_1^t$ over time, 
it still yields recognizable and clear images very early ($t\approx 0.1$). 
In comparison,  
it is need $t\approx 0.6$ to get a clear image
from the extrapolation of sub-VP ODE.

\paragraph{High-resolution image generation}
Figure~\ref{fig:high_res} shows the result of 1-rectified flow on image generation  
on high-resolution ($256 \times 256$) datasets, including  
 LSUN Bedroom~\citep{yu2015lsun}, LSUN Church~\citep{yu2015lsun}, CelebA HQ~\citep{karras2018progressive} to AFHQ Cat~\citep{choi2020stargan}. We can see that it can generate high quality results across the different datasets.  
 Figure~\ref{fig:cat_k} \& \ref{fig:cat_target} show that 
1-(2-)rectified flow yields good results within one or few Euler steps.  

Figure~\ref{fig:interp_editing} shows a simple example of image editing using 1-rectified flow: 
We first obtain an unnatural image $z_1$ by stitching the upper and lower parts of two natural images, 
and then run 1-rectified flow backwards to get a latent code $z_0$. 
We then modify $z_0$ to increase its likelihood under $\tg_0$ (which is $\normal(0,I)$) 
to get more naturally looking variants of the stitched image. %

\begin{figure}
    \centering
    \includegraphics[width=0.9\textwidth]{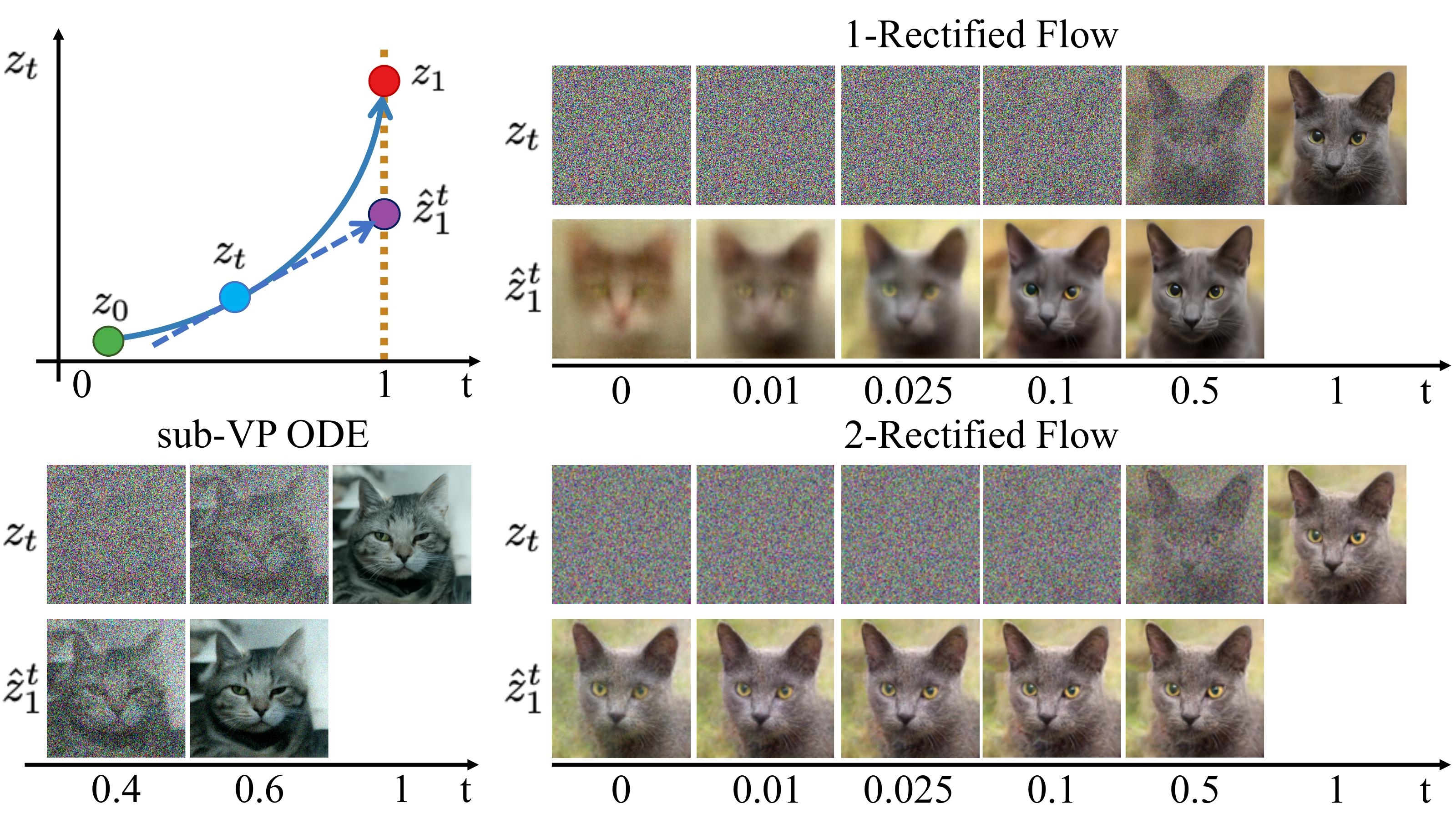}
    \caption{
    Sample trajectories $z_t$ of different flows on the AFHQ Cat dataset,  %
    and the extrapolation $\hat{z}_1^t =z_t + (1-t) v(z_t, t)$ from different $z_t$. The same random seed is adopted for all three methods. The $\hat z_1^t$ of 2-rectified flow is almost independent with $t$, indicating that its trajectory is almost straight. 
    }
    \label{fig:cat_target}
\end{figure}

\begin{figure}
    \centering
    \includegraphics[width=0.95\textwidth]{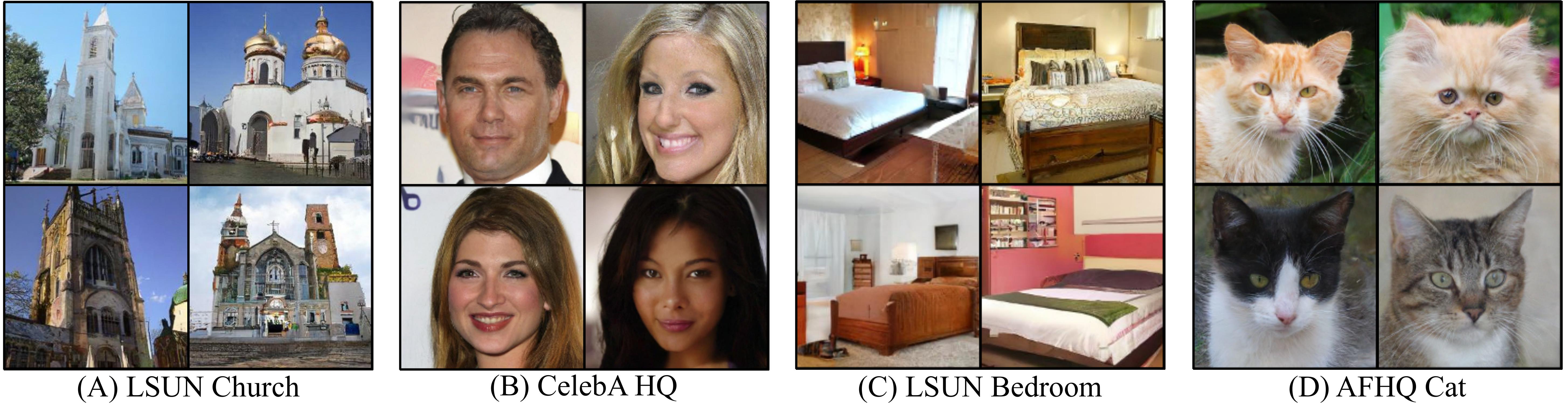}
    \caption{Examples of $256\times 256$ images generated by 1-rectified flow.}
    \label{fig:high_res}
\end{figure}

\begin{figure}
    \centering
    \includegraphics[width=0.95\textwidth]{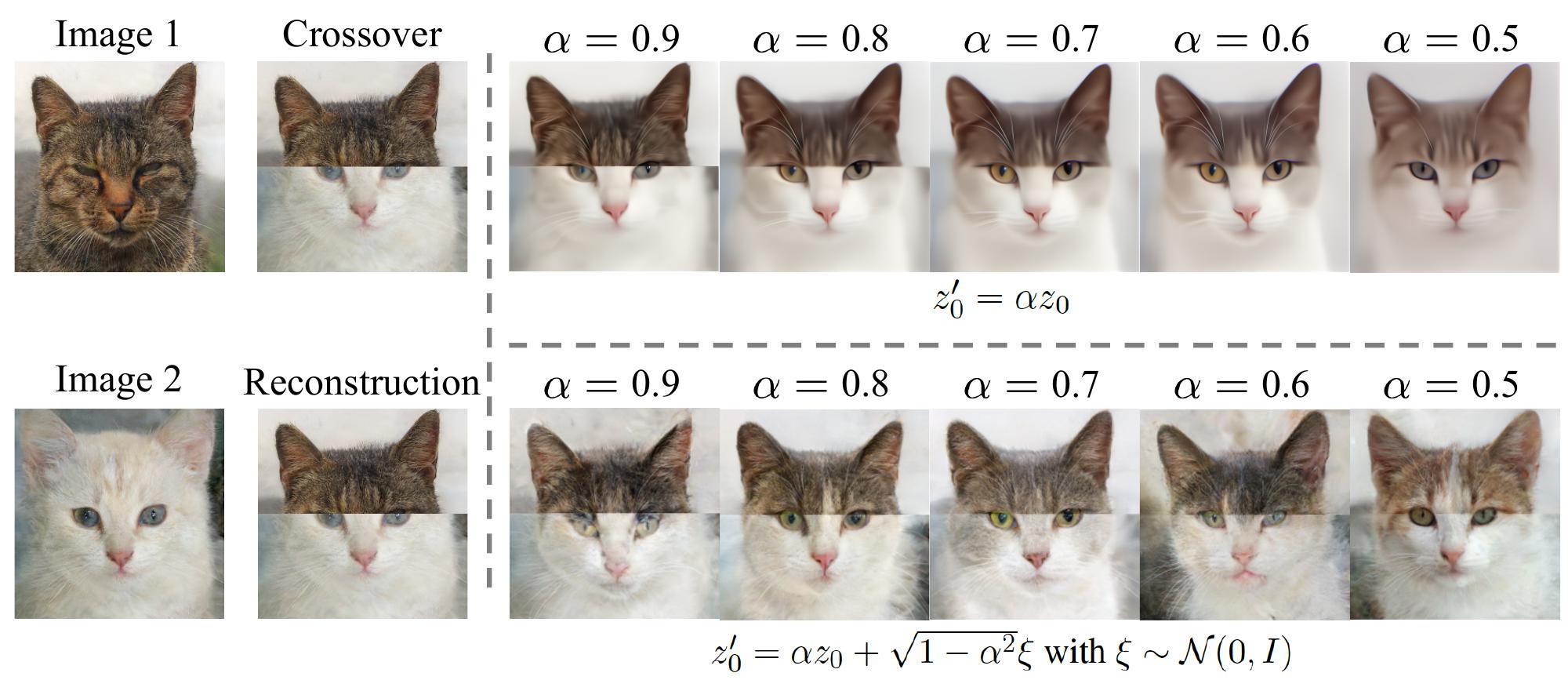}
    \caption{An example of image editing using 1-rectified flow. 
    Here, we stitch the images of a white cat and a black cat into an unnatural image (denoted as $z_1$).  
    We simulate the  ODE reversely from $z_1$ to get the latent code $z_0$. %
    Because $z_1$ is not a natural image, $z_0$ should have low likelihood under $\tg_0 = \normal(0,I)$.  
    Hence, we move $z_0$ towards the high probability region of $\tg_0$ to get $z_0'$ and solve the ODE forwardly to get a more realistically looking image $z_1'$. 
    The modification can be done deterministically by improving the $\tg_0$-likelihood via $z_0' =\alpha z_0$ with $\alpha\in(0,1)$, or 
    stochastically by  Langevin dynamics, 
    which yields a  formula of 
    $z'_0 = \alpha z_0 + \sqrt{1-\alpha^2}\xi$ with $\xi \sim \normal(0, I)$. 
    }
    \label{fig:interp_editing}
\end{figure}

\begin{figure}[h]
\resizebox{1\textwidth}{!}{
\begin{tabular}{cccc}
     \includegraphics[width=0.23\textwidth]{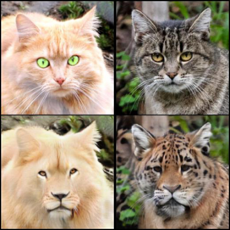} & \includegraphics[width=0.23\textwidth]{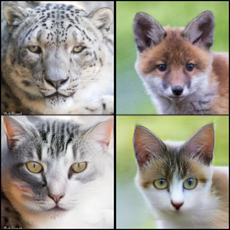} & %
     \includegraphics[width=0.23\textwidth]{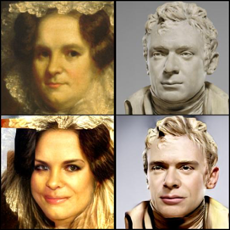} &%
     \includegraphics[width=0.23\textwidth]{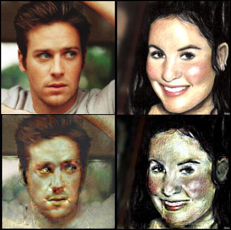}\\
     \small (A) Cat $\rightarrow$ Wild Animals & \small (B) Wild Animals $\rightarrow$ Cat & \small (C) MetFace $\rightarrow$ CelebA Face & \small (D) CelebA Face $\rightarrow$ MetFace 
\end{tabular}
}
    \caption{
Samples of 1-rectified flow simulated with $N=100$ Euler steps between different domains. 
    }
    \label{fig:met2cat}
\end{figure}

\begin{figure}[h]
\centering
\begin{tabular}{ccccc}
     Initialization & 1-Rectified Flow & 2-Rectified Flow & 1-Rectified Flow  & 2-Rectified Flow\\
     & $N=100$ & $N=100$ & $N=1$  & $N=1$\\
     \includegraphics[width=0.16\textwidth]{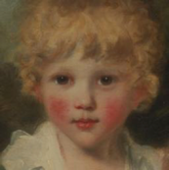} & 
     \includegraphics[width=0.16\textwidth]{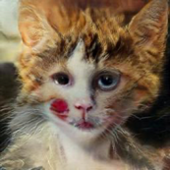} &
     \includegraphics[width=0.16\textwidth]{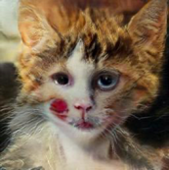} &
     \includegraphics[width=0.16\textwidth]{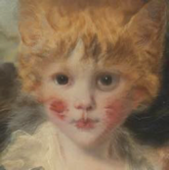} &
     \includegraphics[width=0.16\textwidth]{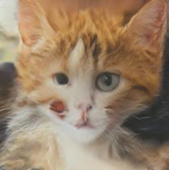}\\
     \includegraphics[width=0.16\textwidth]{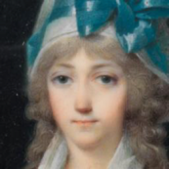} & 
     \includegraphics[width=0.16\textwidth]{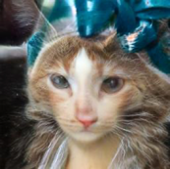} &
     \includegraphics[width=0.16\textwidth]{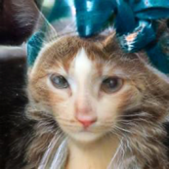} &
     \includegraphics[width=0.16\textwidth]{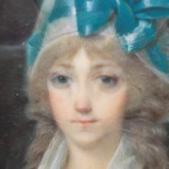} &
     \includegraphics[width=0.16\textwidth]{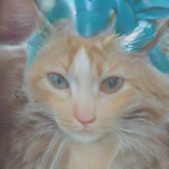}\\
\end{tabular}
    \caption{Samples of results of 1- and 2-rectified flow simulated with $N=1$ and $N=100$ Euler steps.
    }
    \label{fig:met2cat_reflow}
\end{figure}

\subsection{Image-to-Image Translation}

Assume we are given two sets of images of different styles (a.k.a. domains),  
whose distributions are denoted by $\tg_0,\tg_1$, respectively. 
We are interested in transferring the style (or other key characteristics) of the images in one domain to the other domain, in the absence of paired examples. 
A classical approach to achieving this is  
 cycle-consistent adversarial networks (a.k.a. CycleGAN) \citep{cyclegan, isola2017image}, 
 which jointly learns a forward and backward mapping $F,G$ 
 by minimizing the sum of adversarial losses on the two domains, regularized by a cycle consistency loss to enforce $F(G(x))\approx x$ for all image $x$.

By constructing the rectified flow of $\tg_0$ and $\tg_1$,  
we obtain a simple approach to image translation that requires no adversarial optimization and cycle-consistency regularization:  
training the rectified flow requires a simple optimization procedure and the cycle consistency is automatically in  flow  models satisfied due to reversibility of ODEs.  

As the main goal here is to obtain good visual results,  
we are not interested in faithfully transferring $X_0\sim \tg_0$  
to an $X_1$ that exactly follows $\tg_1$. 
Rather, we are interested in 
transferring the image styles while preserving the identity of the main object in the image. For example, 
when transferring a human face image to a cat face, 
we are interested in getting a unrealistic face of human-cat hybrid that still ``looks like" the original human face.

To achieve this, 
let $h(x)$ be a feature mapping of image $x$ representing the styles that we are interested in transferring. 
Let $X_t = t X_1 + (1-t) X_0$. Then $H_t = h(X_t)$ follows an ODE of $\d H_t = \dd h(X_t)\tt (X_1-X_0)\dt .$ Hence, to ensure that the style is transferred correctly, 
we propose to %
learn $v$ such that $H_t' = h(Z_t)$ with $\d Z_t  = v(Z_t,t)\dt $ approximates $H_t$ as much as possible. 
Because $\d H_t' = \dd h(Z_t) \tt v(Z_t, t)\dt$, we propose to minimize the following loss: 
\begin{align}\label{equ:hloss}
    \min_v \int_0^1\mathbb{E} \left[ \norm{ \nabla h(X_t)\tt (X_1 - X_0 - v (X_t, t)) }_2^2 \right ] \dt , && X_t = t X_1 + (1 - t) X_0. 
\end{align}
In practice, we set $h(x)$ to be latent representation of a classifier trained 
to distinguish the images from the two domains $\tg_0,\tg_1$, fine-tuned from a pre-trained ImageNet \citep{tan2019efficientnet} model. %
Intuitively, $\nabla_x h(x)$ serves as a saliency score and re-weights coordinates so that the loss in \eqref{equ:hloss} focuses on penalizing the error that causes significant changes on $h$.

\begin{figure}[h]
\centering
\setlength{\tabcolsep}{1pt}
\renewcommand\arraystretch{0.4}
\begin{tabular}{cc}
     & \scriptsize{1-Rectified Flow} \\
     \includegraphics[width=0.48\textwidth]{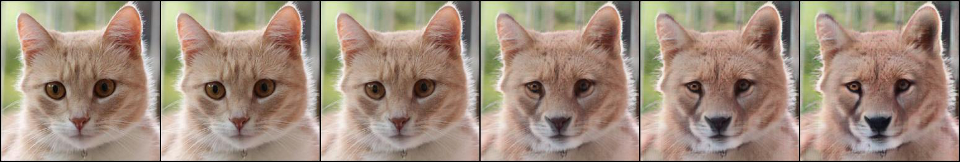} &
     \includegraphics[width=0.48\textwidth]{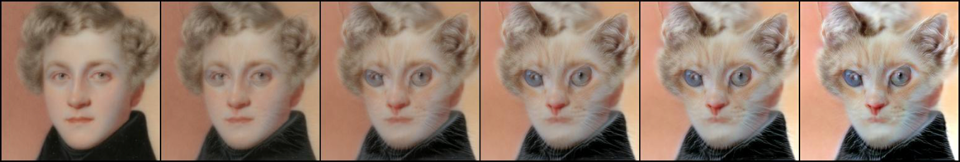}
     \\
     & \scriptsize{2-Rectified Flow} \\
     \includegraphics[width=0.48\textwidth]{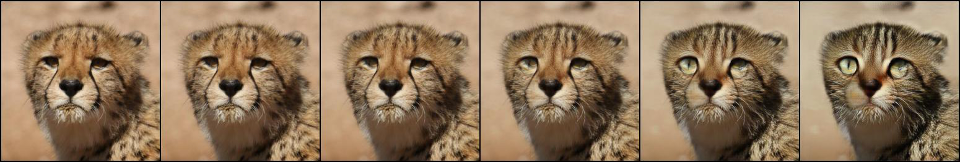} & 
     \includegraphics[width=0.48\textwidth]{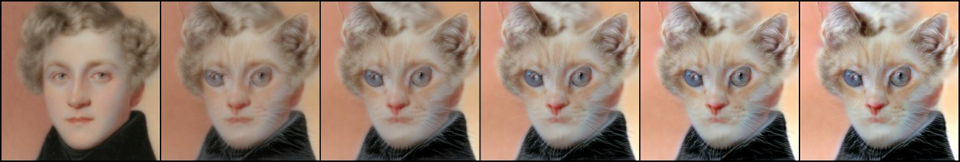}
     \\
     & \scriptsize{1-Rectified Flow} \\
     \includegraphics[width=0.48\textwidth]{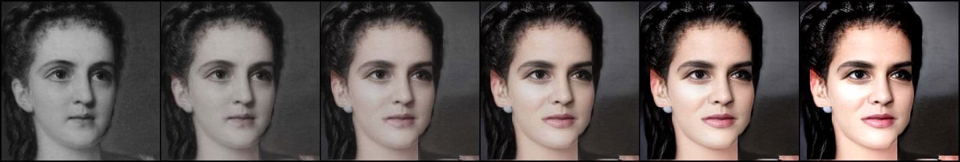} & 
     \includegraphics[width=0.48\textwidth]{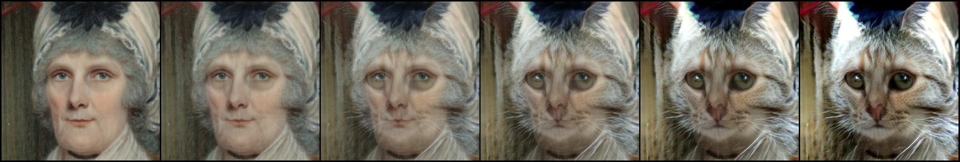}
     \\
     & \scriptsize{2-Rectified Flow} \\
     \includegraphics[width=0.48\textwidth]{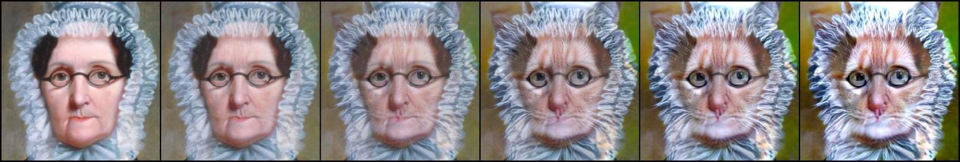} & 
     \includegraphics[width=0.48\textwidth]{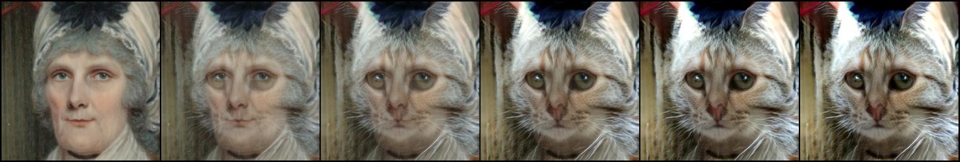}
     \\
     \hspace{3pt}\includegraphics[width=0.485\textwidth]{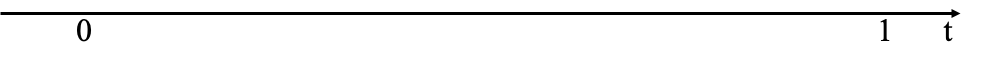} & 
     \hspace{3pt}\includegraphics[width=0.485\textwidth]{arxiv_figures/trajectory/axis.png} \\
     (a) 1-rectified flow between different domains & (b) 1- and 2-rectified flow for MetFace $\rightarrow$ Cat.  \\
     
\end{tabular}
    \caption{(a) Samples of  trajectories $z_t$ of 1- and 2-rectified flow for transferring between different domains. }
    \label{fig:met2cat_traj}
\end{figure}

\paragraph{Experiment settings}
We set the domains $\tg_0,\tg_1$ to be pairs of 
the AFHQ \citep{choi2020stargan}, MetFace \citep{karras2020training} and CelebA-HQ \citep{karras2018progressive} dataset. 
For each dataset, we randomly select $80\%$ as the training data and regard the rest as the test data; and the results are shown by initializing the trained flows from the test data.
We resize the image to $512 \times 512$.
The training and network configurations  
generally follow the experiment settings in Section~\ref{sec:exp:cifar}. See the appendix for detailed descriptions. 

\paragraph{Results}
Figure~\ref{fig:cat_k}, \ref{fig:met2cat}, \ref{fig:met2cat_reflow}, \ref{fig:met2cat_traj} 
show examples of results of 1- and 2-rectified flow simulated with Euler method with different number of steps $N$.  
We can see that rectified flows can successfully transfer the styles and generate high quality images. 
For example, when transferring cats to wild animals,  
we can generate diverse images with different animal faces, {e.g.}, fox, lion, tiger and cheetah. 
Moreover, with one step of reflow, 2-rectified flow returns good results with a single Euler step ($N=1$).  
See more examples in Appendix.

\subsection{Domain Adaptation} 
A key challenge of applying machine learning  to  real-world problems is the domain shift between the training and test datasets: the performance of machine learning models may degrade significantly 
when tested on a novel domain different from the training set.  
Rectified flow can be applied to transfer the novel domain ($\tg_0$) to the training domain ($\tg_1$) to mitigate the impact of domain shift. 

\paragraph{Experiment settings} We test the rectified flow for domain adaptation on 
a number of datasets. 
DomainNet \citep{peng2019domainnet} is a dataset of common objects in six different domain taken from DomainBed  \citep{gulrajani2020search}. %
All domains  from DomainNet  include 345 categories (classes) of objects such as Bracelet, plane, bird and cello. 
Office-Home \citep{venkateswara2017officehome} is a benchmark dataset for domain adaptation which contains 4 domains where each domain consists of 65 categories. 
To apply our method, 
first we map both the training and testing data 
to the latent representation from final hidden layer of the pre-trained model, and construct the rectified flow on the latent representation. 
We use the same DDPM++ model architecture %
for training. %
For inference, we set the number of steps of our flow model as $100$ using uniform discretization.
The methods are evaluated by the prediction accuracy of the transferred testing data 
on the classification model trained on the training data. 

\paragraph{Results}
As demonstrated in Table \ref{tab:domainbed}, 
the 1-rectified flow shows  
state-of-the-art performance 
on both DomainNet and OfficeHome. 
It is better or on  par with 
the previous best approach (Deep CORAL \citep{sun2016coral}), 
while sustainably improve over all other methods.

\begin{table}[]
    \centering
    \scalebox{0.88}{
    \begin{tabular}{l|cccccc|c}
    \hline
        Method & ERM & IRM & ARM & Mixup & MLDG & CORAL & Ours\\
        \hline
        OfficeHome & $66.5\pm0.3$ & $64.3\pm2.2$ & $64.8\pm0.3$ & $68.1\pm0.3$ & $66.8\pm0.6$  & $68.7\pm0.3$  & $\bf 69.2\pm0.5$ \\
        DomainNet & $40.9\pm0.1$ & $33.9\pm2.8$ & $35.5\pm0.2$ & $39.2\pm0.1$ & $41.2\pm0.1$ & $\bf 41.5\pm0.2$ & $\bf 41.4\pm0.1$ \\
    \hline
    \end{tabular}}
    \caption{The accuracy of the transferred testing data using different methods, on the OfficeHome and DomainNet dataset. 
    Higher accuracy means the better performance. }
    \label{tab:domainbed}
\end{table}

\newpage\clearpage

\newpage
\bibliography{z_diffusion_models} 
\bibliographystyle{plainnat}

\appendix 
\newpage \clearpage

\begin{algorithm}[h]
\caption{\PyCode{Train(Data)}}
\begin{flushleft}
\PyInput{Input: Data=\{x0, x1\}} \\ 
\PyInput{Output: Model $v(x,t)$ for the rectified flow} \\ 
\PyCode{initialize Model} \\ 
\PyCode{for x0, x1 in Data:}
\PyComment{x0, x1:~samples from  $\tg_0$,  $\tg_1$} \\  %
\PyCode{~~~~Optimizer.zero\_grad()} \\
\PyCode{~~~~t = torch.rand(batchsize)}~~\PyComment{Randomly sample  $t \in$ [0,1]} \\
\PyCode{~~~~\blue{Loss = ( Model(t*x1+(1-t)*x0,~t) - (x1-x0) ).pow(2).mean()}} \\
\PyCode{~~~~Loss.backward()} \\
\PyCode{~~~~Optimizer.step()} \\
\PyCode{return Model} 
\end{flushleft}
\label{alg:pytorchAlgorithm}
\end{algorithm}

\begin{algorithm}[h]
\caption{\PyCode{Sample(Model, Data)}}
\begin{flushleft}
\PyInput{Input: Model $v(x,t)$ of the rectified flow} \\ 
\PyInput{Output: draws of the rectified coupling $(Z_0,Z_1)$} \\ 
\PyCode{coupling = []} \\
\PyCode{for x0, \_ in Data:}
\PyComment{x0:~samples from  $\tg_0$ (batchsize$\times$dim)} \\
\PyCode{~~~~x1 = model.ODE\_solver(x0)} \\
\PyCode{~~~~coupling.append((x0, x1))} \\
\PyCode{return coupling}
\end{flushleft}
\label{alg:rectify}
\end{algorithm}

\begin{algorithm}[h]
\caption{\PyCode{Reflow(Data)}}
\begin{flushleft}
\PyInput{Input: Data=\{x0, x1\}} \\ %
\PyInput{Output: draws of the $K$-th rectified coupling} \\ 
\PyCode{Coupling = Data} \\
\PyCode{for  $k=1,\ldots,K$:} \\ 
\PyCode{~~~~Model = Train(Coupling)}   \\
\PyCode{~~~~Coupling = sample(Model, Data)}   \\
\PyCode{return Coupling}
\end{flushleft}
\label{alg:rectify}
\end{algorithm}

\begin{figure}[h]%
  \begin{center}
  \begin{tabular}{lr}
    \includegraphics[width=0.45\textwidth]{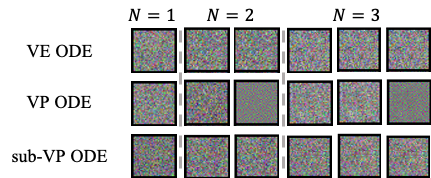}  &
    \smash{\raisebox{10pt}{ 
    \includegraphics[width=0.45\textwidth]{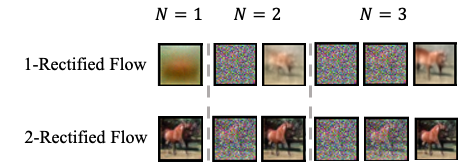} }
    }
    \end{tabular} 
  \end{center}
  \vspace{0\baselineskip}
  \caption{Few-step generation with different ODEs. Compared with VE,VP,sub-VP ODE, 1-rectified flow can generate blurry images using only 1,2,3 steps. After one time of rectification, 2-rectified flow can generate clear images with 1,2,3 steps. 
  }
  \label{fig:cifar10_triangle}
\end{figure}

\section{Additional Experiment Details}
\paragraph{Experiment Configuration on CIFAR10}  We conduct unconditional image generation with the CIFAR-10 dataset \citep{krizhevsky2009learning}. The resolution of the images are set to $32 \times 32$. For rectified flow, we adopt the same network structure as DDPM++ in \citep{song2020score}. The training of the network is smoothed by exponential moving average as in~\citep{song2020score}, with a ratio of $0.999999$. We adopt Adam~\citep{kingma2014adam} optimizer with a learning rate of $2e-4$ and a dropout rate of $0.15$.

For reflow, we first generate 4 million pairs of $(z_0, z_1)$ to get a new dataset $D$, then fine-tune the $i$-rectified flow model for $300,000$ steps to get the $(i+1)$-rectified flow model. We further distill these rectified flow models for few-step generation. To get a $k$-step image generator from the $i$-rectified flow, we randomly sample $t \in \{0, 1/k, \cdots, (k-1)/k\}$ during fine-tuning, instead of randomly sampling $t \in [0, 1]$. Specifically, for $k=1$, we replace the L2 loss function with the LPIPS similarity~\citep{zhang2018unreasonable} since it empirically brings better performance.

\paragraph{Expreiment Configuration on Image-to-Image Translation}
In this experiment, we also adopt the same U-Net structure of DDPM++~\cite{song2020score} for representing the drift $v^X$. 
We follow the procedure in Algorithm~\ref{alg:cap}. 
For the purpose of generative modeling, we set $\tg_0$ to be one domain dataset and $\tg_1$ the other domain dataset. 
For optimization, we use AdamW~\citep{loshchilov2017decoupled} optimizer with $\beta$ $(0.9, 0.999)$, weight decay $0.1$ and dropout rate $0.1$.
We train the model with a batch size of $4$ for $1,000$ epochs. 
We further apply exponential moving average (EMA) optimizer with coefficient $0.9999$.
We perform grid-search on the learning rate from $\{5\times 10^{-4}, 2\times 10^{-4}, 5\times 10^{-5}, 2\times 10^{-5}, 5\times 10^{-6}\}$ and pick the model with the lowest training loss.

We use the AFHQ \citep{choi2020stargan}, MetFace \citep{karras2020training} and CelebA-HQ \citep{karras2018progressive} dataset. 
Animal Faces HQ (AFHQ) is an animal-face dataset consisting of 15,000 high-quality images at $512 \times 512$ resolution. The dataset includes three domains of cat, dog, and wild animals, each providing 5000 images. 
MetFace consists of 1,336 high-quality PNG human-face images at $1024 \times 1024$ resolution, extracted from works of art.
CelebA-HQ is a human-face dataset which consists of 30,000 images at $1024 \times 1024$ resolution.
We randomly select $80\%$ as the training data and regard the rest as the test data, and resize the image to $512 \times 512$.

\paragraph{Experiment Configuration on Domain Adaptation}
For training the model, we apply AdamW~\citep{loshchilov2017decoupled} optimizer with batch size $16$, number of iterations $50$k, learning rate $10^{-4}$, weight decay $0.1$ and OneCycle \citep{smith2019onecycle} learning rate schedule.

\begin{figure}
    \centering
    \includegraphics[width=0.95\textwidth]{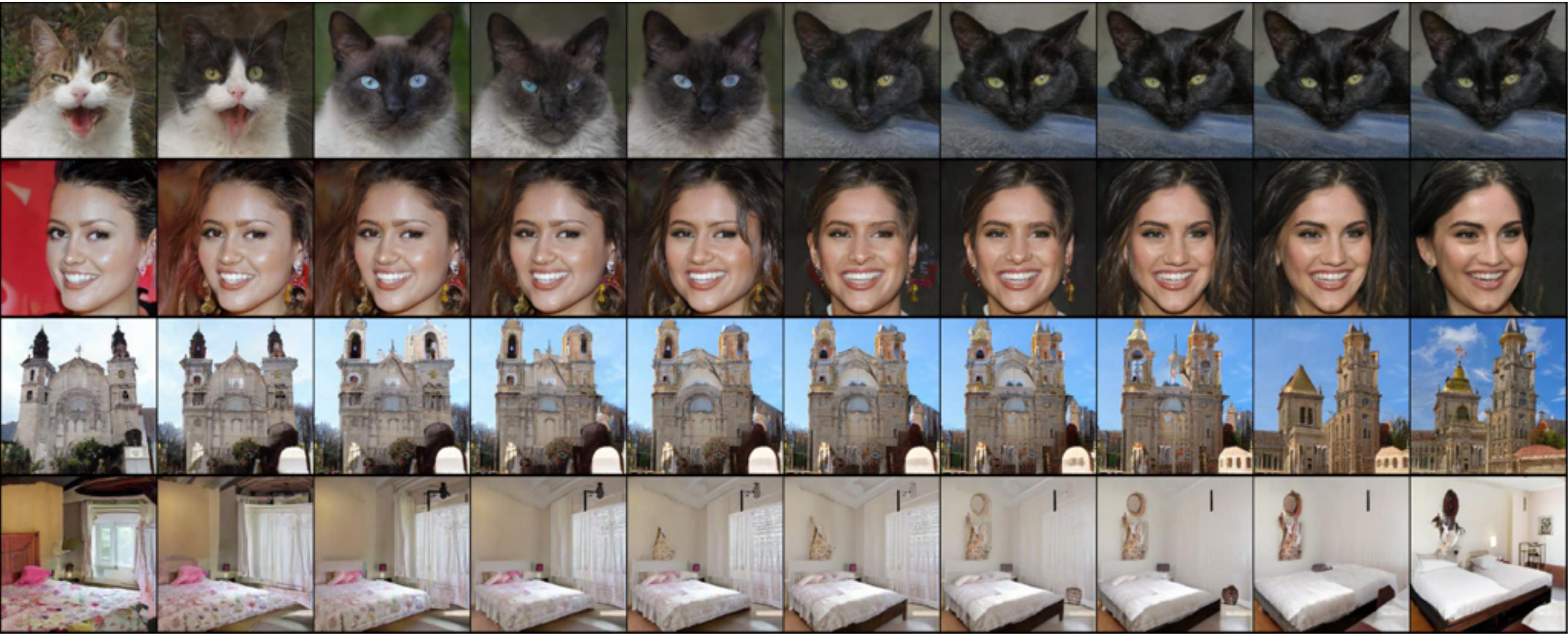}
    \caption{
    To visualize the latent space, we randomly sample $z_0$ and $z_1$ from $\mathcal{N}(0, I)$, and show the generated images of $\sqrt{\alpha} z_0 + \sqrt{1- \alpha} z_1$ for $\alpha \in [0,1]$.
    }
    \label{fig:interp}
\end{figure}

\begin{figure}
    \centering
    \includegraphics[width=0.98\textwidth]{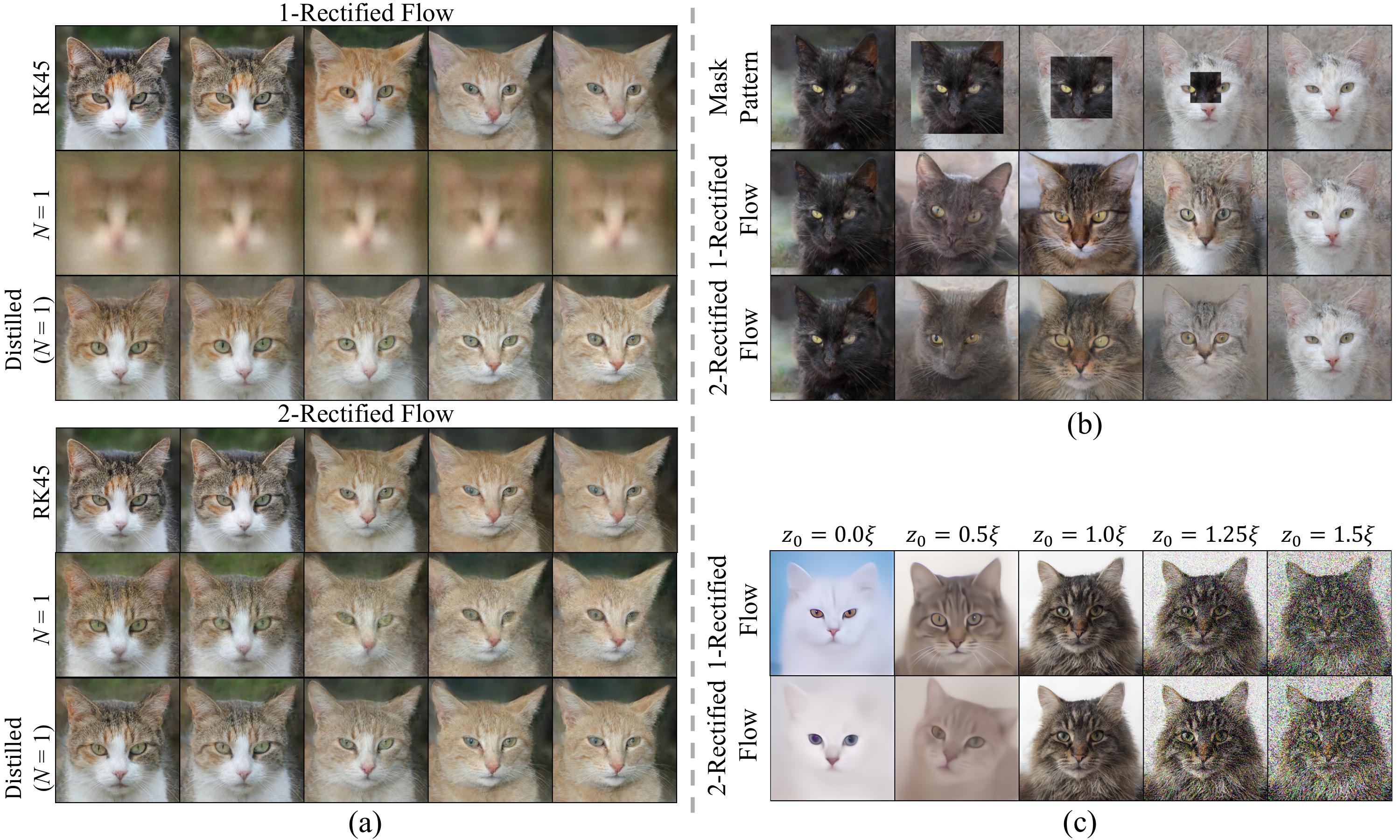}
    \caption{
    (a) We compare the latent space between Rectified Flow (0) and (1) using different sampling strategies with the same random seeds. We observe that (i) both 1-Rectified Flow and 2-Rectified Flow can provide a smooth latent interpolation, and their latent spaces look similar; (ii) when using one-step sampling ($N=1$), 2-Rectified Flow can still provide visually recognizable interpolation, while 1-Rectified Flow cannot; (iii) Distilled one-step models can also continuously interpolate between the images, and their latent spaces have little difference with the original flow.
    (b) We composite the latent codes of two images by replacing the boundary of a black cat with a white cat, then visualize the variation along the trajectory. The black cat turns into a grey cat at first, then a cat with mixing colors, and finally a white cat.
    (c) We randomly sample $\xi \sim \normal(0, I)$, then generate images with $\alpha \xi$ to examine the influence of $\alpha$ on the generated images. We find $\alpha<1$ results in overly smooth images, while $\alpha > 1$ leads to noisy images. 
    }
    \label{fig:interp_combine}
\end{figure}

\begin{figure}
    \centering
    \includegraphics[width=0.98\textwidth]{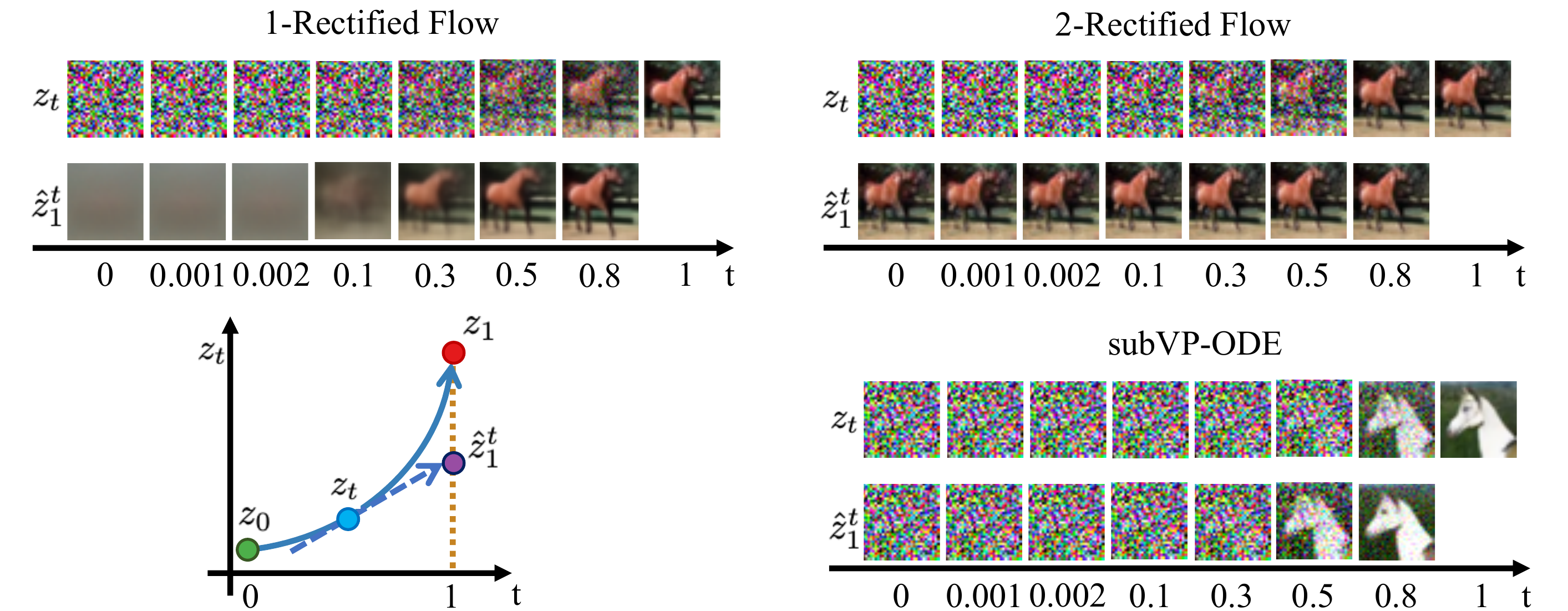}
    \caption{
    Sample trajectories $z_t$ of different flows on the CIFAR10 dataset,  %
    and the extrapolation $\hat{z}_1^t =z_t + (1-t) v(z_t, t)$ from different $z_t$. The same random seed is adopted for all three methods. The $\hat z_1^t$ of 2-rectified flow is almost independent with $t$, indicating that its trajectory is almost straight. 
    }
    \label{fig:cifar_target}
\end{figure}

\begin{figure}
    \centering
    \includegraphics[width=0.95\textwidth]{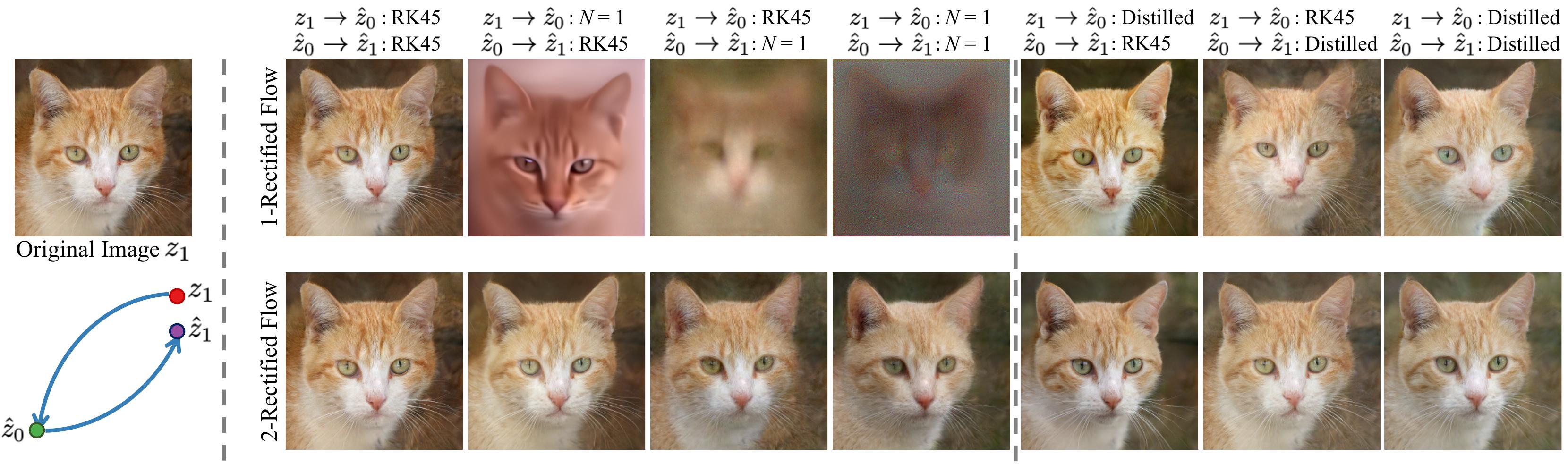}
    \caption{
    We perform latent space embedding / image reconstruction here. Given an image $z_1$, we use an \emph{reverse ODE solver} to get a latent code $\hat{z}_0$, then use a \emph{forward ODE solver} to get a reconstruction $\hat{z}_1$ of the image. The columns in the figure are \emph{reverse ODE solver (forward ODE solver)}. (i) Thanks to the`straightening' effect, 2-rectified flow can get meaningful latent code with only one reverse step. It can also generate recognizable images using one forward step. 
    (ii) With the help of distilled models, one-step embedding and reconstruction is significantly improved. 
    }
    \label{fig:image_reconstruction}
\end{figure}

\begin{figure}[h]
    \centering
    \includegraphics[width=\textwidth]{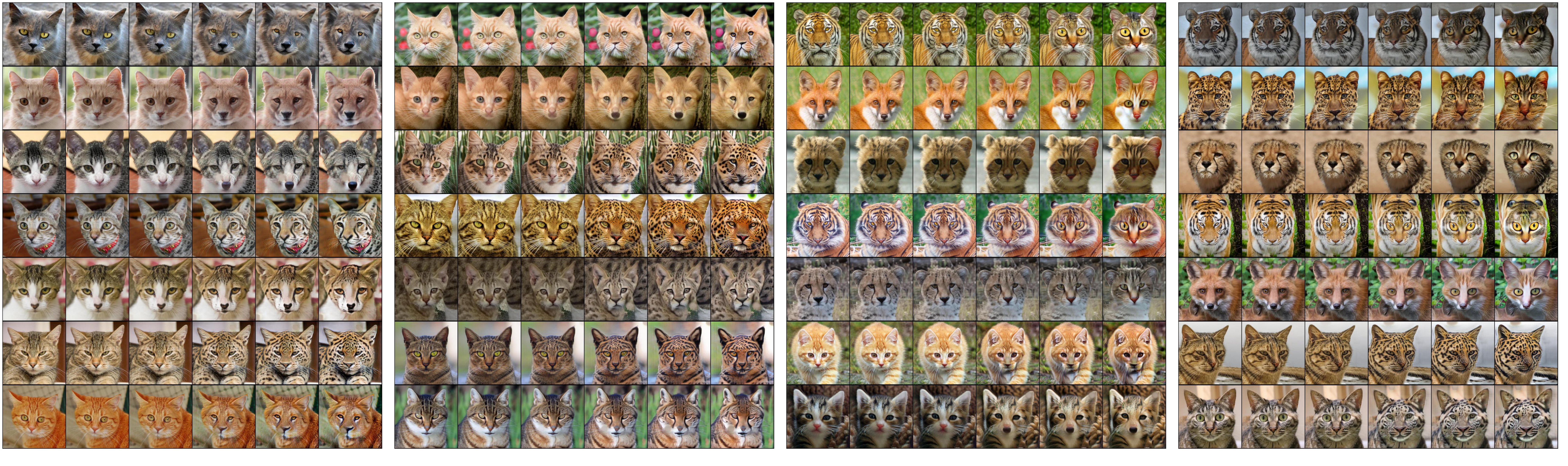} \\ 
    \includegraphics[width=\textwidth]{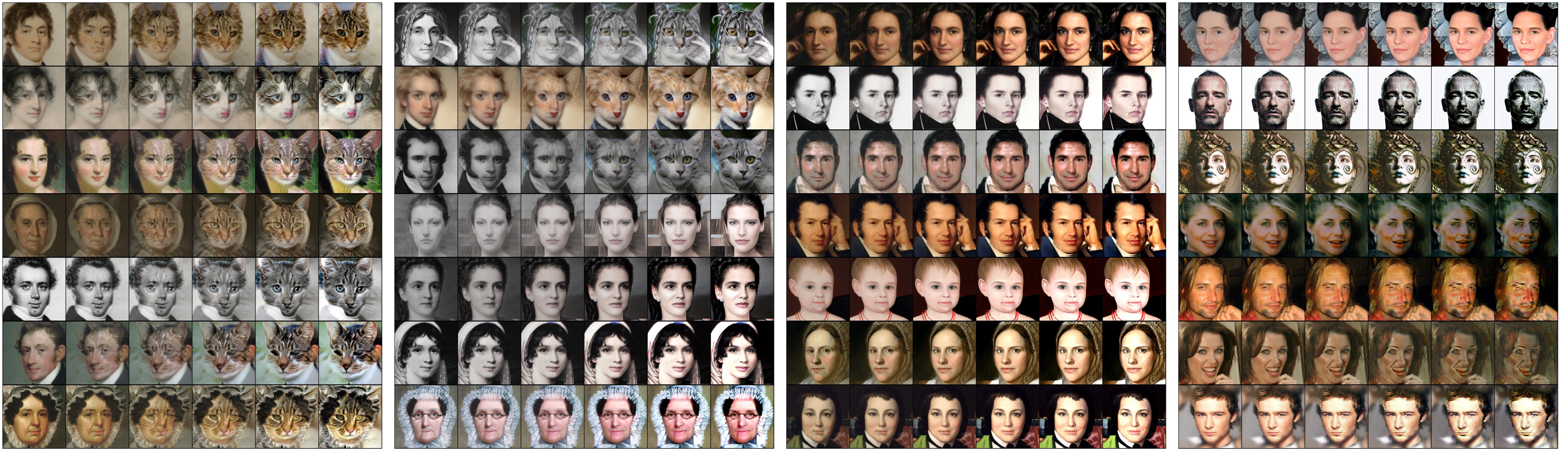}\\
    \includegraphics[width=\textwidth]{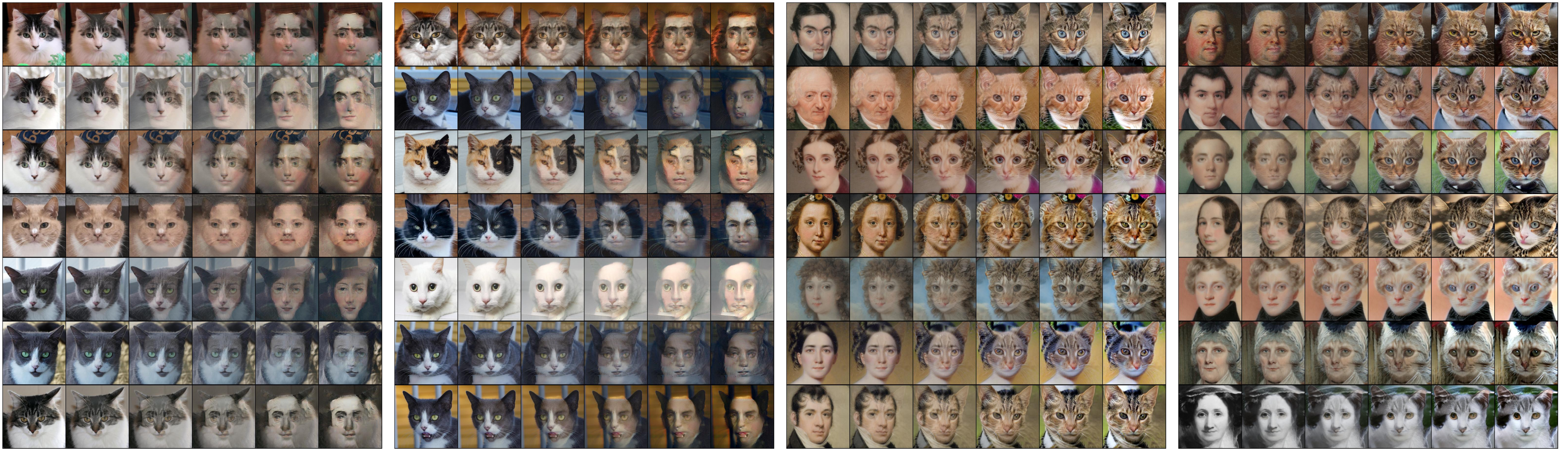}
    \caption{
    More results for image-to-image translation between different domains. 
    The images in each row are time-uniformly sampled from the trajectory of 1-rectified flow solved $N=100$ Euler steps with constant step size. 
    }
    \label{fig:appendix-traj1}
\end{figure}

%\newpage \clearpage 
%\input{texfiles/trash}
%\input{texfiles/method}

\end{document}